\documentclass{article}
\usepackage[preprint]{arvix}

\usepackage{hyperref}
\usepackage{url}
\usepackage[utf8]{inputenc} 
\usepackage[T1]{fontenc}    
\usepackage{booktabs}       
\usepackage{amsfonts}       
\usepackage{nicefrac}       
\usepackage{microtype}      
\usepackage{xcolor}         

\newtheorem{theorem}{Theorem}
\newtheorem{proof}{Proof}
\newtheorem{lemma}{Lemma}
\newtheorem{assumption}{Assumption}
\usepackage{algorithm}
\usepackage{algorithmic}
\usepackage{amssymb}
\usepackage{amsmath}
\usepackage{graphicx,subfigure}
\usepackage{adjustbox}
\usepackage{wrapfig}
\usepackage{extarrows}
\usepackage{color}
\usepackage[most]{tcolorbox}
\usepackage{listings}
\usepackage{enumitem}
\newcommand{\modelname}{{\text{WISDOM~}}}

\title{Wavelet Predictive Representations for Non-Stationary Reinforcement Learning}

\author{
\textbf{Min Wang}\textsuperscript{1},
\textbf{Xin Li}\textsuperscript{1},
\textbf{Ye He}\textsuperscript{1},
\textbf{Yao-Hui Li}\textsuperscript{1},
\textbf{Hasnaa Bennis}\textsuperscript{1},\\
\textbf{Riashat Islam}\textsuperscript{2},
\textbf{Mingzhong Wang}\textsuperscript{3} \\
\textsuperscript{1}Beijing Institute of Technology,
\textsuperscript{2}Microsoft Research AI Frontiers,\\
\textsuperscript{3}University of the Sunshine Coast}

\begin{document}
\maketitle

\begin{abstract}
The real world is inherently non-stationary, with ever-changing factors, such as weather conditions and traffic flows, making it challenging for agents to adapt to varying environmental dynamics. Non-Stationary Reinforcement Learning (NSRL) addresses this challenge by training agents to adapt rapidly to sequences of distinct Markov Decision Processes (MDPs). However, existing NSRL approaches often focus on tasks with regularly evolving patterns, leading to limited adaptability in highly dynamic settings. Inspired by the success of Wavelet analysis in time series modeling, specifically its ability to capture signal trends at multiple scales, we propose WISDOM to leverage wavelet-domain predictive task representations to enhance NSRL. WISDOM captures these multi-scale features in evolving MDP sequences by transforming task representation sequences into the wavelet domain, where wavelet coefficients represent both global trends and fine-grained variations of non-stationary changes. In addition to the auto-regressive modeling commonly employed in time series forecasting, we devise a wavelet temporal difference (TD) update operator to enhance tracking and prediction of MDP evolution. We theoretically prove the convergence of this operator and demonstrate policy improvement with wavelet task representations. Experiments on diverse benchmarks show that WISDOM significantly outperforms existing baselines in both sample efficiency and asymptotic performance, demonstrating its remarkable adaptability in complex environments characterized by non-stationary and stochastically evolving tasks.
\end{abstract}

\section{Introduction}
\label{Introduction}
Humans excel at dynamically adjusting their behaviors to adapt to continually changing goals and environments over extended periods. A promising paradigm to realize this adaptability in artificial agents is Non-Stationary Reinforcement Learning (NSRL), which focuses on training policies that can adapt sequentially to multiple tasks with varying state transition dynamics or reward functions. 
Meta-RL enhances sample efficiency and policy generality by identifying shared structures across diverse tasks, making it particularly effective for rapid adaptation. Based on meta-RL, recent NSRL methods have improved task inference through Gaussian mixture distributions, which emphasize the independence of different tasks~\cite{bing2023meta} or distance metric losses that enforce stringent smoothness conditions~\cite{sodhani2022block}. However, they generally exhibit suboptimal performance on non-stationary tasks due to meta-RL's inherent stationarity assumption and its neglect of temporal correlations between tasks.

In NSRL, the task evolution process\footnote{Time-evolving changes of tasks in dynamics or rewards that result in non-stationarity.} is typically modeled as a history-dependent stochastic process, implying that task changes follow certain trends or periodic patterns, thus allowing the tracking and prediction of task evolution~\cite{xie2020deep,Poiani2021MetaReinforcementLB,chen2022adaptive}.
A line of work~\cite{xie2020deep,ren2022reinforcement} attempts to explicitly model the task evolution process as a first-order Markov chain, which limits its applicability to smoothly evolving non-stationary tasks, making it ineffective for more complex tasks with rapid changes, where it can result in the accumulation of prediction errors. 
\cite{bing2023meta2} employs recurrent-based neural networks to handle non-stationary sequential data. \cite{Poiani2021MetaReinforcementLB,chen2022adaptive,tennenholtz2023reinforcement} assume a history-dependent task evolution process and approximate it with Gaussian Process (GP)~\cite{seeger2004gaussian} or planning in a latent space. 
The non-stationary kernel function adapts GP to varying covariance, but requires prior knowledge and introduces more parameters. In addition, most existing methods have been primarily limited to task evolution processes that exhibit a regular pattern with fixed periods. However, practical non-stationary tasks generally exhibit time-varying periods or frequencies, posing a challenge for handling such irregular patterns with stochastic periods.

As in Fig.~\ref{fig:motivation} (a), a noisy non-stationary signal changes with increasing frequency, with its three input stages (A, B, C) having similar means ($\approx 0$) and variances ($\approx 2$). Although indistinguishable in the time domain, they can be distinguished in the Fourier spectrum with different main frequencies in Fig.~\ref{fig:motivation} (b). However, the Fourier Transform (FT)~\cite{cooley1965algorithm} does not reveal when each frequency component appears. For instance, reversing the sequence to create a fast-to-slow pattern (C-B-A), thereby generating a different signal, yields the same Fourier spectrum. In contrast, Wavelet Transform (WT)~\cite{polikar1996wavelet} can resolve this by simultaneously preserving time-frequency domain information of the non-stationary signal (initial \emph{approximation coefficient}) and iteratively extracting multi-scale features using different frequencies. Based on \emph{sampling theorem}~\cite{shannon1949communication}, each decomposition of the approximation coefficient halves the sequence length, reducing the amount of data without losing fundamental features. High-frequency noise (\emph{detail coefficient}) in Fig.~\ref{fig:motivation} (c-d) is gradually separated from low-frequency components (\emph{approximation coefficient}), and after the second decomposition in Fig.~\ref{fig:motivation}, the smooth low-frequency feature (Fig.~\ref{fig:motivation} (e)) basically reflects the original evolving trend (Fig.~\ref{fig:motivation} (a)).

Inspired by the success of WT theory in handling non-stationary time series by gradually separating different frequency trends~\cite{polikar1996wavelet,Yang2024WaveNetTN}, we propose \textbf{W}avelet-\textbf{I}nduced task repre\textbf{S}entation pre\textbf{D}iction for n\textbf{O}n-stationary reinforce\textbf{M}ent learning (WISDOM), which leverages wavelet domain information from the task representation sequence $\mathbf{z}$ to extract underlying structural and temporal patterns.
Specifically, we start by encoding trajectories
\begin{wrapfigure}{r}{0.43\textwidth}
\centering
\setlength{\abovecaptionskip}{0cm}
\vskip -0.15in
\subfigure{
\includegraphics[width=0.43\textwidth]{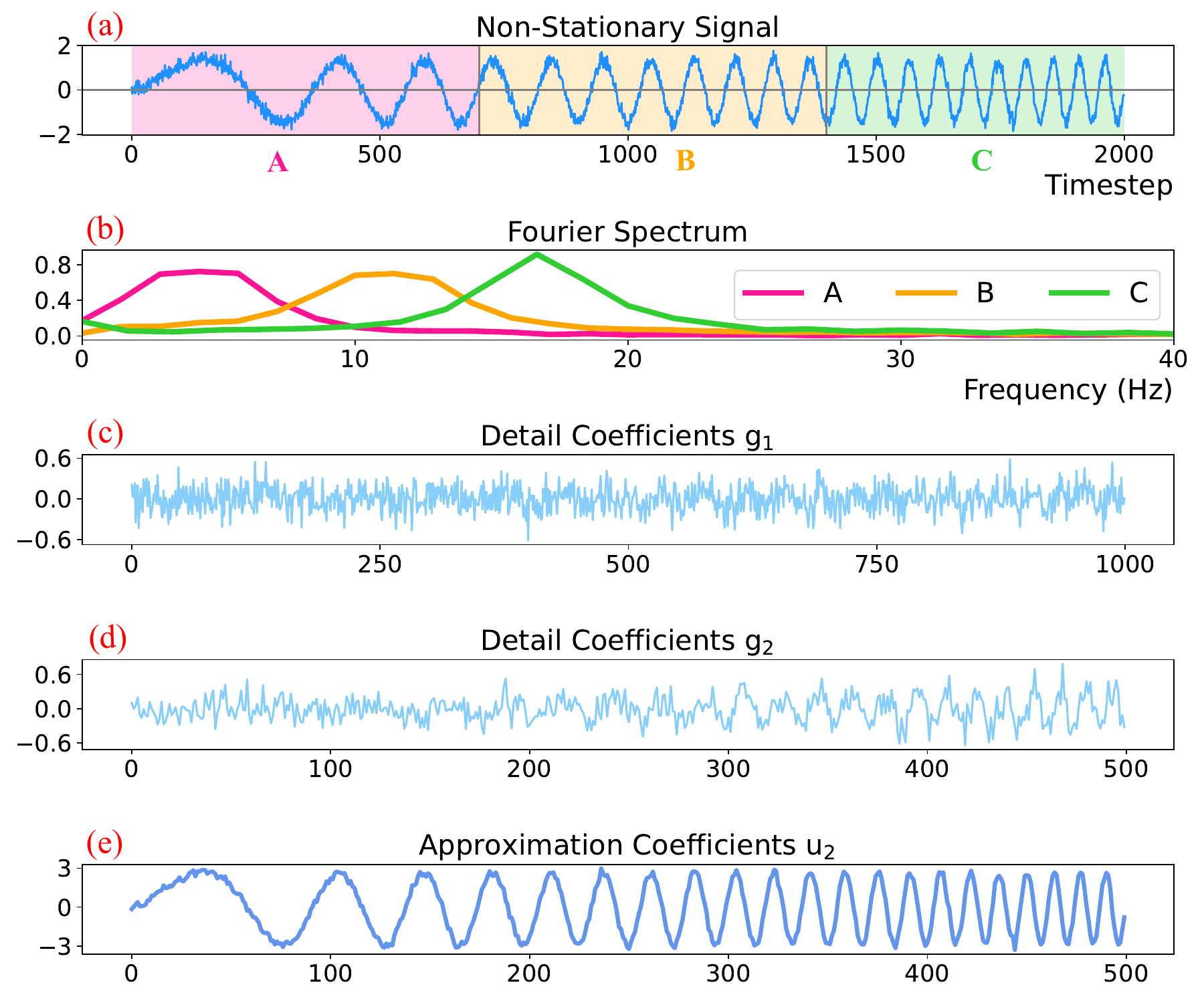}
}
\vskip -0.05in
\caption{A motivating example.
}
\label{fig:motivation}
\vskip -0.15in
\end{wrapfigure}
$(s_i,a_i,s_{i+1},r_i)_{i=0}^T$ to obtain $\mathbf{z}$ to imply the changes of tasks~\cite{xie2020deep,chen2022adaptive}. Each $z_i$ in this $\mathbf{z}$ sequence preserves the influence of dynamics and rewards. Each dimension of $z_i$ corresponds to a variate in a multivariate time series. Subsequently, we utilize a wavelet representation network to transform $\mathbf{z}$ into the wavelet domain. Iterative decompositions of $\mathbf{z}$ yield an approximation coefficient that captures slowly changing features and a series of detail coefficients that capture rapidly evolving local features. 
To remove noise and preserve more rapid task changes, we selectively filter detail coefficients through downsampling. We then transform these wavelet coefficients back to the time domain to restore more intrinsic task representations. Moreover, we design a wavelet operator for the explicit temporal difference (TD) update of the wavelet representation network to help capture changes in task structures. Additionally, an auto-regressive loss is incorporated to reinforce long-term temporal dependencies, empowering task evolution predictions. Finally, we integrate restored task representations into policy learning, facilitating prompt policy adjustments and improving policy adaptability across diverse non-stationary tasks. Our key contributions include:
\begin{enumerate}[leftmargin=*]
    \item We first propose to address non-stationary RL via perceiving its evolution process in the wavelet domain.  
    By iteratively performing wavelet decomposition, we can capture different evolutionary patterns emerging with stochastic period by utilizing a series of varying resolutions or frequencies. This allows \modelname to dynamically adjust its behavior in response to evolving task conditions, leading to rapid adaptability and improved policy performance. 
    \item  We theoretically demonstrate that wavelet-domain features can serve as informative indicators of policy performance. And our wavelet task representations are demonstrated to lead to improved policies. To better capture structural regularities of the task evolution process, we design a wavelet TD update operator, theoretically ensuring its convergence.
    \item  We perform a comprehensive evaluation on broader task-distributed Meta-World~\cite{yu2020meta}, Type-1 Diabetes~\cite{xie2018simglucose}, and MuJoCo~\cite{Todorov2012MuJoCoAP}, and the overall results demonstrate the superior and efficient adaptability of WISDOM compared to baselines. We provide new tasks on these benchmarks and will open-source for evaluation in challenging non-stationary settings.
\end{enumerate}

\section{Related Work}
\label{Related Work}
\textbf{Wavelet Features-Based DL.} Recent studies primarily investigate the incorporation of wavelet features for image processing~\cite{yu2021wavefill,yao2022wave}, audio signal analysis~\cite{pan2022wnet,shi2023sequence,zhang2024speaking}, and time series prediction~\cite{zhou2022fedformer,minhao2021t}. In addition, some studies~\cite{bolos2020new, kong2022non} have turned to wavelet transforms to integrate time-frequency features for handling non-stationary signals. A wavelet neural network~\cite{stock2022trainable} has been proposed to learn a specialized filter-bank for non-stationarity modeling. \citealt{Yang2024WaveNetTN} designs a graph spectral wavelet to capture the high-frequency components of non-stationary graph signals. To address the issue of online label shift, \citealt{qian2024efficient} proposes a streaming wavelet operator that estimates environmental changes for efficient non-stationary learning. Furthermore, a wavelet attention network is introduced to analyze heterogeneous time series~\cite{wang2023wavelet} or even predict seasonal components~\cite{wan2024tcdformer}. 
 
\textbf{Non-Stationary RL.}
A significant strand of recent research frames the problem of non-stationarity within the meta-RL framework. ~\citealt{alshedivat2018continuous} makes an initial endeavor to apply gradient-based meta-RL to address non-stationary tasks. Meta-experience replay~\cite{MER} and parameter learning rate modulation~\cite{gupta2020look} are introduced to enhance adaptability. However, policy gradient updates are typically sample inefficient, and they primarily focus on combining context-based meta-RL to achieve faster adaptation in non-stationary tasks. To improve task representation, Gaussian Mixture Models~\cite{bing2023meta}, task-metric approaches~\cite{sodhani2022block}, causal graph~\cite{zhang2024tackling} and recurrent neural networks~\cite{Poiani2021MetaReinforcementLB} have been successively leveraged to assist the context encoder in extracting invariant features. To accurately model the task evolution process, some approaches~\cite{xie2020deep,ren2022reinforcement} hypothesize the latent context sequence to follow a Markov chain formulation. While such methods demonstrate superior adaptability in smoothly changing tasks, they struggle to adapt effectively to complex tasks that involve sudden changes. Therefore, another branch of research~\cite{Poiani2021MetaReinforcementLB,tennenholtz2023reinforcement,chen2022adaptive} assumes history-dependence of the latent context and seeks to utilize Gaussian Processes~\cite{seeger2004gaussian}, the Rescorla-Wagner model~\cite{rescorla1972theory} and reward functions~\cite{chen2022adaptive} to better characterize the task evolution process.

\section{Preliminary}
\subsection{Context-Based Meta-RL}
Meta-RL is defined on a distribution of tasks $p(\mathcal{T})$, where each task is an MDP represented by a tuple $(\mathcal{S}, \mathcal{A}, \mathcal{P}, \mathcal{R}, \gamma, \mathcal{P}(s_0))$, in which $\mathcal{S}$ denotes the state space, $\mathcal{A}$ the action space, $\mathcal{P}$ the transition dynamics, $\mathcal{R}$ the reward function, $\mathcal{P}(s_0)$ the initial state distribution, and $\gamma \in [0,1)$ the discount factor~\cite{1998Reinforcement}. In meta-RL, both $\mathcal{P}$ and $\mathcal{R}$ are assumed to be unknown, with different tasks differing only in $\mathcal{P}$ or $\mathcal{R}$. In context-based meta-RL, the task representation $z$ is derived from past transition histories (referred to as context) by a context encoder $e_\eta$. Building on task inference, recent NSRL methods focus on modeling the evolution of $z$ to improve the accuracy of trend predictions, thereby enhancing the adaptability of the policy $\pi(a|s,z)$. In a conventional meta-RL setting, the agent can interact with numerous MDPs, yet the $\mathcal{P}$ and $\mathcal{R}$ of these MDPs remain constant over time, indicating stationary MDPs. In contrast, NSRL requires the agent to engage with multiple MDPs where $\mathcal{P}$ and/or $\mathcal{R}$ evolve over time, introducing new challenges for rapid adaptation.

\subsection{Non-stationary RL}
\label{Problem Formulation}
Non-stationary RL is defined on the time-evolving task distribution $p(\mathcal{T}_t)$, where an agent interacts with a sequence of MDPs $\mathcal{M}_{\omega_0}, \mathcal{M}_{\omega_1}, \cdots, \mathcal{M}_{\omega_h}$. The evolution of these MDPs is determined by a history-dependent stochastic process $\rho$, i.e., $\omega_{h+1}\sim \rho\left(\omega_{h+1}|\omega_0, \omega_1, \ldots, \omega_{h}\right)$, where $\omega$ is a task ID that regulates the properties of different MDPs~\cite{Poiani2021MetaReinforcementLB}. To better adapt to more practical scenarios, $\omega$ no longer undergoes inter-episode changes, nor does it follow equidistant intra-episode changes. Instead, in our setting, and consistent with~\cite{bing2023meta,chen2022adaptive}, $\omega$ varies within episodes with a relatively stochastic period.
The overall objective of NSRL is formulated as:
\begin{align}
\textstyle\underset{\pi}{\operatorname{argmax}} \mathbb{E}_{\omega_h}\left[\sum_{t=0}^{\infty} \mathbb{E}_{{\mathcal{T}_t}}\left[\sum_{h=0}^{\infty} \gamma^t r^h_t \mid \mathcal{M}_{\omega_h}, \pi\right]\right].
\end{align}

\subsection{Wavelet Transform}
The Wavelet Transform (WT)~\cite{burrus1998wavelets} analyzes the various frequency components of a non-stationary signal utilizing a scalable and translatable mother wavelet function $\psi(t)$. To achieve multi-resolution analysis, the Discrete Wavelet Transform (DWT)~\cite{shensa1992discrete} is introduced to decompose the signal layer by layer at different resolutions. Let the initial approximation coefficient $\mathbf{u}_0$ be equal to an input discrete time series $x(n)$. By substituting a pair of low-pass filter $y_0$ and high-pass filter $y_1$ for $\psi(t)$, the following iterative form is derived through \textbf{discrete convolution}: 
\begin{align}
\label{eq:dwt}
\mathbf{u}_m(n) = \sum\textstyle_{k=1}^{K} y_0(k)\mathbf{u}_{m-1}(2n-k),\quad \mathbf{g}_m(n) = \sum_{k=1}^{K} y_1(k)\mathbf{u}_{m-1}(2n-k),
\end{align}
where $m$ is the decomposition level and $k$ is the resolution size. Approximation coefficients $\mathbf{u}_m$ represent low-frequency components, whereas detail coefficients $\mathbf{g}_m$ represent high-frequency components. These multi-resolution and localization properties allow the DWT to simultaneously capture both the global features and local variations of the signal, making it well-suited for handling non-stationary signals, where frequency and period change over time. 
We also provide a derivation of DWT in Appendix~\ref{a:dwt} to enhance understanding.

\section{Methodology}
\label{method}
Sec.~\ref{sec:inference} first outlines the setting of non-stationary tasks and describes how to conduct task inference. Subsequently, Sec.~\ref{sec:wt} elaborates on the implementation of a trainable wavelet representation network that converts task representation sequences into the wavelet domain to capture inherent evolving patterns. Thereafter, Sec.~\ref{sec:update} devises a wavelet temporal difference (TD) update operator to collaborate with the auto-regressive (AR) loss in optimizing the wavelet representation network. Finally, Sec.~\ref{sec:policy} explains how to effectively harness this inherent structural representation in policy learning.

\subsection{Module A: Non-Stationary Task Inference with Context Encoder}
\label{sec:inference}
In scenarios with non-stationary task distributions $p(\mathcal{T}_t)$, assuming the agent has collectively undergone $H$ instances/times of MDP evolution when interacting with the environment, the resulting trajectory, composed of a recent history of transitions $c^{0:H}_{0:T}$ (abbreviated as context $\mathcal{C}$), corresponding to $H+1$ MDP segments (i.e. $\mathcal{M}_{\omega_{0}},\mathcal{M}_{\omega_{1}},\ldots,\mathcal{M}_{\omega_{H}}$). To further enhance applicability, in our non-stationary setup, the period of each $\mathcal{M}_{\omega_{h}}$, denoted as $T_h$, is assumed unknown. The current MDP duration $T_h$ potentially differs from its predecessors $T_{h-1}$, and $T=\sum^{H}_{h=0}T_h$. Each transition $c = (s,a,s',r)$ within $\mathcal{C}$ encapsulates the current state $s$, action $a$, next state $s'$, and the reward $r$. \footnote{Since the time steps of changes in the MDP is unknown, superscript $h$ is omitted in the following notation.}

To enhance practicality, we assume that the agent lacks prior knowledge of the environment and that the true task ID $\omega$ is inaccessible. We employ a context encoder $e$ parameterized by $\eta$ to infer task-relevant information. The encoder $e_\eta$ processes $\mathcal{C}$ to generate a \textbf{task representation sequence} $\textstyle\mathbf{z}=[z_0,z_1,\ldots,z_{T}]$ that is generally viewed as an approximation of the $\omega$ sequence~\cite{xie2020deep,chen2022adaptive}, where each $z$ uniquely corresponds to each transition $c$. The KL-divergence is used as a variational approximation to an information bottleneck~\cite{rakelly2019efficient} to train $e_\eta$: 
\begin{equation}
\mathcal{J}_{\eta}=\mathbb{E}_{\mathcal{C}\sim \mathcal{B}}[D_{\mathrm{KL}}(e_\eta(\mathbf{z}|\mathcal{C})\|p(\mathbf{z}))],
\label{equ:e}
\end{equation}
where $p(\mathbf{z})$ represents a Gaussian prior, and the context $\mathcal{C}$ is sampled from the replay buffer $\mathcal{B}$.
\begin{figure*}[t]
    \vskip -0.1in
    \centering
    \includegraphics[width=0.95\textwidth]{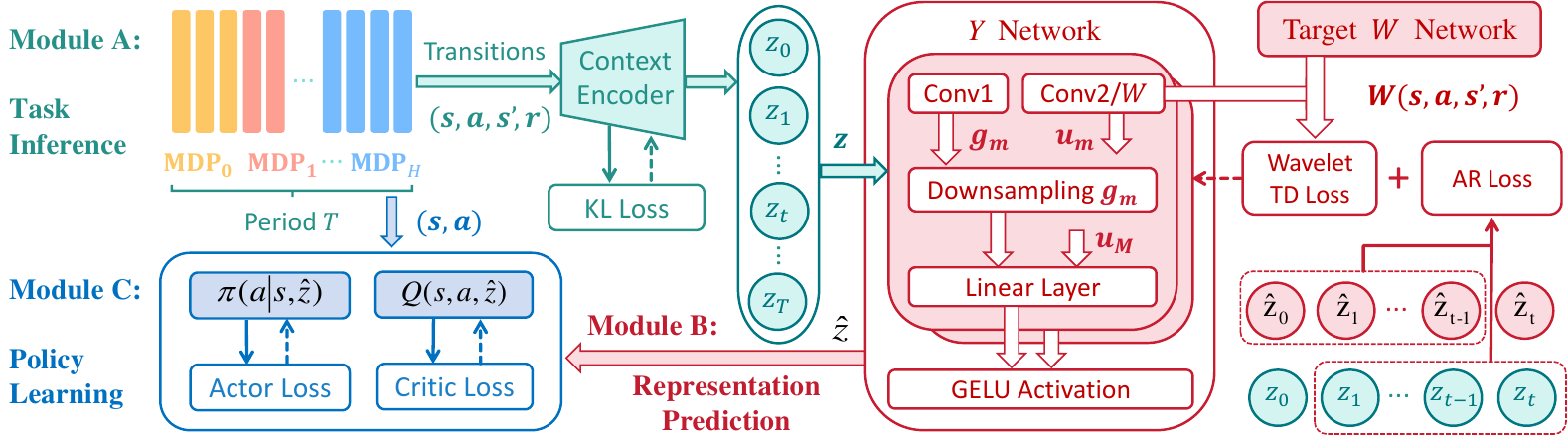}
    \vskip -0.1in
    \caption{\textbf{The architecture of WISDOM.} Module A begins with task inference to derive time-domain task representation $z$. Then module B transforms $z$ into wavelet domain by a wavelet representation network $Y_\phi$, jointly optimized by wavelet TD loss and AR loss to derive wavelet task representation $\hat{z}$. Finally, module C integrates $\hat{z}$ to adjust the policy based on predicted evolving trend.}
    \label{fig:Model structure}
    \vskip -0.1in
\end{figure*}

\subsection{Module B: Tracking Task Evolution via Wavelet Representation}
\label{sec:wt}
In NSRL, the task sequence $\mathcal{M}_{\omega_0:\omega_H}$ is generally modeled as a history-dependent stochastic process~\cite{xie2020deep,chen2022adaptive,Poiani2021MetaReinforcementLB}. From a time series analysis perspective, this sequence can be viewed as a non-stationary signal, with the underlying frequency of each type of task $f_h$ changing over time. WT is widely recognized for its ability to process non-stationary signals by decomposing them into different frequency components (multi-scale features)~\cite{polikar1996wavelet}. Conceptually supported by~\cite{xie2020deep,Poiani2021MetaReinforcementLB}, \modelname is the first to address NSRL by leveraging wavelet-domain task representations to track the temporal evolution of tasks.

Inspired by the design of prevalent WT networks~\cite{oord2016wavenet,shi2023sequence}, our \textbf{wavelet representation network} $Y_{\phi}$ consists of two dilated causal convolution networks, \textbf{\emph{Conv1}} and \textbf{\emph{Conv2}}, followed by a linear layer to guarantee the capture of underlying task dynamics. Specifically, $Y_{\phi}$ performs the DWT in Eq.~\ref{eq:dwt} with recursive convolution and transforms $\mathbf{z}$ into the wavelet domain: 
\begin{equation}
\mathbf{g}_m = \textbf{\emph{Conv1}}(\mathbf{u}_{m-1}, y_1; M), \mathbf{u}_m = \textbf{\emph{Conv2}}(\mathbf{u}_{m-1}, y_0; M), m\in[1,M], M\in \mathbb{N}^+,
 \mathbf{u}_0=\mathbf{z},   
\end{equation}
where the learnable convolution kernels $y_0$ and $y_1$ can be initialized as classical basis functions such as Haar wavelet\footnote{For haar wavelet~\cite{pattanaik1995haar}, $y_0=[1/\sqrt{2},1/\sqrt{2}]$ computes the average over adjacent elements, smoothing the $\mathbf{z}$ sequence, akin to a low-pass filter. In contrast, $y_1=[1/\sqrt{2},-1/\sqrt{2}]$ computes the difference, highlighting local changes and extracting high-frequency components, similar to a high-pass filter.}. Hence, $Y_{\phi}$ is guided to approximate traditional wavelet behavior while adaptively adjusting the kernel values during training. 
$y_0$ functions as a low-pass filter, eliminating high-frequency components, thereby enabling the approximation coefficient $\mathbf{u}_m$ to capture the overall trends of non-stationary task evolution. Conversely, $y_1$ acts as a high-pass filter, gathering high-frequency components, allowing the detail coefficients $\mathbf{g}_m$ to capture rich detail features of tasks. After performing $M$ decompositions on $\mathbf{u}_m$, $Y_{\phi}$ obtains the $M$th $\mathbf{u}_{M}$, along with the total $M$ detail coefficients $\mathbf{g}_{1:M}$. 
During each decomposition step $m$, the network downsamples and preserves the most recent detail coefficients $\tilde{\mathbf{g}}_m$. Finally, $Y_{\phi}$ applies a linear transformation to revert $\tilde{\mathbf{g}}_m$ and $\mathbf{u}_{M}$ to the time domain, yielding a more expressive and intrinsic \textbf{wavelet task representation sequence} $\hat{\mathbf{z}}$.

\subsection{Optimization Objective Design for Wavelet Representation Network $Y_\phi$}
\label{sec:update}
We propose a TD-style optimization of the representation in the wavelet domain and theoretically demonstrate the convergence of the corresponding wavelet TD operator, thereby facilitating learning and search for task-oriented representations.

In wavelet theory, the \textbf{\emph{Conv2}} layer (referred to as the $W$ network hereafter) integrates a low-pass filter to learn stable and inherent non-stationary features~\cite{Shensa1992TheDW}. We define a TD-style update operator over wavelet-based task features learned by the $W$ network as $\mathcal{F} W(\mathbf{z}_t) = \mathbf{z}_t + \Gamma W(\mathbf{z}_{t+1})$ to explicitly update $Y_\phi$. Similar to successor features~\cite{barreto2017successor}, wavelet features $W(\mathbf{z}_t)$ satisfy the Bellman equation. Theorem~\ref{th:mapping} demonstrates that $\mathcal{F}$ is a contraction mapping (Proof in Appendix~\ref{a-th:loss}), ensuring the convergence of wavelet representation updates and allowing for stable and consistent learning. Unlike most existing work that implicitly updates the representation network with TD loss over the value function, our explicit TD update will not neglect low-reward yet critical features. Although rewards in many RL settings can be sparse or delayed, learned representations tend to be denser and more informative. Intuitively, the wavelet TD update allows $Y_\phi$ to better capture structural regularities in MDP sequences and enhance sample efficiency. 

Moreover, the wavelet TD loss relies solely on the single-step future task representation $z_{t+1}$ and provides a more concise optimization objective than the commonly used auto-regressive (AR) loss in time-series forecasting~\cite{shi2023sequence,oord2016wavenet}. Our $W$ network is also implemented with dilated causal convolution, ensuring that the $i$-th output of $Y_\phi$ only depends on the first $i$ input elements, thus maintaining the conditional dependency essential for AR modeling to predict task changes. However, AR objectives often suffer from accumulated prediction errors, where errors in the current representation propagate and degrade future predictions. As a remedy, the wavelet TD update with the target $W$ network helps mitigate error propagation and stabilizes optimization through delayed target update.

\begin{theorem}
\label{th:mapping}
Let \(\mathcal{W}\) denote the set of all functions \(W:\mathcal{S}\times\mathcal{A}\rightarrow\mathbb{C}^{D}\) that map from the time domain to the wavelet domain. The wavelet update operator \(\mathcal{F}:\mathcal{W}\rightarrow\mathcal{W}\), defined as $\mathcal{F} W(\mathbf{z}_t) = \mathbf{z}_t + \Gamma W(\mathbf{z}_{t+1})$, is a contraction mapping, where $\mathbf{z}_{t}$ and $\mathbf{z}_{t+1}$ represents the current and next task representation sequence, respectively, and \(\Gamma\) denotes the discount factor in a diagonal matrix form. 
\end{theorem}

The overall optimization objective of wavelet representation network $Y_\phi$ is formulated as follows: 
\begin{equation}
\label{eq:Ynet}
\textstyle\mathcal{J}_{\phi}=\textstyle\alpha_{Y}\underbrace{\mathbb{E}_{c \sim \mathcal{B}}\left[\frac{1}{2}\left(W_{\varphi}(\mathbf{z}_t)-(\mathbf{z}_t+\Gamma\mathbb{E}\left[W_{\bar{\mu}}(\mathbf{z}_{t+1} )\right])\right)^2\right]}_{\text{Wavelet TD loss}}
    -\underbrace{\textstyle\mathbb{E}_{\hat{z}\sim Y_\phi}\left[\text{log}\prod\limits_{t=0}^TP(\hat{z}_t|\hat{z}_{<t})\right]}_{\text{AR loss}},
\end{equation}
where $\alpha_{Y}$ balances the contribution of each loss. Similarly to DQN~\cite{mnih2013playing}, $W_{\bar{\mu}}$ is a target network to stabilize the training and is updated separately and softly. Complemented to the wavelet TD loss, the AR loss adheres to stricter temporal constraints, preventing the trends captured by $Y_\phi$ from being temporally misaligned due to the relaxed orthogonality of the learnable filters. Moreover, it reinforces long-term temporal dependencies, empowering $Y_\phi$ to make predictions based on the task representation sequence. Together, they facilitate rapid adaptation to evolving task structures.

\subsection{Module C: Policy Learning Through Wavelet Task Representations}
\label{sec:policy}
To establish a more intimate relationship between dynamic adjustment of the policy in advance and the non-stationary evolving trends, we further incorporate the predicted wavelet task representation $\hat{z}$ generated by $Y_{\phi}$ into the policy iteration. Our \modelname is based on the Soft Actor-Critic (SAC) algorithm~\cite{Haarnoja2018SoftAA} and can also be integrated with any downstream RL algorithms. 

The objective of training the contextual critic network $Q_\upsilon$ is to minimize the squared residual error: 
\begin{equation}
\label{eq:critic}
\textstyle\mathcal{J}_{\upsilon}=\mathbb{E}_{\left(s,a\right) \sim \mathcal{B},\hat{z}\sim Y_{\phi}}\left[\frac{1}{2}\left(Q_\upsilon\left(s,a,\hat{z}\right)-Q_\text{target}\right)^2\right],
\end{equation}
with the target $Q$ value defined as $\textstyle Q_\text{target}=r+\gamma \mathbb{E}_{s'\sim\mathcal{B}, a'\sim \pi_\theta,\hat{z}\sim Y_{\phi}}\left[Q_{\bar{\zeta}}\left(s',a',\hat{z}\right)\right]$, where $\bar{\zeta}$ denotes stopping the backpropagation of gradients of the target critic network $Q_\zeta$. The parameters of $Q_\zeta$ are updated with an exponential moving average derived from the weights of the $Q_\upsilon$ network.
The contextual policy network $\pi_\theta$ can be updated by optimizing the following objective:
\begin{equation}
\label{eq:actor}
    \mathcal{J}_{\theta}=\mathbb{E}_{s\sim \mathcal{B},a \sim \pi_\theta,\hat{z}\sim Y_{\phi}}\left[\alpha \log \left(\pi_\theta\left(a|s,\hat{z}\right)\right)-Q_\upsilon\left(s, a, \hat{z}\right)\right],
\end{equation}
where $\alpha$ is a temperature coefficient. Alg.~\ref{alg} in Appendix~\ref{pseudocode} presents the pseudocode of WISDOM. The following theorems establish a connection between task representations through the wavelet transform and policy performance. In Theorem~\ref{th:pi-w}, we demonstrate that the wavelet-domain task representations preserve the ability to indicate policy performance. And in Theorem~\ref{th:pi-hatz}, we prove that the policy is improved with restored time-domain task representations during policy iteration. 
\begin{theorem}
\label{th:pi-w}
Suppose that the reward function \(\mathcal{R}(z)\) can be expanded into a Bth-degree Taylor series for \(z\in\mathbb{R}^D\), then for any two policies \(\pi_1\) and \(\pi_2\), their performance difference is bounded as:
\begin{align}\label{eq:thm2-desired-bound}
    |J_{\pi_1}-J_{\pi_2}|\leq \frac{\sqrt{D}}{1-\gamma}\cdot \sum_{b=1}^{B} \frac{\left\|\mathcal{R}^{(b)}(0)\right\|_D}{b!}\cdot\mathop{\max}\limits_{1\leq q \leq D}\mathop{\sup}\limits_{d_q\in\mathbb{R},\beta_q>0}\left|\mathcal{W}_{\pi_1}^{(b)}(d_q, \beta_q)-\mathcal{W}_{\pi_2}^{(b)}(d_q, \beta_q)\right|,
\end{align}
where \(\mathcal{W}_{\pi}^{(b)}(d,\beta)\) denotes the wavelet transform of the $b$th power of the task representation sequence $\mathbf{z}^{(b)}=[z_0,z_1,\ldots,z_{T_H}]^{(b)}$ for any integer \(b\in [1,B]\).
\end{theorem}
See Appendix~\ref{a-theorem2} for the proof. Theorem~\ref{th:pi-w} describes how the performance differences of the wavelet-domain features control the performance of the corresponding policies. We prove that wavelet-domain features can serve as an informative indicator of policy performance, aiming at more efficient policy iteration to find the optimal policy and exhibiting fast convergence in experiments. 

\begin{theorem}
\label{th:pi-hatz}
\modelname returns a policy $\pi(\cdot | s,\hat{z})$ conditioned on the wavelet task representation that improves upon its iterated history policy $\pi_h$, satisfying $J_{\text {\modelname}}\geq J_{\pi_h}$ under certain assumptions. 
\end{theorem}

The assumptions and proof are provided in Appendix~\ref{a-th:PI}. Theorem~\ref{th:pi-hatz} describes that our restored $\hat{z}$ in the time domain after wavelet transform leads to an improved contextual policy during iteration. The wavelet transform is well known for enhancing the signal-to-noise ratio (SNR) by separating different frequencies~\cite{polikar1996wavelet}. And $\hat{z}$ is expected to filter out task-irrelevant information and helps the policy focus on essential non-stationary characteristics, providing a clearer optimization direction.

\section{Experiments}
\label{Experiments}
\textbf{Task settings.} Following SeCBAD~\cite{chen2022adaptive}, the stochastic period $T_h$ of an MDP $\mathcal{M}_{\omega_h}$ is sampled from a Gaussian distribution with a mean of 60 and a variance of 20 time steps. We mainly performed experiments on the following three benchmarks, and more details are in Appendix~\ref{detail setting}.
\begin{enumerate}[leftmargin=*]
\item \textbf{Meta-World}~\cite{yu2020meta}, consisting of 50 robotic manipulation tasks with a broader distribution that is generally considered challenging, exhibits non-stationarity in the continuous variation of the target position for the robotic arm movement, which is correlated with the reward function; \item \textbf{Type-1 Diabetes}~\cite{xie2018simglucose} aims to test the control of blood glucose levels (observation) in type-1 diabetic patients by injecting insulin (action).
The non-stationarity is reflected in the continuous variation of food intake, which affects blood glucose levels and the dynamics of tasks; 
\item \textbf{MuJoCo}~\cite{Todorov2012MuJoCoAP}, widely adopted in NSRL, has only parametric diversity. We considered two MuJoCo scenarios: a) Evolution tasks with changing reward functions (e.g., Walker-Vel); b) Evolution tasks with changing dynamics (e.g., Cheetah-Damping).
\end{enumerate}

\textbf{Baselines.} We compared \modelname with four competitive NSRL baselines: CEMRL~\cite{bing2023meta2}, TRIO~\cite{Poiani2021MetaReinforcementLB}, SeCBAD~\cite{chen2022adaptive}, and COREP~\cite{zhang2024tackling}, which utilize the Gaussian Mixture Model, Gaussian Process, reward function, and causal graph to model the task evolution process, respectively. Additionally, we compared SAC~\cite{Haarnoja2018SoftAA}, PEARL~\cite{rakelly2019efficient} and RL$^2$~\cite{duan2016rl} to highlight the importance of task inference and modeling of the task evolution process, respectively. 
\begin{figure*}[ht]
\vskip -0.15in
\centering
\subfigure{
\includegraphics[width=0.23\textwidth]{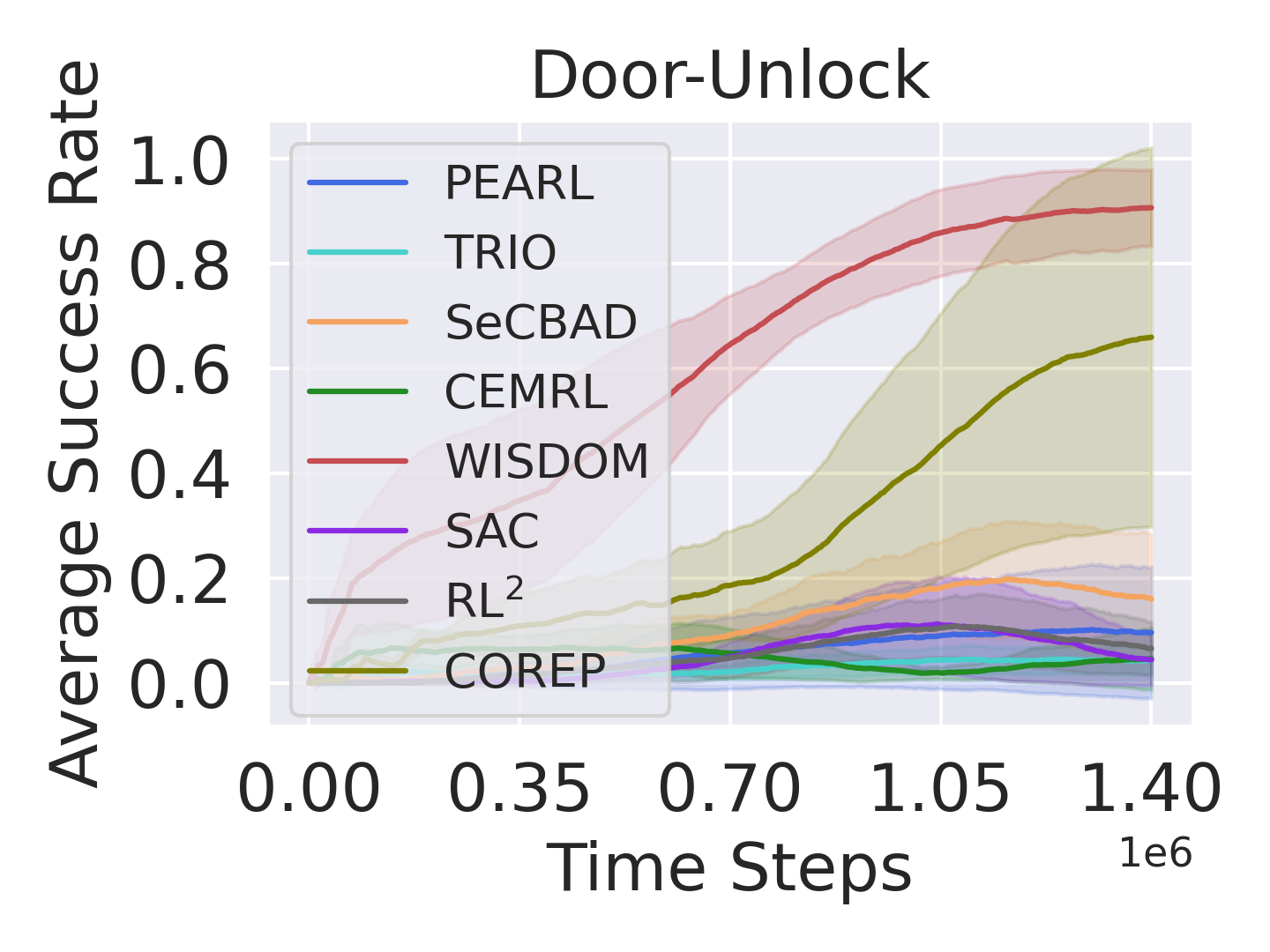}
}
\subfigure{
\includegraphics[width=0.23\textwidth]{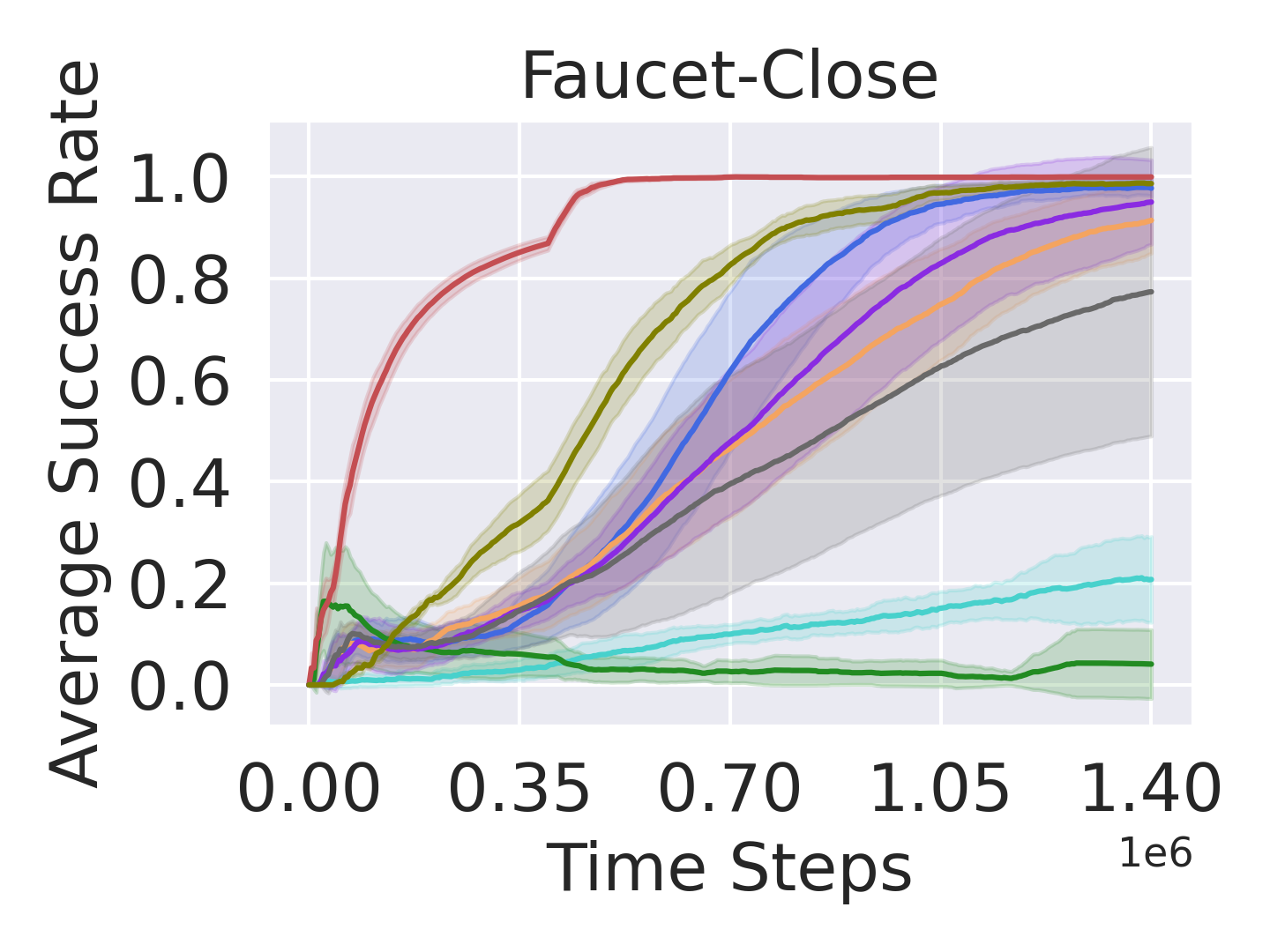}
}
\subfigure{
\includegraphics[width=0.23\textwidth]{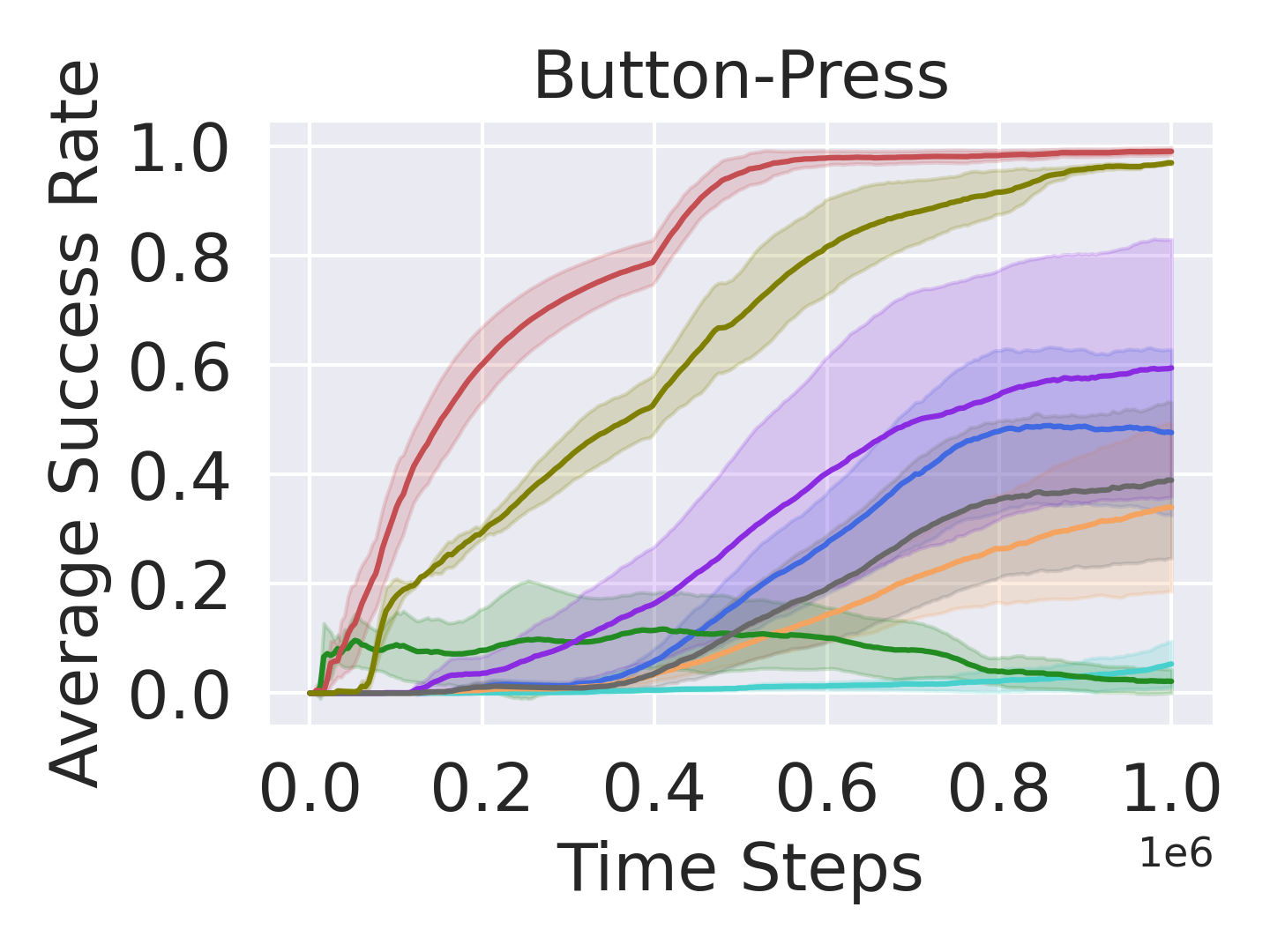}
}
\subfigure{
\includegraphics[width=0.23\textwidth]{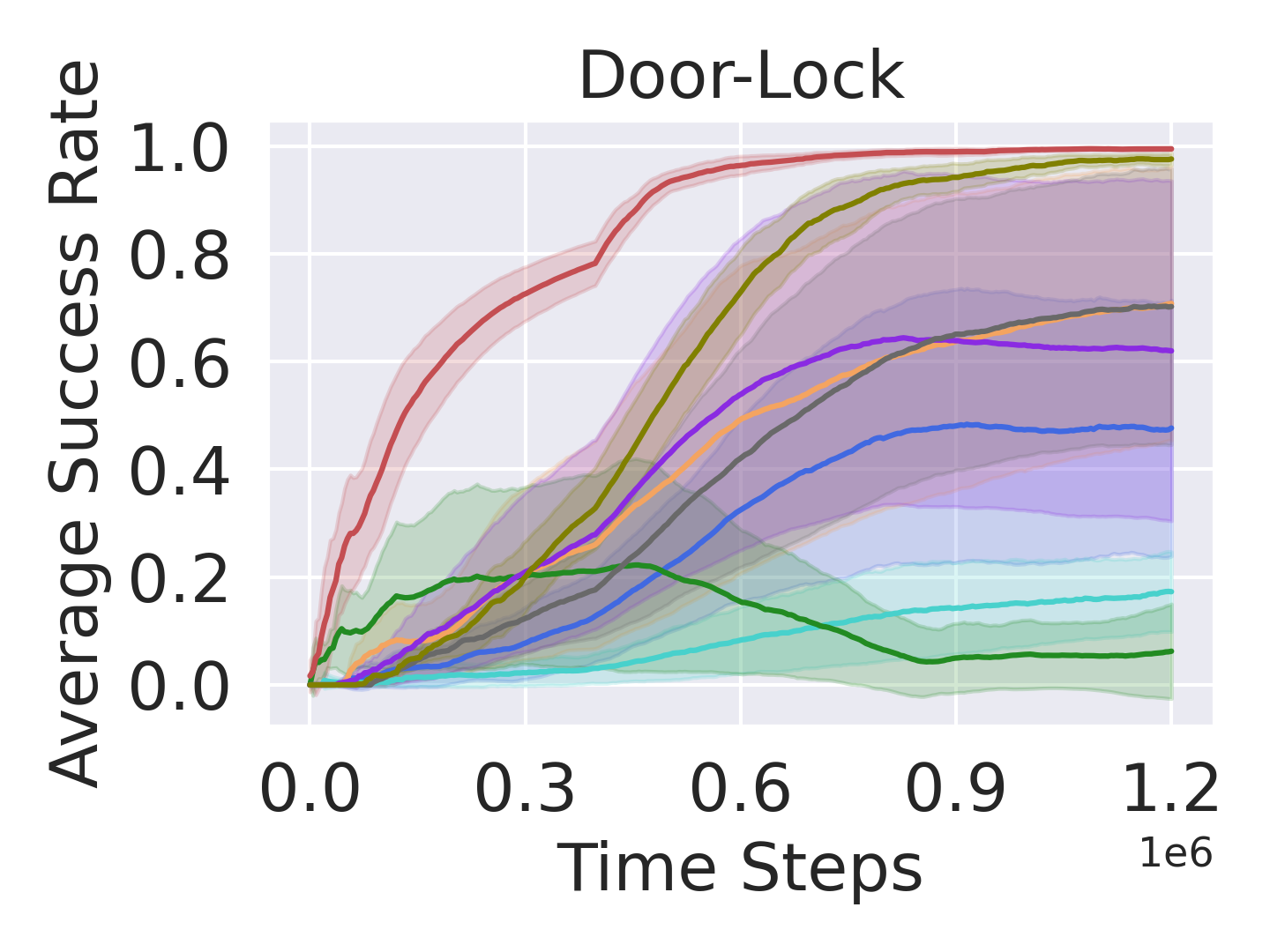}
}
\vskip -0.15in
\subfigure{
\includegraphics[width=0.23\textwidth]{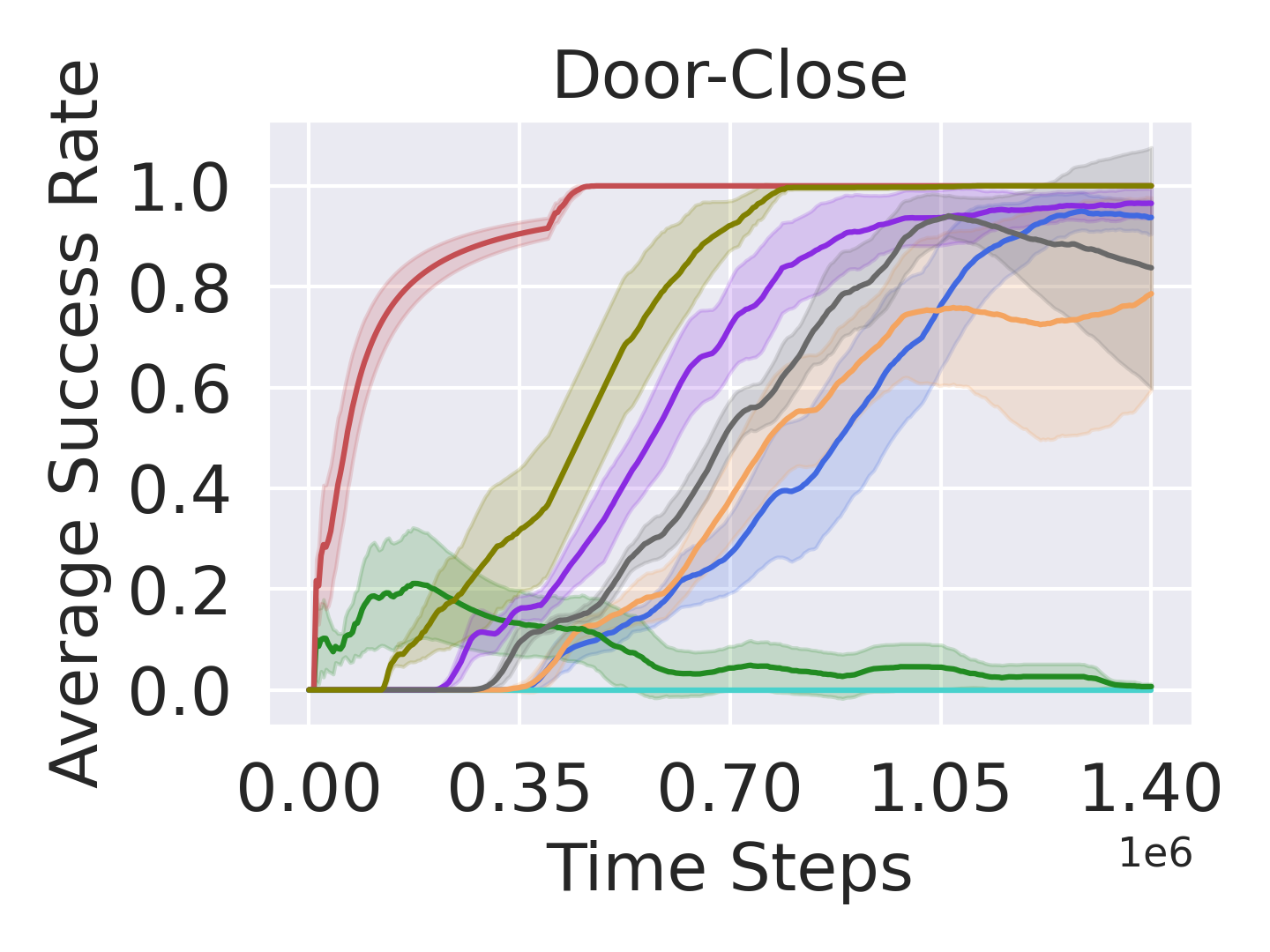}
}
\subfigure{
\includegraphics[width=0.23\textwidth]{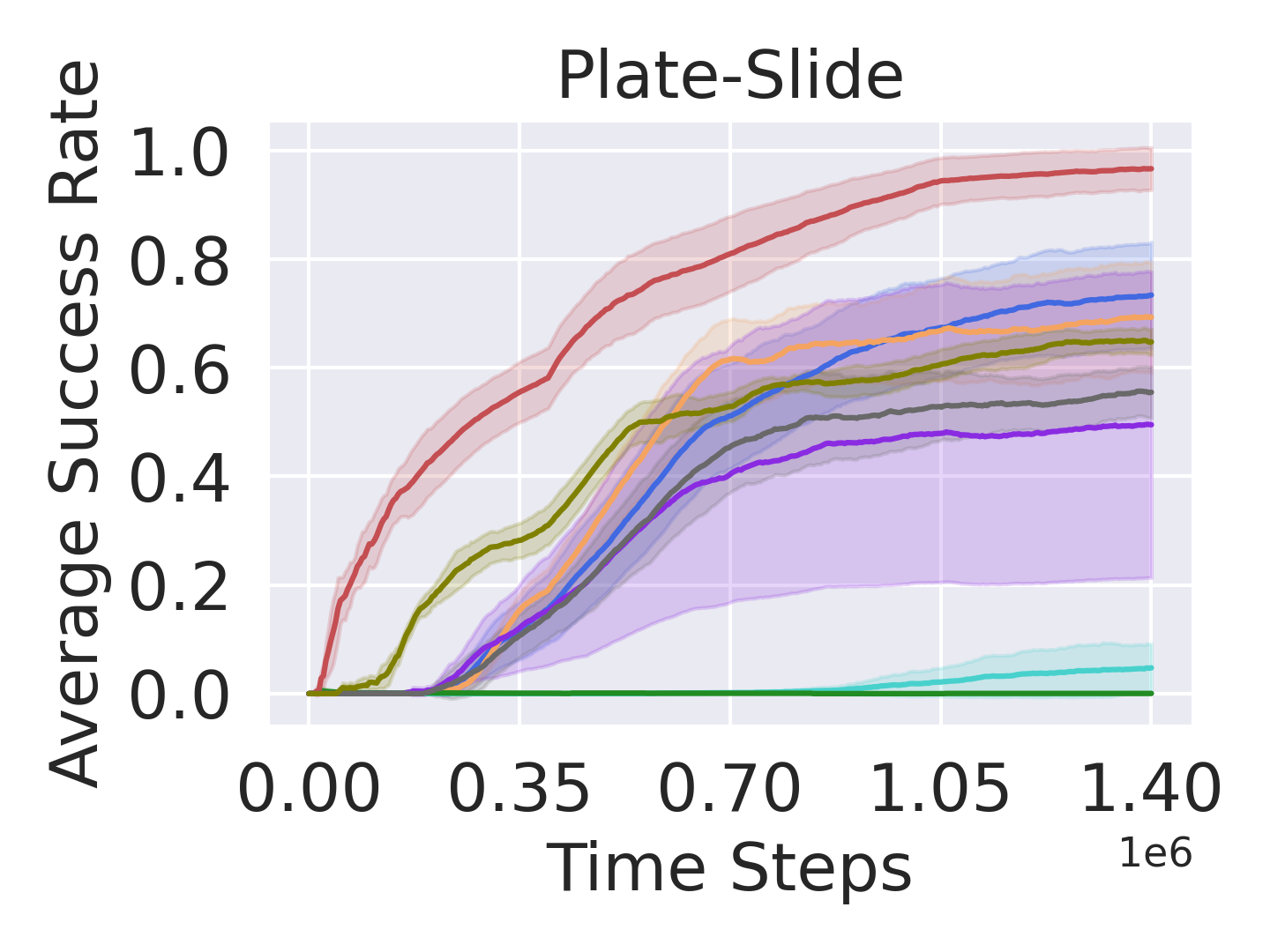}
}
\subfigure{
\includegraphics[width=0.23\textwidth]{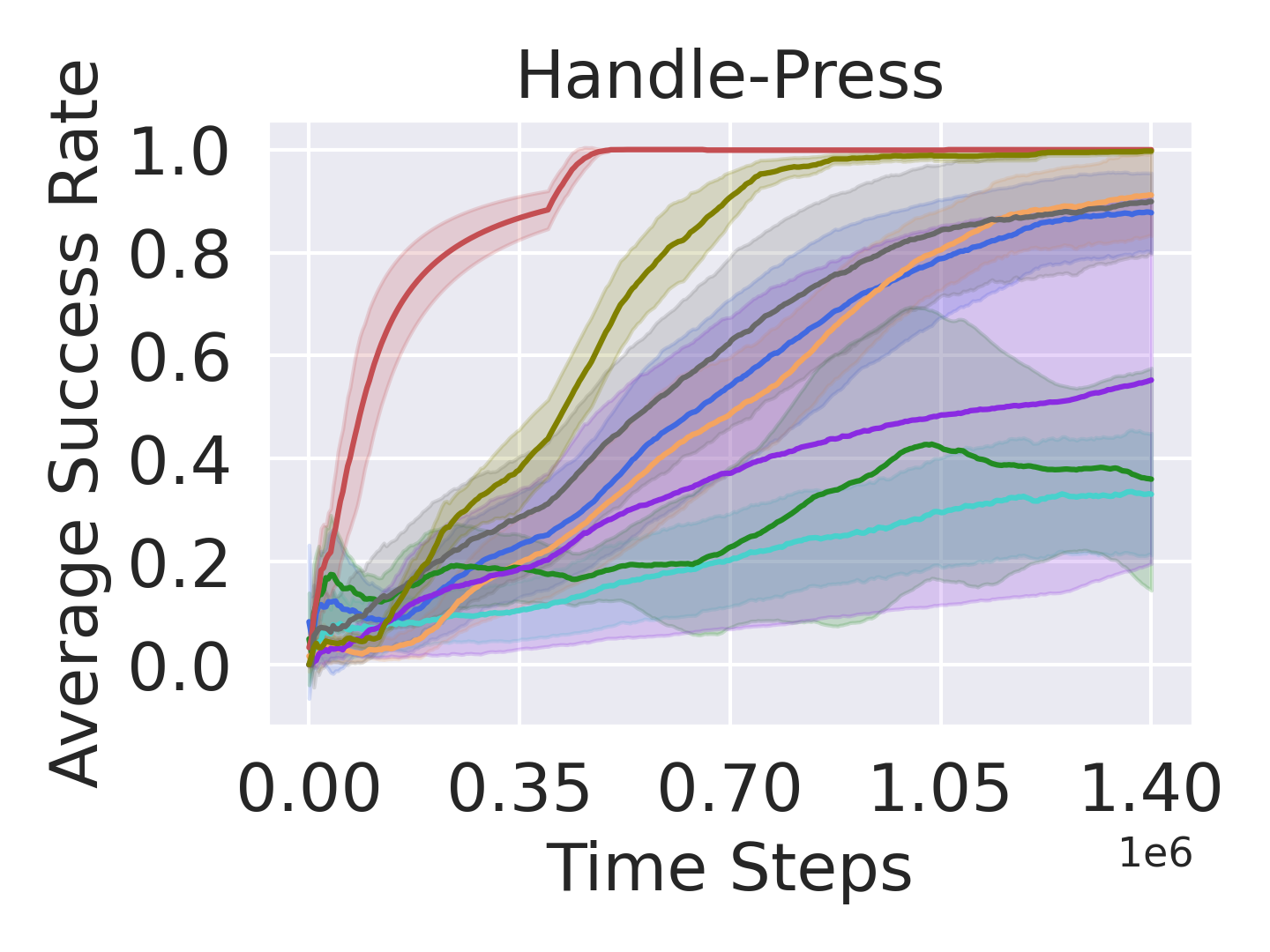}
}
\subfigure{
\includegraphics[width=0.23\textwidth]{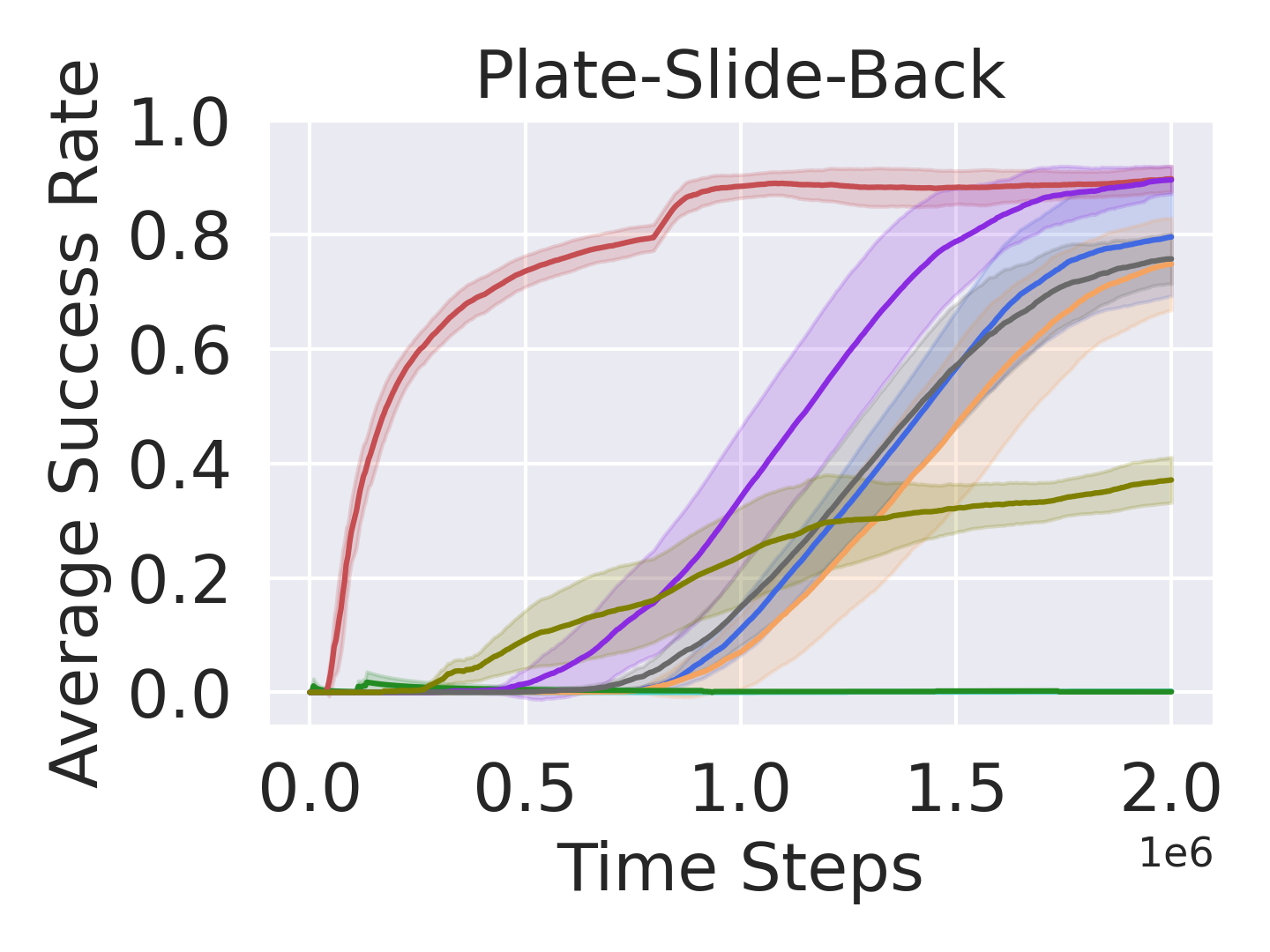}
}
\vskip -0.15in
\caption{Testing average performance on Meta-World over 6 random seeds. Our \modelname achieves rapid convergence and exhibits excellent asymptotic performance.
}
\vskip -0.1in
\label{fig:meta-world}
\end{figure*}

\textbf{Meta-World results.} Fig.~\ref{fig:meta-world} and Table~\ref{tab:ml1} illustrate the testing curves and converged performance, respectively. As in Fig.~\ref{fig:meta-world}, \modelname consistently demonstrates exceptional adaptation efficiency across all environments, surpassing all baseline methods. As in Table~\ref{tab:ml1}, \modelname achieves the highest average success rate with minimal variance in most environments, significantly contributing to effective noise filtering and the retention of fundamental low-frequency components. Although SeCBAD outperforms PEARL in Door-Unlock and Door-Lock, it falls behind PEARL regarding rapid adaptability in most environments. In Meta-World, where the evolutionary factors are relatively complex due to the potential variations in the three-dimensional coordinates of the target position, agents require substantial trial and error during the early training due to fluctuating rewards. Consequently, SeCBAD faces challenges in relying solely on rewards to identify whether task evolution has occurred. Surprisingly, PEARL and RL$^2$ have impressive overall converged performance, indicating their potential to adapt to non-stationary tasks by solely meta-learning a policy, even without explicitly modeling the inherent changing characteristics of task evolution. The competitive SAC further illustrates that inaccuracies in the inference and modeling of non-stationary tasks in more challenging environments may undermine the adaptability of the policy. When CEMRL attempts to cluster complex task distributions into a few categories, representation collapse often occurs. Conversely, clustering into a larger number of categories significantly increases the complexity. The selection of kernel functions in TRIO has a similar dilemma, causing wholly poor performance.
\begin{table*}[htbp]
\centering
\vskip -0.15in
\huge
\caption{Converged average test success rate $\pm$ standard error $(\%)$ on Meta-World.}
\begin{adjustbox}{max width=0.95\textwidth}
\begin{tabular}{@{}lccccccccccc@{}}
\toprule

& \multicolumn{1}{c}{Door-Unlock} & \multicolumn{1}{c}{Faucet-Close} & \multicolumn{1}{c}{Button-Press} & \multicolumn{1}{c}{Door-Lock} & \multicolumn{1}{c}{Door-Close} & \multicolumn{1}{c}{Plate-Slide} & \multicolumn{1}{c}{Handle-Press} & \multicolumn{1}{c}{Plate-Slide-Back}\\
\midrule
CEMRL & 4.08$\pm{8.74}$ & 0.17$\pm{0.37}$ & 1.83$\pm{4.10}$ & 6.67$\pm{14.14}$ & 0.42$\pm{0.93}$ & 0.00$\pm{0.00}$ & 45.17$\pm{38.58}$ & 0.00$\pm{0.00}$\\
TRIO & 3.92$\pm{6.46}$ & 20.17$\pm{22.14}$ & 10.42$\pm{16.36}$ & 20.75$\pm{19.64}$ & 0.00$\pm{0.00}$ & 6.42$\pm{10.39}$ & 33.25$\pm{26.92}$ & 0.20$\pm{\textbf{0.28}}$\\
PEARL & 10.25$\pm{19.31}$ & 96.33$\pm{6.88}$ & 39.42$\pm{32.54}$ & 50.67$\pm{32.74}$ & 88.58$\pm{14.64}$ & 73.50$\pm{17.18}$ & 85.33$\pm{20.78}$ & 82.50$\pm{13.73}$\\
SeCBAD & 11.58$\pm{17.01}$ & 93.42$\pm{8.99}$ & 36.58$\pm{32.46}$ & 72.50$\pm{31.95}$ & 89.33$\pm{18.20}$ & 71.50$\pm{19.93}$ & 93.92$\pm{12.37}$ & 79.53$\pm{11.38}$\\
COREP & 67.50$\pm{45.96}$ & 98.00$\pm{2.21}$& 96.83$\pm{4.24}$ & 98.67$\pm{1.89}$ &\textbf{100.00}$\pm{\textbf{0.00}}$& 64.17$\pm{6.86}$&\textbf{100.00}$\pm{\textbf{0.00}}$&
43.00$\pm{10.19}$\\
SAC & 1.67$\pm{\textbf{3.14}}$ & 98.33$\pm{3.73}$ & 62.83$\pm{38.22}$ & 62.67$\pm{39.33}$ &
\textbf{100.00}$\pm{\textbf{0.00}}$ & 50.00$\pm{37.15}$& 61.25$\pm{45.63}$& 90.03$\pm{5.98}$\\
RL$^2$ & 5.75$\pm{9.38}$ & 81.00$\pm{36.88}$ & 39.83$\pm{28.49}$ & 70.67$\pm{35.14}$ & 81.83$\pm{36.92}$& 56.92$\pm{15.71}$& 94.75$\pm{9.65}$& 79.23$\pm{9.85}$\\
\textbf{\modelname} & \textbf{91.58}$\pm{9.36}$ & \textbf{99.92}$\pm{\textbf{0.19}}$ & \textbf{99.42}$\pm{\textbf{1.30}}$ & \textbf{99.67}$\pm{\textbf{0.75}}$ & \textbf{100.00}$\pm{\textbf{0.00}}$& \textbf{96.50}$\pm{\textbf{6.04}}$& \textbf{100.00}$\pm{\textbf{0.00}}$ & \textbf{90.57}$\pm{7.66}$\\
\bottomrule
\end{tabular}
\label{tab:ml1}
\end{adjustbox}
\vskip -0.15in
\end{table*}

\begin{wrapfigure}{r}{0.5\textwidth}
\centering
\setlength{\abovecaptionskip}{0cm}
\vskip -0.25in
\subfigure{
\includegraphics[width=0.23\textwidth]{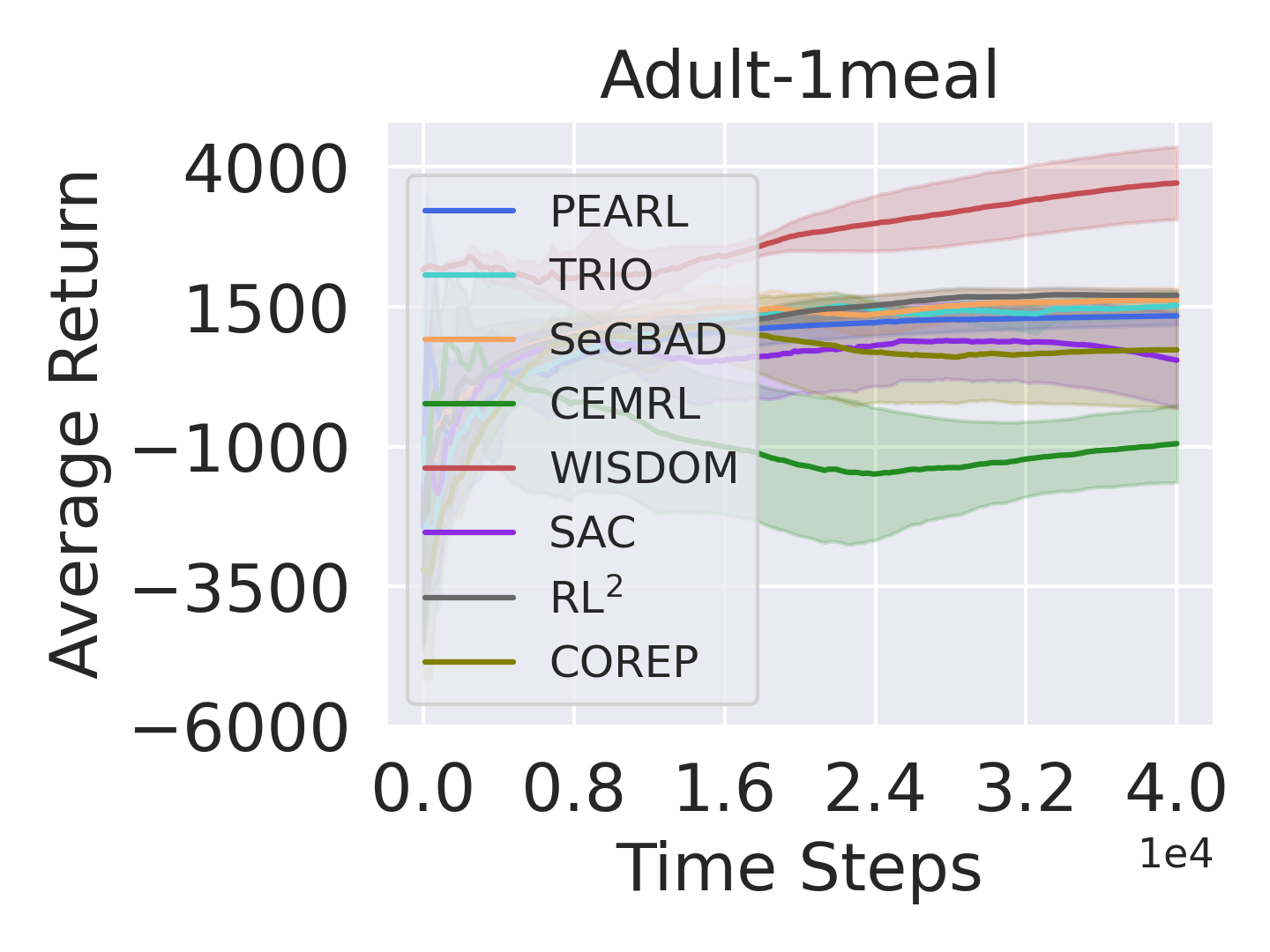}
}
\subfigure{
\includegraphics[width=0.23\textwidth]{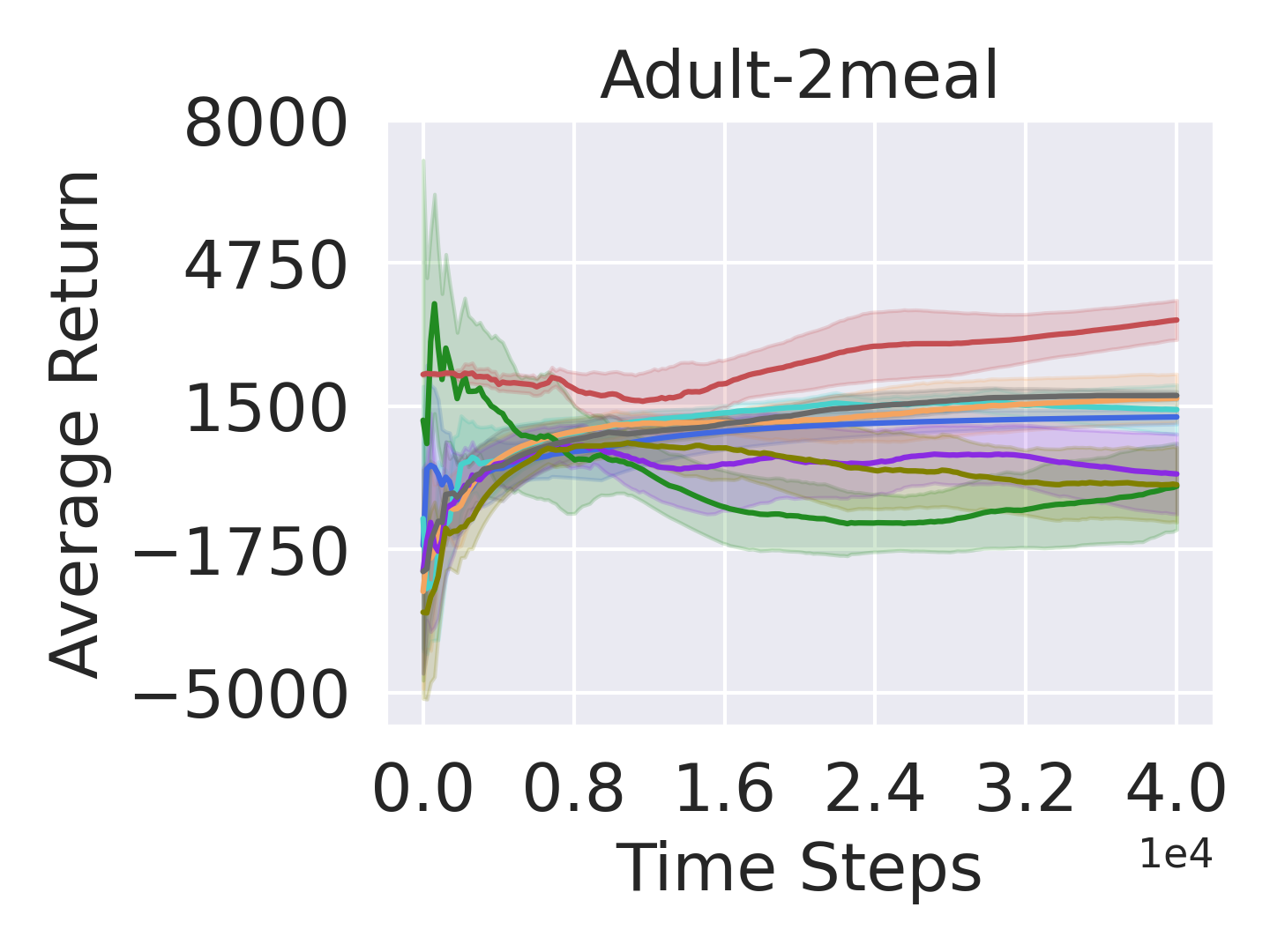}
}
\vskip -0.05in
\caption{Testing average return on glucose control environments over 6 random seeds.}
\label{fig:medicine}
\vskip -0.15in
\end{wrapfigure}
\textbf{Type-1 Diabete results.} We further examined the adaptability of all models in a more realistic environment that investigates the relationship between the control of blood glucose levels and insulin injections in diabetic patients. As observed in Fig.~\ref{fig:medicine}, \modelname demonstrates notably more efficient adaptability, indicating that wavelet task representations can effectively capture the trend of changing blood glucose levels and aid in selecting appropriate insulin dosages to maintain normal levels. Moreover, SeCBAD and CEMRL struggle to adapt to environments with changing food intake, such as environments where both lunch and dinner (2 meal) quantities vary over time. This difficulty may stem from the limited distinctiveness of the trajectory data, making it challenging to rely on rewards to determine the evolution moment or to cluster the trajectories into limited categories.

\begin{wrapfigure}{r}{0.5\textwidth}
\centering
\setlength{\abovecaptionskip}{0cm}
\vskip -0.25in
\subfigure{
\includegraphics[width=0.23\textwidth]{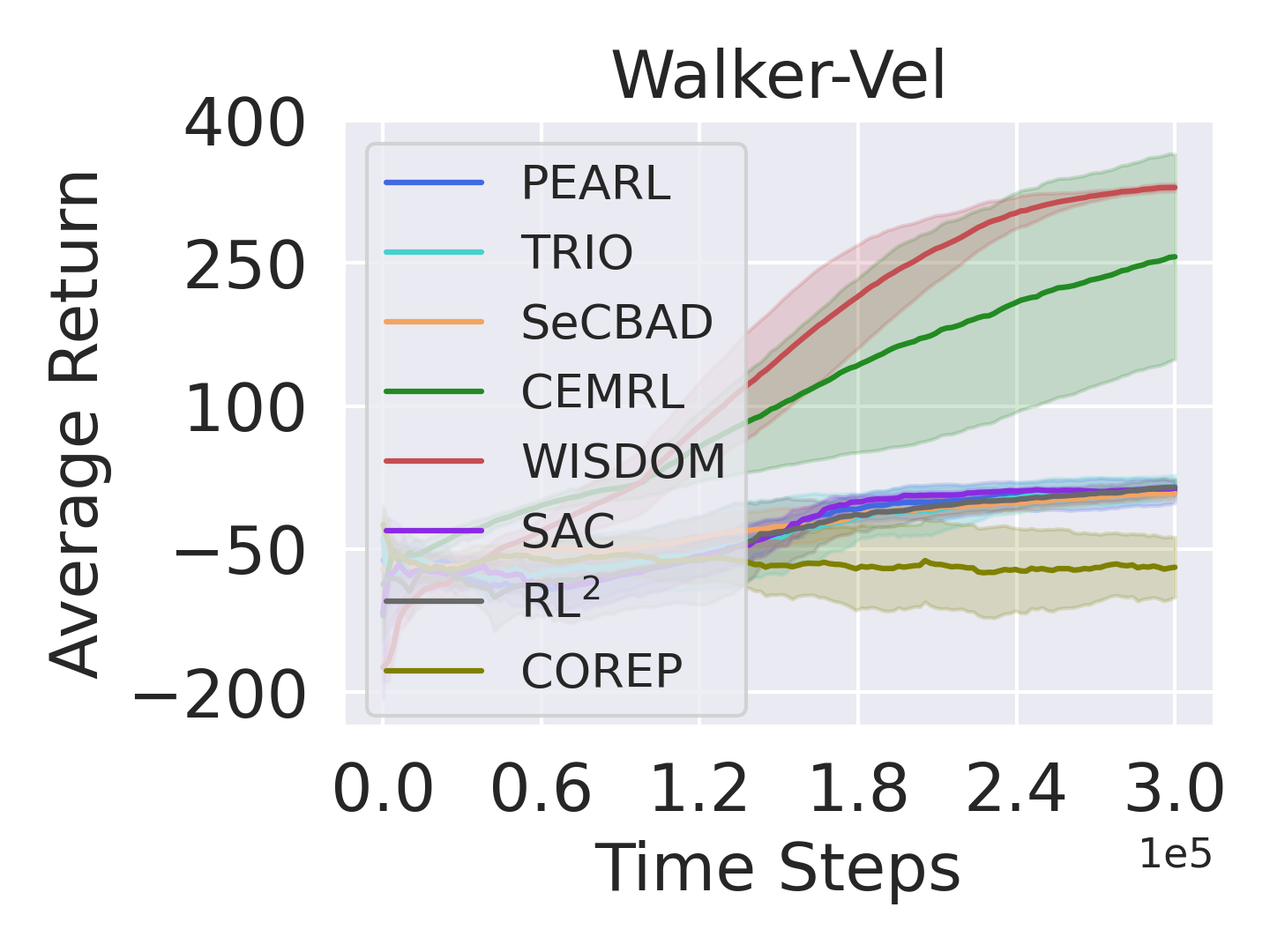}
}
\subfigure{
\includegraphics[width=0.23\textwidth]{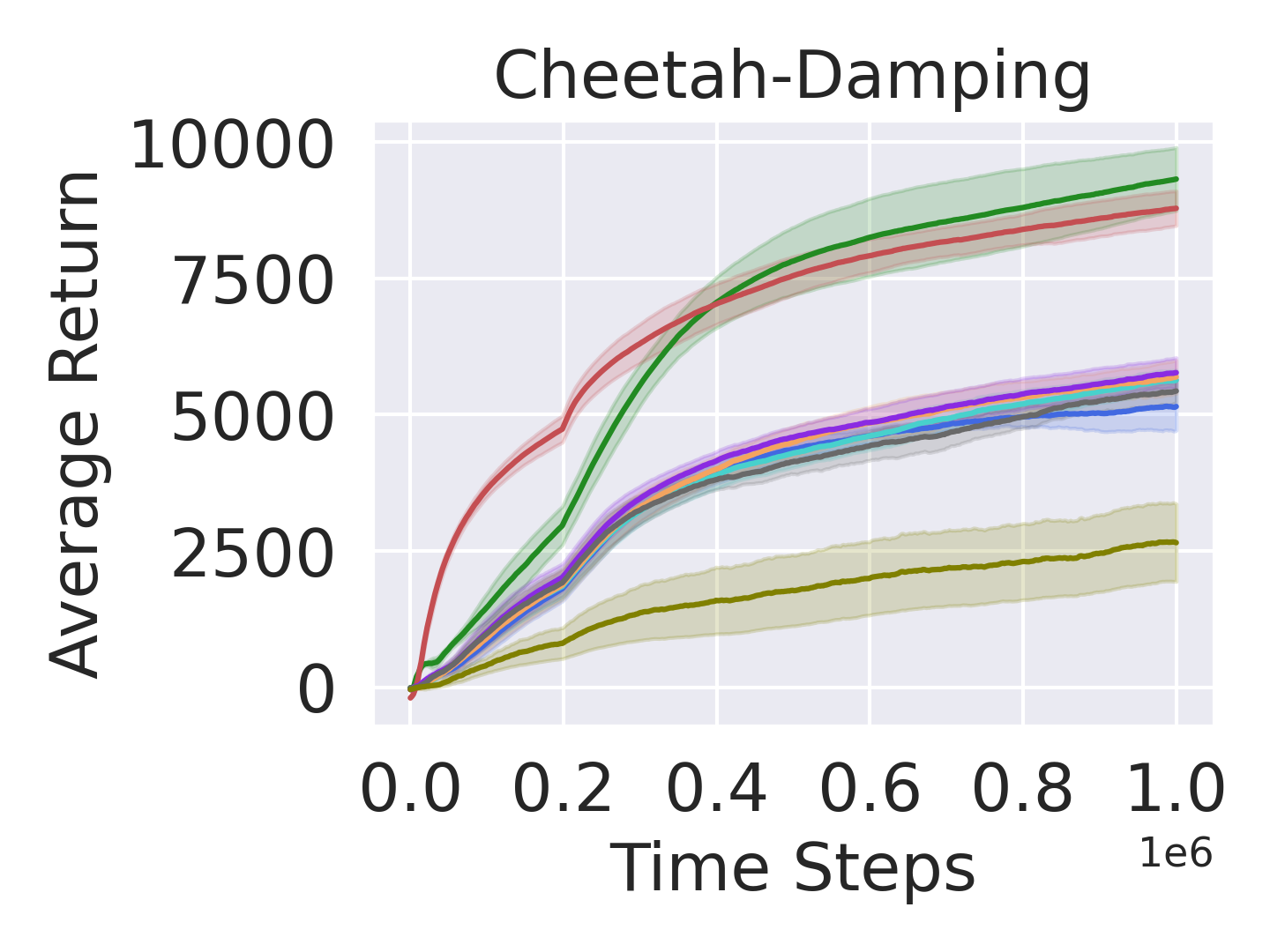}
}
\vskip -0.05in
\caption{Testing return on MuJoCo over 6 seeds.}
\label{fig:mujoco}
\vskip -0.15in
\end{wrapfigure}
\textbf{MuJoCo results.} Following CEMRL, the non-stationarity in MuJoCo is simulated by dynamically adjusting parameters, such as target velocity. However, these adjustments are relatively small, leading to a narrower distribution of non-stationary tasks. Consequently, these tasks tend to be less challenging compared to those in Meta-World. Besides, the lower state dimension of MuJoCo results in relatively simple relationships between the graph nodes, which explains why COREP's advantage over other baselines is not as significant as in Meta-World. In Fig.~\ref{fig:mujoco}, CEMRL demonstrates enhanced capability in clustering the task distribution, which facilitates precise classification of data in the replay buffer and effective extraction of task-specific representations. Since both TRIO and SeCBAD explicitly model the task evolution process, and TRIO additionally incorporates a recurrent encoder to capture and leverage history information, they converge faster than PEARL and RL$^2$. Significantly, \modelname shows consistently rapid adaptability and enhanced final performance.

\begin{figure*}[htbp]
\centering
\vskip -0.2in
\subfigure[Ablation study of WISDOM.]{
\includegraphics[width=0.23\textwidth]{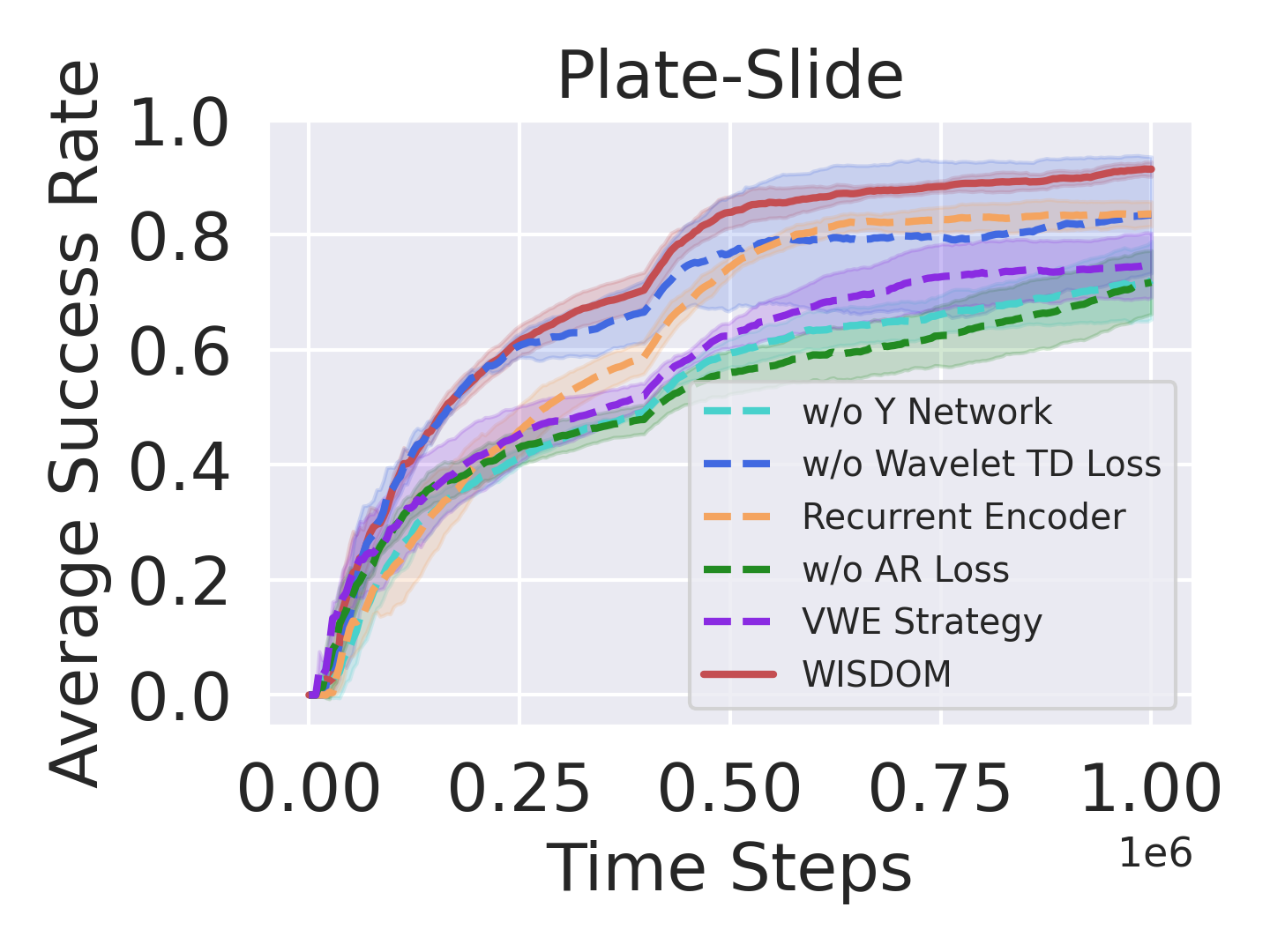}
\includegraphics[width=0.23\textwidth]{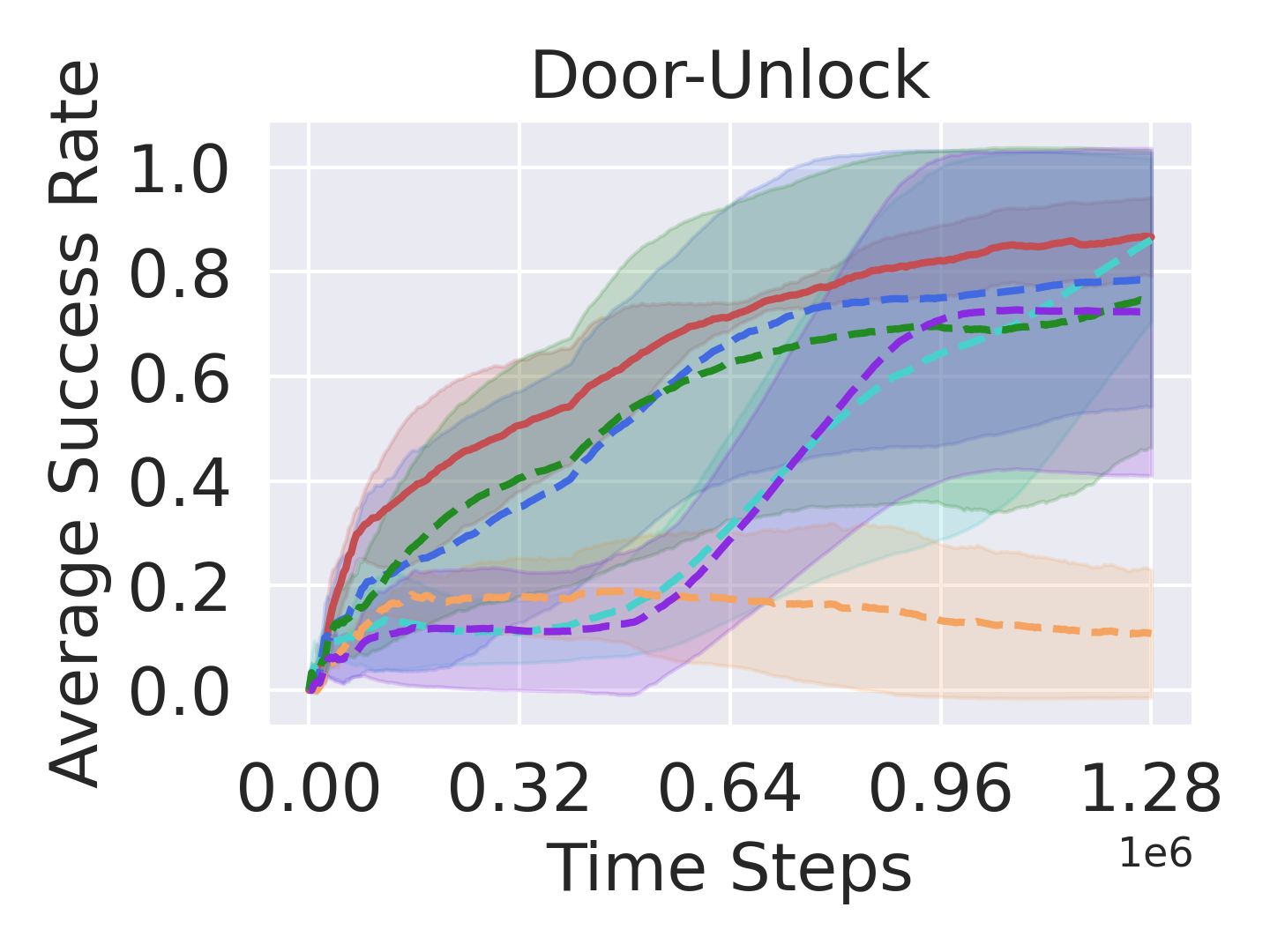}
\label{fig:ablation}
}
\subfigure[Impact of N-S degrees.]{
\includegraphics[width=0.23\textwidth]{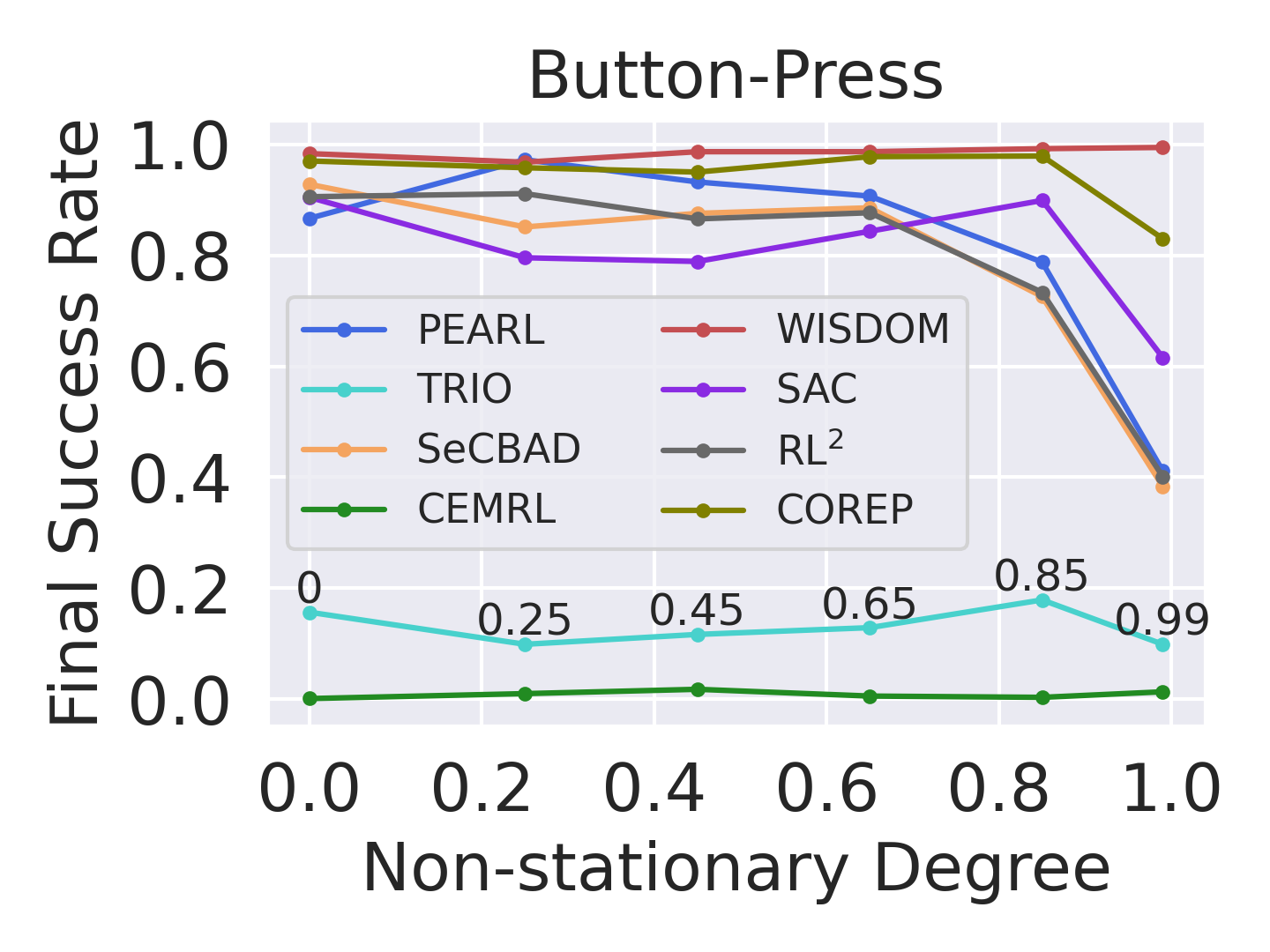}
\label{fig:degree}
}
\subfigure[Impact of different RL.]{
\includegraphics[width=0.23\textwidth]{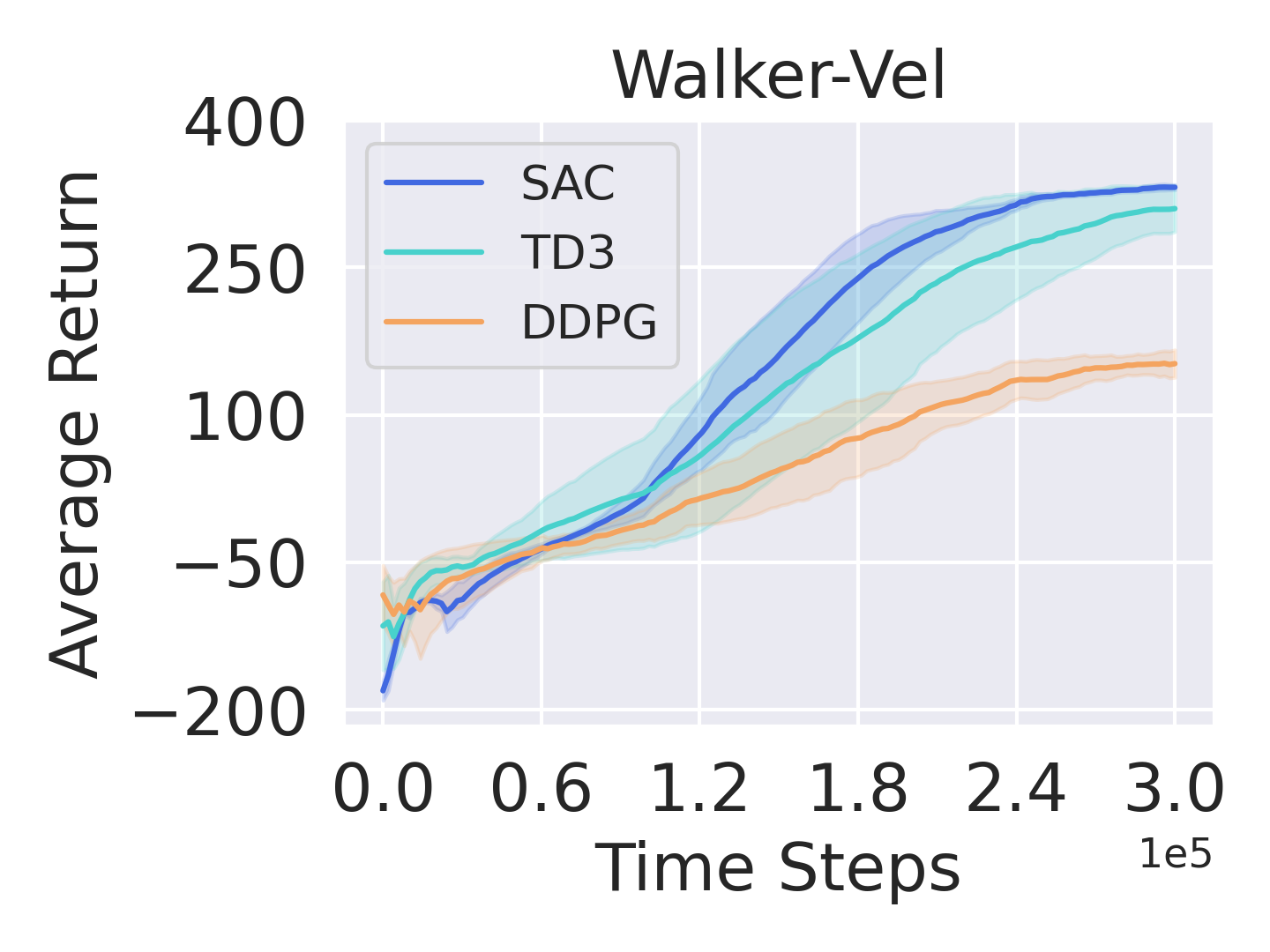}
\label{fig:backbone}
}
\vskip -0.1in
\caption{Ablation study and evaluation of different non-stationary (N-S) degrees and RL backbones.}
\vskip -0.1in
\end{figure*}
\textbf{Ablation studies.} We conduct ablation studies to validate the role of the wavelet representation network ($Y$ network) and its optimization objective, and the structure of the context encoder. In Fig.~\ref{fig:ablation}, the $Y$ network provides significant gains, further indicating that the wavelet task representation reflects non-stationary trends. The AR loss accelerates convergence and improves final performance, and the wavelet TD loss stabilizes training and reduces variance. In contrast, the RNN encoder tends to forget changes over time and is susceptible to gradient vanishing, whereas the MLP encoder (WISDOM) shows greater stability. Since each $z$ in $\mathbf{z}$ is generally multidimensional, each dimension of $z$ can be viewed as a variate in the multivariate time series. In addition to performing feature concatenation at each time step when performing DWT as WISDOM, we investigate a \textit{Variable-Wise Encoding (VWE)} strategy which applies DWT to the sequences constructed from each variate independently. The \emph{VWE} strategy leads our model to converge more slowly and yields inferior final results. We attribute this to the loss of cross-variate interaction modeling caused by disrupting the dependencies among variates. 
In the time series domain, this \emph{VWE} strategy often necessitates additional structural components to re-establish inter-variate relationships.

\textbf{Adaptability analysis of different non-stationary degrees and RL backbones.} The non-stationary degrees\footnote{Measured as $(T-\overline{T_h})/T\in[0,1)$, where $T$ is the total period and $\overline{T_h}$ denotes the mean of the stochastic period of an MDP $\mathcal{M}_{\omega_h}$. Larger values of degrees indicate severer non-stationarity.} we set in Meta-World, MuJoCo, and Type-1 Diabetes are 0.99, 0.97, and 0.7, respectively. In Fig.~\ref{fig:degree}, most models' final performance shows a downward trend as non-stationarity increases. Contrastively, our \modelname demonstrates remarkable and consistent adaptability, highlighting the more expressive wavelet task representation and robust policy.
In Fig.~\ref{fig:backbone}, \modelname can quickly adapt and converge utilizing various RL algorithms. Due to limited exploration, DDPG often converges to a local optimum. For fair comparisons, all models employ SAC as the backbone. 

\begin{wrapfigure}{r}{0.25\textwidth}
\centering
\setlength{\abovecaptionskip}{0cm}
\vskip -0.25in
\subfigure{
\includegraphics[width=0.23\textwidth]{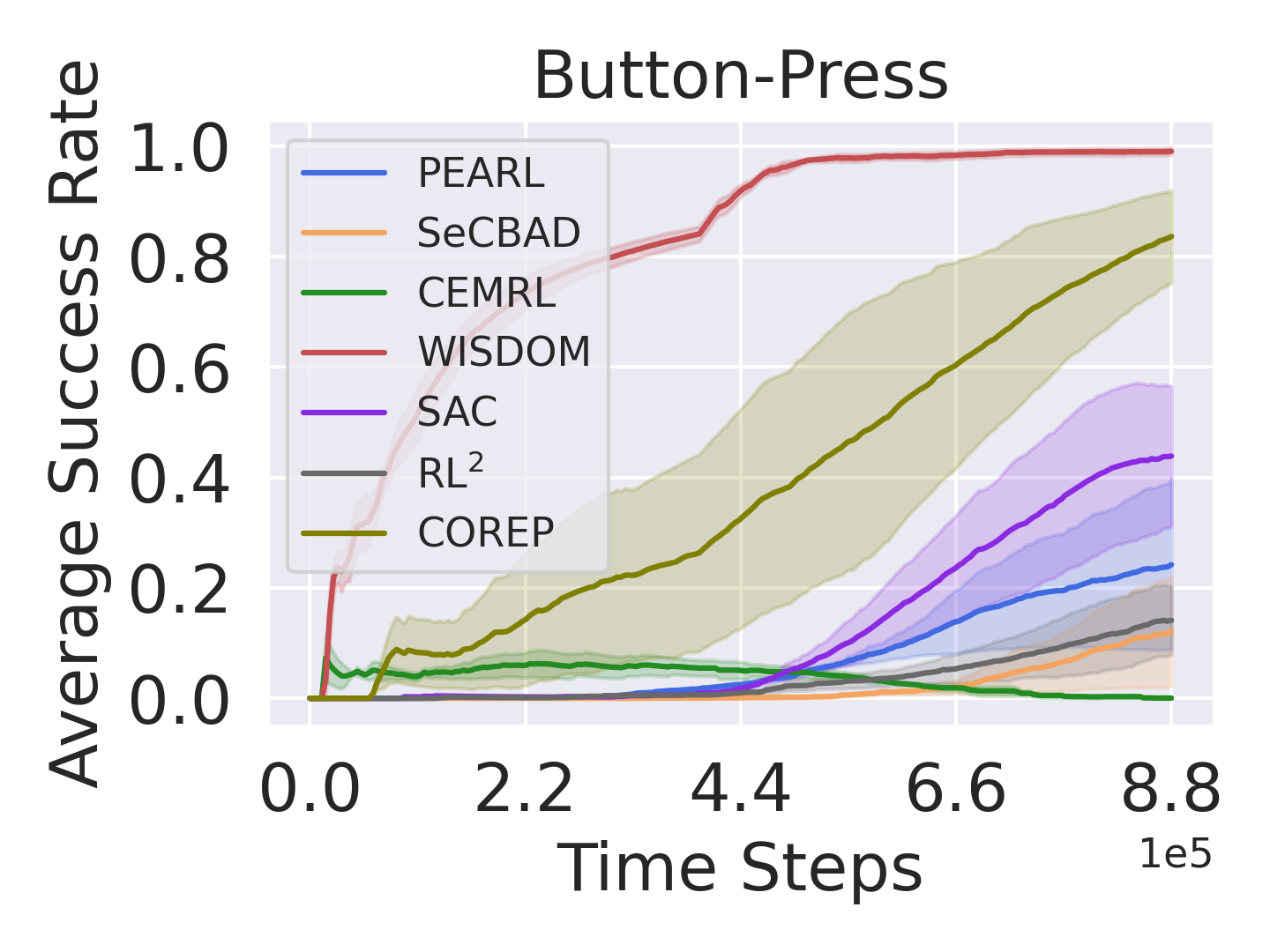}
}
\vskip -0.05in
\caption{Evaluation of robustness to noise.}
\label{fig:snr}
\vskip -0.15in
\end{wrapfigure}
\textbf{Robustness to noise for \modelname and baseline methods.} We conducted experiments by injecting Gaussian noise $N(0,1)$ into states to evaluate robustness. As illustrated in Fig.~\ref{fig:snr}, all baselines exhibit slower convergence and reduced performance under noisy conditions. In contrast, our \modelname maintains the highest success rate and rapid convergence. We credit this to the wavelet-based representation learning process, which not only suppresses noise but also retains fast-varying and task-relevant signals more effectively, thereby improving the signal-to-noise ratio. 

\begin{wrapfigure}{r}{0.5\textwidth}
\centering
\setlength{\abovecaptionskip}{0cm}
\vskip -0.2in
\subfigure{
\includegraphics[width=0.5\textwidth]{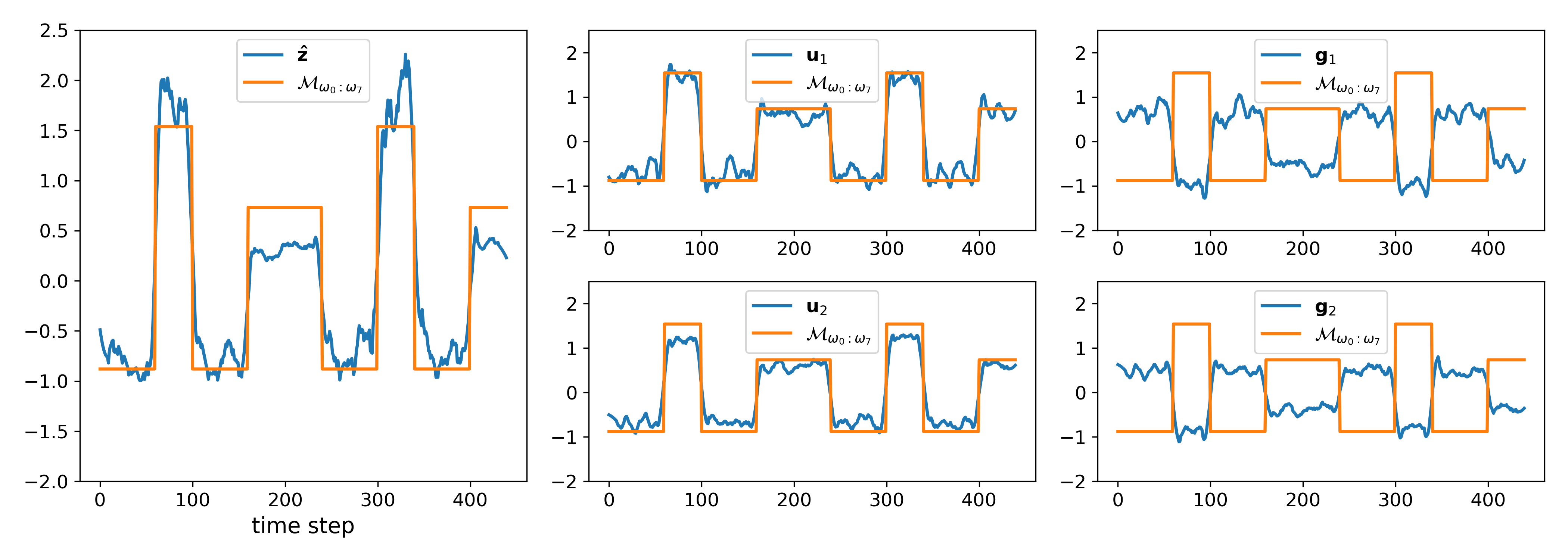}
}
\vskip -0.05in
\caption{The case study illustrates that the wavelet task representation $\hat{z}$ predicted by the $Y_\phi$ network accurately reflect the true changes of task $\mathcal{M}_\omega$.}
\label{fig:case}
\vskip -0.15in
\end{wrapfigure}

\textbf{Case study.} We visualize the changes of representations tracked by the $Y_\phi$ network to validate its ability to capture intrinsic non-stationary changes. In Fig.~\ref{fig:case}, the velocity of the agent in Inverted Double Pendulum undergoes 7 changes over 440 consecutive time steps, generating a non-stationary sequence $\mathcal{M}_{\omega_{0:7}}$ with varying state transition dynamics. The approximation coefficients $\mathbf{u}$ capture the overall trend of task evolution, while the detail coefficients $\mathbf{g}$ capture slightly sharper variations. Due to the orthogonality of the initialized Haar filters, the changes captured by $\mathbf{g}$ are reversed~\cite{pattanaik1995haar}. By combining $\mathbf{g}$ and $\mathbf{u_2}$, $Y_{\phi}$ ultimately produces the predicted $\mathbf{\hat{z}}$, which closely aligns with the true trend of task evolution.

\section{Conclusion, Limitation and Future Work}
\label{Conclusion}
We introduced WISDOM, a novel approach to address the challenge of non-stationarity in RL by leveraging a learnable wavelet representation network to capture the trends of task evolution in the wavelet domain, facilitating flexible adaptation to complex non-stationary tasks with stochastic evolving periods. \modelname captures more intrinsic evolving features of non-stationary task changes, learns to predict the evolving trend, and ensures efficient and improved policy learning. Experimental results on Meta-World, MuJoCo, and Type-1 Diabetes benchmarks demonstrate its superb adaptability and performance, highlighting its robustness and efficiency in handling non-stationary tasks.

Although \modelname is effective and easy to implement, it presents certain limitations. The compactness of task representations may be influenced by decomposition levels in the wavelet representation network, necessitating tuning appropriate decomposition levels based on specific environments. Future directions include: 1) Developing adaptive mechanisms for dynamically adjusting decomposition levels; and 2) Exploring effective selection of detail coefficients for rapid adaptation.

\bibliographystyle{ACM-Reference-Format}
\bibliography{references}


\begin{thebibliography}{49}


\ifx \showCODEN    \undefined \def \showCODEN     #1{\unskip}     \fi
\ifx \showDOI      \undefined \def \showDOI       #1{#1}\fi
\ifx \showISBNx    \undefined \def \showISBNx     #1{\unskip}     \fi
\ifx \showISBNxiii \undefined \def \showISBNxiii  #1{\unskip}     \fi
\ifx \showISSN     \undefined \def \showISSN      #1{\unskip}     \fi
\ifx \showLCCN     \undefined \def \showLCCN      #1{\unskip}     \fi
\ifx \shownote     \undefined \def \shownote      #1{#1}          \fi
\ifx \showarticletitle \undefined \def \showarticletitle #1{#1}   \fi
\ifx \showURL      \undefined \def \showURL       {\relax}        \fi
\providecommand\bibfield[2]{#2}
\providecommand\bibinfo[2]{#2}
\providecommand\natexlab[1]{#1}
\providecommand\showeprint[2][]{arXiv:#2}

\bibitem[Achiam et~al\mbox{.}(2017)]%
        {achiam2017constrained}
\bibfield{author}{\bibinfo{person}{Joshua Achiam}, \bibinfo{person}{David
  Held}, \bibinfo{person}{Aviv Tamar}, {and} \bibinfo{person}{Pieter Abbeel}.}
  \bibinfo{year}{2017}\natexlab{}.
\newblock \showarticletitle{Constrained policy optimization}. In
  \bibinfo{booktitle}{\emph{International conference on machine learning}}.
  PMLR, \bibinfo{pages}{22--31}.
\newblock


\bibitem[Al-Shedivat et~al\mbox{.}(2018)]%
        {alshedivat2018continuous}
\bibfield{author}{\bibinfo{person}{Maruan Al-Shedivat}, \bibinfo{person}{Trapit
  Bansal}, \bibinfo{person}{Yura Burda}, \bibinfo{person}{Ilya Sutskever},
  \bibinfo{person}{Igor Mordatch}, {and} \bibinfo{person}{Pieter Abbeel}.}
  \bibinfo{year}{2018}\natexlab{}.
\newblock \showarticletitle{Continuous Adaptation via Meta-Learning in
  Nonstationary and Competitive Environments}. In
  \bibinfo{booktitle}{\emph{International Conference on Learning
  Representations (ICLR)}}.
\newblock


\bibitem[Barreto et~al\mbox{.}(2017)]%
        {barreto2017successor}
\bibfield{author}{\bibinfo{person}{Andr{\'e} Barreto}, \bibinfo{person}{Will
  Dabney}, \bibinfo{person}{R{\'e}mi Munos}, \bibinfo{person}{Jonathan~J Hunt},
  \bibinfo{person}{Tom Schaul}, \bibinfo{person}{Hado~P van Hasselt}, {and}
  \bibinfo{person}{David Silver}.} \bibinfo{year}{2017}\natexlab{}.
\newblock \showarticletitle{Successor features for transfer in reinforcement
  learning}.
\newblock \bibinfo{journal}{\emph{Advances in neural information processing
  systems}}  \bibinfo{volume}{30} (\bibinfo{year}{2017}).
\newblock


\bibitem[Basu et~al\mbox{.}(2023)]%
        {Basu2023OnTC}
\bibfield{author}{\bibinfo{person}{Sumana Basu}, \bibinfo{person}{Mark~A
  Legault}, \bibinfo{person}{Adriana Romero-Soriano}, {and}
  \bibinfo{person}{Doina Precup}.} \bibinfo{year}{2023}\natexlab{}.
\newblock \showarticletitle{On the Challenges of using Reinforcement Learning
  in Precision Drug Dosing: Delay and Prolongedness of Action Effects}. In
  \bibinfo{booktitle}{\emph{AAAI Conference on Artificial Intelligence}}.
\newblock
\urldef\tempurl%
\url{https://api.semanticscholar.org/CorpusID:255372521}
\showURL{%
\tempurl}


\bibitem[Bing et~al\mbox{.}(2023a)]%
        {bing2023meta2}
\bibfield{author}{\bibinfo{person}{Zhenshan Bing}, \bibinfo{person}{Lukas
  Knak}, \bibinfo{person}{Long Cheng}, \bibinfo{person}{Fabrice~O Morin},
  \bibinfo{person}{Kai Huang}, {and} \bibinfo{person}{Alois Knoll}.}
  \bibinfo{year}{2023}\natexlab{a}.
\newblock \showarticletitle{Meta-reinforcement learning in nonstationary and
  nonparametric environments}.
\newblock \bibinfo{journal}{\emph{IEEE Transactions on Neural Networks and
  Learning Systems}} (\bibinfo{year}{2023}).
\newblock


\bibitem[Bing et~al\mbox{.}(2023b)]%
        {bing2023meta}
\bibfield{author}{\bibinfo{person}{Zhenshan Bing}, \bibinfo{person}{David
  Lerch}, \bibinfo{person}{Kai Huang}, {and} \bibinfo{person}{Alois Knoll}.}
  \bibinfo{year}{2023}\natexlab{b}.
\newblock \showarticletitle{Meta-Reinforcement Learning in Non-Stationary and
  Dynamic Environments}.
\newblock \bibinfo{journal}{\emph{IEEE Transactions on Pattern Analysis \&
  Machine Intelligence}} \bibinfo{volume}{45}, \bibinfo{number}{03}
  (\bibinfo{year}{2023}), \bibinfo{pages}{3476--3491}.
\newblock


\bibitem[Bol{\'o}s et~al\mbox{.}(2020)]%
        {bolos2020new}
\bibfield{author}{\bibinfo{person}{Vicente~J Bol{\'o}s},
  \bibinfo{person}{Rafael Ben{\'\i}tez}, {and} \bibinfo{person}{Rom{\'a}n
  Ferrer}.} \bibinfo{year}{2020}\natexlab{}.
\newblock \showarticletitle{A new wavelet tool to quantify non-periodicity of
  non-stationary economic time series}.
\newblock \bibinfo{journal}{\emph{Mathematics}} \bibinfo{volume}{8},
  \bibinfo{number}{5} (\bibinfo{year}{2020}), \bibinfo{pages}{844}.
\newblock


\bibitem[Burrus et~al\mbox{.}(1998)]%
        {burrus1998wavelets}
\bibfield{author}{\bibinfo{person}{C~Sidney Burrus}, \bibinfo{person}{Ramesh~A
  Gopinath}, {and} \bibinfo{person}{Haitao Guo}.}
  \bibinfo{year}{1998}\natexlab{}.
\newblock \showarticletitle{Wavelets and wavelet transforms}.
\newblock \bibinfo{journal}{\emph{rice university, houston edition}}
  \bibinfo{volume}{98} (\bibinfo{year}{1998}).
\newblock


\bibitem[Chen et~al\mbox{.}(2022)]%
        {chen2022adaptive}
\bibfield{author}{\bibinfo{person}{Xiaoyu Chen}, \bibinfo{person}{Xiangming
  Zhu}, \bibinfo{person}{Yufeng Zheng}, \bibinfo{person}{Pushi Zhang},
  \bibinfo{person}{Li Zhao}, \bibinfo{person}{Wenxue Cheng},
  \bibinfo{person}{Peng Cheng}, \bibinfo{person}{Yongqiang Xiong},
  \bibinfo{person}{Tao Qin}, \bibinfo{person}{Jianyu Chen}, {et~al\mbox{.}}}
  \bibinfo{year}{2022}\natexlab{}.
\newblock \showarticletitle{An adaptive deep rl method for non-stationary
  environments with piecewise stable context}.
\newblock \bibinfo{journal}{\emph{Advances in Neural Information Processing
  Systems}}  \bibinfo{volume}{35} (\bibinfo{year}{2022}),
  \bibinfo{pages}{35449--35461}.
\newblock


\bibitem[Cooley and Tukey(1965)]%
        {cooley1965algorithm}
\bibfield{author}{\bibinfo{person}{James~W Cooley} {and}
  \bibinfo{person}{John~W Tukey}.} \bibinfo{year}{1965}\natexlab{}.
\newblock \showarticletitle{An algorithm for the machine calculation of complex
  Fourier series}.
\newblock \bibinfo{journal}{\emph{Mathematics of computation}}
  \bibinfo{volume}{19}, \bibinfo{number}{90} (\bibinfo{year}{1965}),
  \bibinfo{pages}{297--301}.
\newblock


\bibitem[Duan et~al\mbox{.}(2016)]%
        {duan2016rl}
\bibfield{author}{\bibinfo{person}{Yan Duan}, \bibinfo{person}{John Schulman},
  \bibinfo{person}{Xi Chen}, \bibinfo{person}{Peter~L Bartlett},
  \bibinfo{person}{Ilya Sutskever}, {and} \bibinfo{person}{Pieter Abbeel}.}
  \bibinfo{year}{2016}\natexlab{}.
\newblock \showarticletitle{RL$^{2}$: Fast reinforcement learning via slow
  reinforcement learning}.
\newblock \bibinfo{journal}{\emph{arXiv preprint arXiv:1611.02779}}
  (\bibinfo{year}{2016}).
\newblock


\bibitem[Gupta et~al\mbox{.}(2020)]%
        {gupta2020look}
\bibfield{author}{\bibinfo{person}{Gunshi Gupta}, \bibinfo{person}{Karmesh
  Yadav}, {and} \bibinfo{person}{Liam Paull}.} \bibinfo{year}{2020}\natexlab{}.
\newblock \showarticletitle{Look-ahead meta learning for continual learning}.
\newblock \bibinfo{journal}{\emph{Advances in Neural Information Processing
  Systems}}  \bibinfo{volume}{33} (\bibinfo{year}{2020}),
  \bibinfo{pages}{11588--11598}.
\newblock


\bibitem[Haarnoja et~al\mbox{.}(2018)]%
        {Haarnoja2018SoftAA}
\bibfield{author}{\bibinfo{person}{Tuomas Haarnoja}, \bibinfo{person}{Aurick
  Zhou}, \bibinfo{person}{Kristian Hartikainen}, \bibinfo{person}{G. Tucker},
  \bibinfo{person}{Sehoon Ha}, \bibinfo{person}{Jie Tan},
  \bibinfo{person}{Vikash Kumar}, \bibinfo{person}{Henry Zhu},
  \bibinfo{person}{Abhishek Gupta}, \bibinfo{person}{P. Abbeel}, {and}
  \bibinfo{person}{Sergey Levine}.} \bibinfo{year}{2018}\natexlab{}.
\newblock \showarticletitle{Soft Actor-Critic Algorithms and Applications}.
\newblock \bibinfo{journal}{\emph{ArXiv}}  \bibinfo{volume}{abs/1812.05905}
  (\bibinfo{year}{2018}).
\newblock
\urldef\tempurl%
\url{https://api.semanticscholar.org/CorpusID:55703664}
\showURL{%
\tempurl}


\bibitem[Kong et~al\mbox{.}(2022)]%
        {kong2022non}
\bibfield{author}{\bibinfo{person}{Fan Kong}, \bibinfo{person}{Yixin Zhang},
  {and} \bibinfo{person}{Yuanjin Zhang}.} \bibinfo{year}{2022}\natexlab{}.
\newblock \showarticletitle{Non-stationary response power spectrum
  determination of linear/non-linear systems endowed with fractional derivative
  elements via harmonic wavelet}.
\newblock \bibinfo{journal}{\emph{Mechanical Systems and Signal Processing}}
  \bibinfo{volume}{162} (\bibinfo{year}{2022}), \bibinfo{pages}{108024}.
\newblock


\bibitem[Minhao et~al\mbox{.}(2021)]%
        {minhao2021t}
\bibfield{author}{\bibinfo{person}{LIU Minhao}, \bibinfo{person}{Ailing Zeng},
  \bibinfo{person}{LAI Qiuxia}, \bibinfo{person}{Ruiyuan Gao},
  \bibinfo{person}{Min Li}, \bibinfo{person}{Jing Qin}, {and}
  \bibinfo{person}{Qiang Xu}.} \bibinfo{year}{2021}\natexlab{}.
\newblock \showarticletitle{T-wavenet: A tree-structured wavelet neural network
  for time series signal analysis}. In \bibinfo{booktitle}{\emph{International
  Conference on Learning Representations}}.
\newblock


\bibitem[Mnih et~al\mbox{.}(2013)]%
        {mnih2013playing}
\bibfield{author}{\bibinfo{person}{Volodymyr Mnih}, \bibinfo{person}{Koray
  Kavukcuoglu}, \bibinfo{person}{David Silver}, \bibinfo{person}{Alex Graves},
  \bibinfo{person}{Ioannis Antonoglou}, \bibinfo{person}{Daan Wierstra}, {and}
  \bibinfo{person}{Martin Riedmiller}.} \bibinfo{year}{2013}\natexlab{}.
\newblock \showarticletitle{Playing atari with deep reinforcement learning}.
\newblock \bibinfo{journal}{\emph{arXiv preprint arXiv:1312.5602}}
  (\bibinfo{year}{2013}).
\newblock


\bibitem[Oord(2016)]%
        {oord2016wavenet}
\bibfield{author}{\bibinfo{person}{Aaron van~den Oord}.}
  \bibinfo{year}{2016}\natexlab{}.
\newblock \showarticletitle{WaveNet: A Generative Model for Raw Audio}.
\newblock \bibinfo{journal}{\emph{arXiv preprint arXiv:1609.03499}}
  (\bibinfo{year}{2016}).
\newblock


\bibitem[Pan et~al\mbox{.}(2022)]%
        {pan2022wnet}
\bibfield{author}{\bibinfo{person}{Wenwen Pan}, \bibinfo{person}{Haonan Shi},
  \bibinfo{person}{Zhou Zhao}, \bibinfo{person}{Jieming Zhu},
  \bibinfo{person}{Xiuqiang He}, \bibinfo{person}{Zhigeng Pan},
  \bibinfo{person}{Lianli Gao}, \bibinfo{person}{Jun Yu}, \bibinfo{person}{Fei
  Wu}, {and} \bibinfo{person}{Qi Tian}.} \bibinfo{year}{2022}\natexlab{}.
\newblock \showarticletitle{Wnet: Audio-guided video object segmentation via
  wavelet-based cross-modal denoising networks}. In
  \bibinfo{booktitle}{\emph{Proceedings of the IEEE/CVF Conference on Computer
  Vision and Pattern Recognition}}. \bibinfo{pages}{1320--1331}.
\newblock


\bibitem[Pattanaik and Bouatouch(1995)]%
        {pattanaik1995haar}
\bibfield{author}{\bibinfo{person}{Sumanta~N Pattanaik} {and}
  \bibinfo{person}{Kadi Bouatouch}.} \bibinfo{year}{1995}\natexlab{}.
\newblock \showarticletitle{Haar wavelet: A solution to global illumination
  with general surface properties}. In \bibinfo{booktitle}{\emph{Photorealistic
  Rendering Techniques}}. Springer, \bibinfo{pages}{281--294}.
\newblock


\bibitem[Poiani et~al\mbox{.}(2021)]%
        {Poiani2021MetaReinforcementLB}
\bibfield{author}{\bibinfo{person}{Riccardo Poiani}, \bibinfo{person}{Andrea
  Tirinzoni}, {and} \bibinfo{person}{Marcello Restelli}.}
  \bibinfo{year}{2021}\natexlab{}.
\newblock \showarticletitle{Meta-Reinforcement Learning by Tracking Task
  Non-stationarity}. In \bibinfo{booktitle}{\emph{International Joint
  Conference on Artificial Intelligence}}.
\newblock
\urldef\tempurl%
\url{https://api.semanticscholar.org/CorpusID:234777726}
\showURL{%
\tempurl}


\bibitem[Polikar et~al\mbox{.}(1996)]%
        {polikar1996wavelet}
\bibfield{author}{\bibinfo{person}{Robi Polikar} {et~al\mbox{.}}}
  \bibinfo{year}{1996}\natexlab{}.
\newblock \bibinfo{title}{The wavelet tutorial}.
\newblock
\newblock


\bibitem[Qian et~al\mbox{.}(2024)]%
        {qian2024efficient}
\bibfield{author}{\bibinfo{person}{Yu-Yang Qian}, \bibinfo{person}{Peng Zhao},
  \bibinfo{person}{Yu-Jie Zhang}, \bibinfo{person}{Masashi Sugiyama}, {and}
  \bibinfo{person}{Zhi-Hua Zhou}.} \bibinfo{year}{2024}\natexlab{}.
\newblock \showarticletitle{Efficient non-stationary online learning by
  wavelets with applications to online distribution shift adaptation}. In
  \bibinfo{booktitle}{\emph{Forty-first International Conference on Machine
  Learning}}.
\newblock


\bibitem[Rakelly et~al\mbox{.}(2019)]%
        {rakelly2019efficient}
\bibfield{author}{\bibinfo{person}{Kate Rakelly}, \bibinfo{person}{Aurick
  Zhou}, \bibinfo{person}{Chelsea Finn}, \bibinfo{person}{Sergey Levine}, {and}
  \bibinfo{person}{Deirdre Quillen}.} \bibinfo{year}{2019}\natexlab{}.
\newblock \showarticletitle{Efficient off-policy meta-reinforcement learning
  via probabilistic context variables}. In
  \bibinfo{booktitle}{\emph{International conference on machine learning}}.
  PMLR, \bibinfo{pages}{5331--5340}.
\newblock


\bibitem[Ren et~al\mbox{.}(2022)]%
        {ren2022reinforcement}
\bibfield{author}{\bibinfo{person}{Hang Ren}, \bibinfo{person}{Aivar Sootla},
  \bibinfo{person}{Taher Jafferjee}, \bibinfo{person}{Junxiao Shen},
  \bibinfo{person}{Jun Wang}, {and} \bibinfo{person}{Haitham Bou-Ammar}.}
  \bibinfo{year}{2022}\natexlab{}.
\newblock \showarticletitle{Reinforcement learning in presence of discrete
  markovian context evolution}.
\newblock \bibinfo{journal}{\emph{arXiv preprint arXiv:2202.06557}}
  (\bibinfo{year}{2022}).
\newblock


\bibitem[Rescorla(1972)]%
        {rescorla1972theory}
\bibfield{author}{\bibinfo{person}{Robert~A Rescorla}.}
  \bibinfo{year}{1972}\natexlab{}.
\newblock \showarticletitle{A theory of Pavlovian conditioning: Variations in
  the effectiveness of reinforcement and non-reinforcement}.
\newblock \bibinfo{journal}{\emph{Classical conditioning, Current research and
  theory}}  \bibinfo{volume}{2} (\bibinfo{year}{1972}),
  \bibinfo{pages}{64--69}.
\newblock


\bibitem[Riemer et~al\mbox{.}(2019)]%
        {MER}
\bibfield{author}{\bibinfo{person}{Matthew Riemer}, \bibinfo{person}{Ignacio
  Cases}, \bibinfo{person}{Robert Ajemian}, \bibinfo{person}{Miao Liu},
  \bibinfo{person}{Irina Rish}, \bibinfo{person}{Yuhai Tu}, {and}
  \bibinfo{person}{Gerald Tesauro}.} \bibinfo{year}{2019}\natexlab{}.
\newblock \showarticletitle{Learning to Learn without Forgetting by Maximizing
  Transfer and Minimizing Interference}. In \bibinfo{booktitle}{\emph{In
  International Conference on Learning Representations (ICLR)}}.
\newblock


\bibitem[Seeger(2004)]%
        {seeger2004gaussian}
\bibfield{author}{\bibinfo{person}{Matthias Seeger}.}
  \bibinfo{year}{2004}\natexlab{}.
\newblock \showarticletitle{Gaussian processes for machine learning}.
\newblock \bibinfo{journal}{\emph{International journal of neural systems}}
  \bibinfo{volume}{14}, \bibinfo{number}{02} (\bibinfo{year}{2004}),
  \bibinfo{pages}{69--106}.
\newblock


\bibitem[Shannon(1949)]%
        {shannon1949communication}
\bibfield{author}{\bibinfo{person}{Claude~E Shannon}.}
  \bibinfo{year}{1949}\natexlab{}.
\newblock \showarticletitle{Communication in the presence of noise}.
\newblock \bibinfo{journal}{\emph{Proceedings of the IRE}}
  \bibinfo{volume}{37}, \bibinfo{number}{1} (\bibinfo{year}{1949}),
  \bibinfo{pages}{10--21}.
\newblock


\bibitem[Shensa(1992)]%
        {Shensa1992TheDW}
\bibfield{author}{\bibinfo{person}{Mark~J. Shensa}.}
  \bibinfo{year}{1992}\natexlab{}.
\newblock \showarticletitle{The discrete wavelet transform: wedding the a trous
  and Mallat algorithms}.
\newblock \bibinfo{journal}{\emph{IEEE Trans. Signal Process.}}
  \bibinfo{volume}{40} (\bibinfo{year}{1992}), \bibinfo{pages}{2464--2482}.
\newblock
\urldef\tempurl%
\url{https://api.semanticscholar.org/CorpusID:9791192}
\showURL{%
\tempurl}


\bibitem[Shensa et~al\mbox{.}(1992)]%
        {shensa1992discrete}
\bibfield{author}{\bibinfo{person}{Mark~J Shensa} {et~al\mbox{.}}}
  \bibinfo{year}{1992}\natexlab{}.
\newblock \showarticletitle{The discrete wavelet transform: wedding the a trous
  and Mallat algorithms}.
\newblock \bibinfo{journal}{\emph{IEEE Transactions on signal processing}}
  \bibinfo{volume}{40}, \bibinfo{number}{10} (\bibinfo{year}{1992}),
  \bibinfo{pages}{2464--2482}.
\newblock


\bibitem[Shi et~al\mbox{.}(2023)]%
        {shi2023sequence}
\bibfield{author}{\bibinfo{person}{Jiaxin Shi}, \bibinfo{person}{Ke~Alexander
  Wang}, {and} \bibinfo{person}{Emily Fox}.} \bibinfo{year}{2023}\natexlab{}.
\newblock \showarticletitle{Sequence modeling with multiresolution
  convolutional memory}. In \bibinfo{booktitle}{\emph{International Conference
  on Machine Learning}}. PMLR, \bibinfo{pages}{31312--31327}.
\newblock


\bibitem[Sodhani et~al\mbox{.}(2022)]%
        {sodhani2022block}
\bibfield{author}{\bibinfo{person}{Shagun Sodhani}, \bibinfo{person}{Franziska
  Meier}, \bibinfo{person}{Joelle Pineau}, {and} \bibinfo{person}{Amy Zhang}.}
  \bibinfo{year}{2022}\natexlab{}.
\newblock \showarticletitle{Block contextual mdps for continual learning}. In
  \bibinfo{booktitle}{\emph{Learning for Dynamics and Control Conference}}.
  PMLR, \bibinfo{pages}{608--623}.
\newblock


\bibitem[Stock and Anderson(2022)]%
        {stock2022trainable}
\bibfield{author}{\bibinfo{person}{Jason Stock} {and} \bibinfo{person}{Chuck
  Anderson}.} \bibinfo{year}{2022}\natexlab{}.
\newblock \showarticletitle{Trainable wavelet neural network for non-stationary
  signals}.
\newblock \bibinfo{journal}{\emph{arXiv preprint arXiv:2205.03355}}
  (\bibinfo{year}{2022}).
\newblock


\bibitem[Sutton and Barto(1998)]%
        {1998Reinforcement}
\bibfield{author}{\bibinfo{person}{Richard~S. Sutton} {and}
  \bibinfo{person}{Andrew~G. Barto}.} \bibinfo{year}{1998}\natexlab{}.
\newblock \showarticletitle{Reinforcement Learning: An Introduction}.
\newblock \bibinfo{journal}{\emph{IEEE Transactions on Neural Networks}}
  \bibinfo{volume}{9}, \bibinfo{number}{5} (\bibinfo{year}{1998}),
  \bibinfo{pages}{1054}.
\newblock


\bibitem[Tennenholtz et~al\mbox{.}(2023)]%
        {tennenholtz2023reinforcement}
\bibfield{author}{\bibinfo{person}{Guy Tennenholtz}, \bibinfo{person}{Nadav
  Merlis}, \bibinfo{person}{Lior Shani}, \bibinfo{person}{Martin Mladenov},
  {and} \bibinfo{person}{Craig Boutilier}.} \bibinfo{year}{2023}\natexlab{}.
\newblock \showarticletitle{Reinforcement Learning with History Dependent
  Dynamic Contexts}. In \bibinfo{booktitle}{\emph{International Conference on
  Machine Learning}}. PMLR, \bibinfo{pages}{34011--34053}.
\newblock


\bibitem[Todorov et~al\mbox{.}(2012)]%
        {Todorov2012MuJoCoAP}
\bibfield{author}{\bibinfo{person}{Emanuel Todorov}, \bibinfo{person}{Tom
  Erez}, {and} \bibinfo{person}{Yuval Tassa}.} \bibinfo{year}{2012}\natexlab{}.
\newblock \showarticletitle{MuJoCo: A physics engine for model-based control}.
\newblock \bibinfo{journal}{\emph{2012 IEEE/RSJ International Conference on
  Intelligent Robots and Systems}} (\bibinfo{year}{2012}),
  \bibinfo{pages}{5026--5033}.
\newblock
\urldef\tempurl%
\url{https://api.semanticscholar.org/CorpusID:5230692}
\showURL{%
\tempurl}


\bibitem[Wan et~al\mbox{.}(2024)]%
        {wan2024tcdformer}
\bibfield{author}{\bibinfo{person}{Jiashan Wan}, \bibinfo{person}{Na Xia},
  \bibinfo{person}{Yutao Yin}, \bibinfo{person}{Xulei Pan},
  \bibinfo{person}{Jin Hu}, {and} \bibinfo{person}{Jun Yi}.}
  \bibinfo{year}{2024}\natexlab{}.
\newblock \showarticletitle{TCDformer: A transformer framework for
  non-stationary time series forecasting based on trend and change-point
  detection}.
\newblock \bibinfo{journal}{\emph{Neural Networks}} (\bibinfo{year}{2024}),
  \bibinfo{pages}{106196}.
\newblock


\bibitem[Wang et~al\mbox{.}(2023)]%
        {wang2023wavelet}
\bibfield{author}{\bibinfo{person}{Jingyuan Wang}, \bibinfo{person}{Chen Yang},
  \bibinfo{person}{Xiaohan Jiang}, {and} \bibinfo{person}{Junjie Wu}.}
  \bibinfo{year}{2023}\natexlab{}.
\newblock \showarticletitle{When: A wavelet-dtw hybrid attention network for
  heterogeneous time series analysis}. In \bibinfo{booktitle}{\emph{Proceedings
  of the 29th ACM SIGKDD Conference on Knowledge Discovery and Data Mining}}.
  \bibinfo{pages}{2361--2373}.
\newblock


\bibitem[Willsky(2002)]%
        {Willsky2002MultiresolutionMM}
\bibfield{author}{\bibinfo{person}{Alan~S. Willsky}.}
  \bibinfo{year}{2002}\natexlab{}.
\newblock \showarticletitle{Multiresolution Markov models for signal and image
  processing}.
\newblock \bibinfo{journal}{\emph{Proc. IEEE}}  \bibinfo{volume}{90}
  (\bibinfo{year}{2002}), \bibinfo{pages}{1396--1458}.
\newblock
\urldef\tempurl%
\url{https://api.semanticscholar.org/CorpusID:122692461}
\showURL{%
\tempurl}


\bibitem[Xie et~al\mbox{.}(2021)]%
        {xie2020deep}
\bibfield{author}{\bibinfo{person}{Annie Xie}, \bibinfo{person}{James
  Harrison}, {and} \bibinfo{person}{Chelsea Finn}.}
  \bibinfo{year}{2021}\natexlab{}.
\newblock \showarticletitle{Deep reinforcement learning amidst lifelong
  non-stationarity}.
\newblock \bibinfo{journal}{\emph{International Conference on Machine Learning
  (ICML)}} (\bibinfo{year}{2021}).
\newblock


\bibitem[Xie(2018)]%
        {xie2018simglucose}
\bibfield{author}{\bibinfo{person}{Jinyu Xie}.}
  \bibinfo{year}{2018}\natexlab{}.
\newblock \showarticletitle{Simglucose v0. 2.1}.
\newblock \bibinfo{journal}{\emph{Avaible: https://github.
  com/jxx123/simglucose. Accessed on: Jan-20-2020}} (\bibinfo{year}{2018}).
\newblock


\bibitem[Yang et~al\mbox{.}(2024)]%
        {Yang2024WaveNetTN}
\bibfield{author}{\bibinfo{person}{Zhirui Yang}, \bibinfo{person}{Yulan Hu},
  \bibinfo{person}{Ouyang Sheng}, \bibinfo{person}{Jingyu Liu},
  \bibinfo{person}{Shuqiang Wang}, \bibinfo{person}{Xibo Ma},
  \bibinfo{person}{Wenhan Wang}, \bibinfo{person}{Hanjing Su}, {and}
  \bibinfo{person}{Yong Liu}.} \bibinfo{year}{2024}\natexlab{}.
\newblock \showarticletitle{WaveNet: Tackling Non-stationary Graph Signals via
  Graph Spectral Wavelets}. In \bibinfo{booktitle}{\emph{AAAI Conference on
  Artificial Intelligence}}.
\newblock
\urldef\tempurl%
\url{https://api.semanticscholar.org/CorpusID:268692563}
\showURL{%
\tempurl}


\bibitem[Yao et~al\mbox{.}(2022)]%
        {yao2022wave}
\bibfield{author}{\bibinfo{person}{Ting Yao}, \bibinfo{person}{Yingwei Pan},
  \bibinfo{person}{Yehao Li}, \bibinfo{person}{Chong-Wah Ngo}, {and}
  \bibinfo{person}{Tao Mei}.} \bibinfo{year}{2022}\natexlab{}.
\newblock \showarticletitle{Wave-vit: Unifying wavelet and transformers for
  visual representation learning}. In \bibinfo{booktitle}{\emph{European
  Conference on Computer Vision}}. Springer, \bibinfo{pages}{328--345}.
\newblock


\bibitem[Ye et~al\mbox{.}(2024)]%
        {ye2024state}
\bibfield{author}{\bibinfo{person}{Mingxuan Ye}, \bibinfo{person}{Yufei Kuang},
  \bibinfo{person}{Jie Wang}, \bibinfo{person}{Yang Rui},
  \bibinfo{person}{Wengang Zhou}, \bibinfo{person}{Houqiang Li}, {and}
  \bibinfo{person}{Feng Wu}.} \bibinfo{year}{2024}\natexlab{}.
\newblock \showarticletitle{State Sequences Prediction via Fourier Transform
  for Representation Learning}.
\newblock \bibinfo{journal}{\emph{Advances in Neural Information Processing
  Systems}}  \bibinfo{volume}{36} (\bibinfo{year}{2024}).
\newblock


\bibitem[Yu et~al\mbox{.}(2020)]%
        {yu2020meta}
\bibfield{author}{\bibinfo{person}{Tianhe Yu}, \bibinfo{person}{Deirdre
  Quillen}, \bibinfo{person}{Zhanpeng He}, \bibinfo{person}{Ryan Julian},
  \bibinfo{person}{Karol Hausman}, \bibinfo{person}{Chelsea Finn}, {and}
  \bibinfo{person}{Sergey Levine}.} \bibinfo{year}{2020}\natexlab{}.
\newblock \showarticletitle{Meta-world: A benchmark and evaluation for
  multi-task and meta reinforcement learning}. In
  \bibinfo{booktitle}{\emph{Conference on robot learning}}. PMLR,
  \bibinfo{pages}{1094--1100}.
\newblock


\bibitem[Yu et~al\mbox{.}(2021)]%
        {yu2021wavefill}
\bibfield{author}{\bibinfo{person}{Yingchen Yu}, \bibinfo{person}{Fangneng
  Zhan}, \bibinfo{person}{Shijian Lu}, \bibinfo{person}{Jianxiong Pan},
  \bibinfo{person}{Feiying Ma}, \bibinfo{person}{Xuansong Xie}, {and}
  \bibinfo{person}{Chunyan Miao}.} \bibinfo{year}{2021}\natexlab{}.
\newblock \showarticletitle{Wavefill: A wavelet-based generation network for
  image inpainting}. In \bibinfo{booktitle}{\emph{Proceedings of the IEEE/CVF
  international conference on computer vision}}. \bibinfo{pages}{14114--14123}.
\newblock


\bibitem[Zhang et~al\mbox{.}(2024a)]%
        {zhang2024tackling}
\bibfield{author}{\bibinfo{person}{Wanpeng Zhang}, \bibinfo{person}{Yilin Li},
  \bibinfo{person}{Boyu Yang}, {and} \bibinfo{person}{Zongqing Lu}.}
  \bibinfo{year}{2024}\natexlab{a}.
\newblock \showarticletitle{Tackling Non-Stationarity in Reinforcement Learning
  via Causal-Origin Representation}. In \bibinfo{booktitle}{\emph{Proceedings
  of the 41st International Conference on Machine Learning}},
  Vol.~\bibinfo{volume}{235}. \bibinfo{publisher}{PMLR},
  \bibinfo{pages}{59264--59288}.
\newblock


\bibitem[Zhang et~al\mbox{.}(2024b)]%
        {zhang2024speaking}
\bibfield{author}{\bibinfo{person}{Xiangyu Zhang}, \bibinfo{person}{Daijiao
  Liu}, \bibinfo{person}{Hexin Liu}, \bibinfo{person}{Qiquan Zhang},
  \bibinfo{person}{Hanyu Meng}, \bibinfo{person}{Leibny~Paola Garcia},
  \bibinfo{person}{Eng~Siong Chng}, {and} \bibinfo{person}{Lina Yao}.}
  \bibinfo{year}{2024}\natexlab{b}.
\newblock \showarticletitle{Speaking in Wavelet Domain: A Simple and Efficient
  Approach to Speed up Speech Diffusion Model}.
\newblock \bibinfo{journal}{\emph{arXiv preprint arXiv:2402.10642}}
  (\bibinfo{year}{2024}).
\newblock


\bibitem[Zhou et~al\mbox{.}(2022)]%
        {zhou2022fedformer}
\bibfield{author}{\bibinfo{person}{Tian Zhou}, \bibinfo{person}{Ziqing Ma},
  \bibinfo{person}{Qingsong Wen}, \bibinfo{person}{Xue Wang},
  \bibinfo{person}{Liang Sun}, {and} \bibinfo{person}{Rong Jin}.}
  \bibinfo{year}{2022}\natexlab{}.
\newblock \showarticletitle{Fedformer: Frequency enhanced decomposed
  transformer for long-term series forecasting}. In
  \bibinfo{booktitle}{\emph{International conference on machine learning}}.
  PMLR, \bibinfo{pages}{27268--27286}.
\newblock


\end{thebibliography}

\appendix
\setcounter{theorem}{0}
\setcounter{lemma}{0}
\section{Proof}
\label{Proof}
\subsection{Proof of the Discrete Wavelet Transform}
\label{a:dwt}
The Wavelet Transform (WT)~\cite{burrus1998wavelets} analyzes and represents the various frequency components of a signal utilizing a scalable and translatable mother wavelet function $\psi$. Given a time series $x(t)$, the continuous wavelet transform (CWT) extracts its frequency by performing the \textbf{continuous convolution}: 
\begin{equation}
\label{eq:cwt}
\mathcal{W}(d, \beta)= |\beta|^{-\frac{1}{2}}\int_{-\infty}^{\infty} x(t) \psi\left(\beta^{-1}(t-d)\right) \mathrm{d}t.
\end{equation}
The translation position $d$ allows the $\psi$ function to move along the signal $x(t)$ in the time domain, enabling analysis at different time points. The scaling factor $\beta$ controls the degree of stretching and compression of $\psi$ function, which in turn affects its resolution properties and the range of extracted frequencies. Specifically, a smaller $\beta$ corresponds to higher resolution and a smaller timescale, facilitating the capture of finer details, while a larger $\beta$ corresponds to lower resolution and a larger timescale, suitable for analyzing overall trends. CWT typically analyzes signals at a fixed resolution, either fine or coarse. To achieve multi-resolution analysis while simplifying computations, the discrete version of CWT, the Discrete Wavelet Transform (DWT)~\cite{shensa1992discrete}, is introduced to decompose the signal layer by layer at different resolutions, as formalized in Lemma \ref{a:lemma}.
\begin{lemma}
\label{a:lemma}
\textbf{DWT}. Let the initial approximation coefficient $\mathbf{u}_0$ be equal to an input discrete time series $x(n)$. By substituting a pair of \textbf{low-pass filter $y_0$} and \textbf{high-pass filter $y_1$} for the mother wavelet function $\psi(t)$, following iterative form is derived through discrete convolution operations: 
\begin{equation}
\mathbf{u}_m(n) = \sum\textstyle_{k=1}^{K} y_0(k)\mathbf{u}_{m-1}(2n-k), \mathbf{g}_m(n) = \sum\textstyle_{k=1}^{K} y_1(k)\mathbf{u}_{m-1}(2n-k),
\end{equation}
where $m$ represents the decomposition number and $k$ denotes the resolution size. The \textbf{approximation coefficients} $\mathbf{u}_m$ represent low-frequency components, whereas the \textbf{detail coefficients} $\mathbf{g}_m$ represent high-frequency components.
\end{lemma}
\begin{proof}
For the continuous time series $x(t), t\in[0,T)$, the coarsest approximation of $x(t)$ is:
\begin{equation}
\textstyle x^{(0)}(t)\triangleq \mathbf{u}_{0,0}\psi(t), \quad \psi(t)=1 (0 \leq t<T), \quad \mathbf{u}_{0,0}=\int_{0}^T x(t) dt,
\end{equation}
where $\psi(t)$ denotes the \textit{scaling function} and the coefficient $\mathbf{u}_{0,0}$ is the average value of $x(t)$ in this interval. The superscript 0 implies the initial approximation of $x(t)$. To obtain a more accurate estimation of $x(t)$, the approximation can be refined by halving the intervals into smaller parts: 
\begin{equation}
\textstyle x^{(1)}(t) \triangleq \mathbf{u}_{1,0} \psi(2t)+\mathbf{u}_{1,1}\psi(2t-1), \quad \mathbf{u}_{1,0}=\int_0^{T/2} x(t) dt, \quad \mathbf{u}_{1,1}=\int_{T/2}^T x(t) dt.
\end{equation}
Define the mother wavelet function as:  
\begin{equation}
\psi_{m, k}(t) \triangleq 2^{\frac{m}{2}} \psi\left(2^m t-k\right)
\end{equation}
By repeating this procedure, $x(t)$ can be approximated with finer precision, given as:
\begin{equation}
\label{a-eq:uni-res}
x^{(m)}(t)=\sum_{k \in \mathbb{Z}} \mathbf{u}_{m, k}\psi_{m,k}(t), \quad \mathbf{u}_{m, k}=\int_{-\infty}^{\infty} x(t) \psi_{m, k}(t) \mathrm{d}t, 
\end{equation}
where $m \in \mathbb{N}$ represents the $m$-th interval bisection, $k$ represents a series of resolutions, and $\star$ indicates the convolution operation. After the $m$-th bisection of the original time interval, the scaling functions in the subspace $U_m \triangleq \operatorname{span}\left(\left\{\psi_{m, k}\right\}_{k \in \mathbb{Z}}\right)$ are capable of capturing local structures in $x(t)$ at a timescale no longer than $T/2^m$.

Inspired by multi-resolution analysis (MRA)~\cite{Willsky2002MultiresolutionMM}, discrete wavelet transform (DWT)~\cite{shensa1992discrete,shi2023sequence} introduces the orthogonal subspace $G_m \triangleq \operatorname{span}\left(\left\{\delta_{m, k}\right\}_{n \in \mathbb{Z}}\right)$ of $U_m$, where $\delta$ indicates orthogonal scaling functions~\cite{burrus1998wavelets}. $U_{M}$ can be decomposed into the orthogonal sum of the lowest resolution subspace $U_0$ and a series of orthogonal complement $G$:
\begin{equation}
\label{a-eq:decompose}
U_M=U_{M-1} \oplus G_{M-1}=U_0 \oplus G_0 \oplus \ldots \oplus G_{M-2} \oplus G_{M-1} .
\end{equation}
Accordingly, $x(t)$ can be represented by a series of different resolutions:
\begin{equation}
\label{a-eq:cwt}
x^{(M)}(t)=\mathbf{u}_{0,0} \psi(t)+\sum_{m=0}^{M-1} \sum_{k \in \mathbb{N}} \mathbf{g}_{m, k} \delta_{m, k}(t), 
\end{equation}
where $\mathbf{u}$ and $\mathbf{g}$ are referred to as the approximation coefficient and detail coefficient, respectively. Therefore, for discrete time series $x(n)$, we can rewrite Eq.~\ref{a-eq:uni-res} in the following iterated form:
\begin{equation}
\begin{aligned}
\label{a-eq:cwt-z}
x_{m-1} &= \sum_n\mathbf{u}_{m-1}(n) \psi_{m-1}(n)\\
&= \sum_n \mathbf{u}_{m-2}(n) \psi_{m-2}(n)+\sum_n \mathbf{g}_{m-2}(n) \delta_{m-2}(n),
\end{aligned}
\end{equation}
where $n$ represents the sequence length, and $\psi(n)$ and $\delta(n)$ are the discrete forms of $\psi(t)$ and $\delta(t)$, respectively. As a result, by performing the discrete convolution operation $(\star)$, the discrete wavelet transform can be obtained as follows: 
\begin{equation}
\begin{aligned}
\mathbf{u}_m(n) & =x_{m-1}\star\psi_{m}(n) \\
& =\left(\sum\textstyle_{n=1}^N \mathbf{u}_{m-1}(n) \psi_{m-1}(n)\right)\star\left(\sum\textstyle_{k=1}^{K} y_0(k) \psi_{m-1}(2n-k)\right) \\
& =\sum\textstyle_{k=1}^{K} y_0(k)\mathbf{u}_{m-1}(2n-k), \text { and similarly }, \\
\mathbf{g}_m(n) & =x_{m-1}\star\delta_m(n)=\sum\textstyle_{k=1}^{K} y_1(k)\mathbf{u}_{m-1}(2n-k),
\end{aligned}
\end{equation}
which concludes the proof of lemma~\ref{a:lemma}.
\end{proof}

Non-stationary RL is defined on the time-evolving task distribution $p(\mathcal{T}_t)$, where an agent interacts with a sequence of Markov Decision Processes (MDPs)~\cite{1998Reinforcement} $\mathcal{M}_{\omega_0}, \mathcal{M}_{\omega_1}, \cdots, \mathcal{M}_{\omega_h}$. The evolution of these MDPs is determined by a history-dependent stochastic process $\rho$, i.e., $\omega_{h+1}\sim \rho\left(\omega_{h+1}|\omega_0, \omega_1, \ldots, \omega_{h}\right)$, where $\omega$ is a task ID  
that regulates the properties of different MDPs. Each MDP $\mathcal{M}_\omega$ is represented by a tuple $(\mathcal{S}, \mathcal{A}, \mathcal{P}_\omega, \mathcal{R}_\omega, \gamma, \mathcal{P}(s_0))$, in which $\mathcal{S}$ denotes the state space, $\mathcal{A}$ the action space, $\mathcal{P}(s'|s,a)$ the transition dynamics, $\mathcal{R}(s,a,s')$ the reward function, $\mathcal{P}(s_0)$ the initial state distribution, and $\gamma \in [0,1)$ the discount factor. 

\subsection{Proof of the Convergence of the Wavelet TD Loss}
\label{a-th:loss}
\begin{theorem}
Let \(\mathcal{W}\) denote the set of all functions \(W:\mathcal{S}\times\mathcal{A}\rightarrow\mathbb{C}^{D}\) that map from the time domain to the wavelet domain. The wavelet update operator \(\mathcal{F}:\mathcal{W}\rightarrow\mathcal{W}\), defined as
\begin{align}
\mathcal{F} W(\mathbf{z}_t) = \mathbf{z}_t + \Gamma W(\mathbf{z}_{t+1}),
\end{align}
is a contraction mapping, where $\mathbf{z}_{t}$ and $\mathbf{z}_{t+1}$ represents the current and next task representation sequence, respectively, and \(\Gamma\) denotes the discount factor in a diagonal matrix form. 
\end{theorem}

\begin{proof}
The norm on \(\mathcal{W}\) is defined as $\| W \|_\mathcal{W}:=\mathop{\sup}\limits_{\substack{\mathbf{z}\in\mathcal{B}}}\mathop{\max}\limits_{\substack{\mathbb{K}\in\mathcal{K}}} \left\| \big[W(\mathbf{z_t})]_{\mathbb{K}} \right\|_D$, where $\mathcal{K}$ denotes the sequence length of $\mathbf{z_t}$. For any $W_1,W_2\in\mathcal{W}$, we have
\begin{equation}
\begin{aligned} 
\left\|\mathcal{F}^\pi W_1-\mathcal{F}^\pi W_2\right\|_{\mathcal{W}} & =\mathop{\sup}\limits_{\substack{\mathbf{z}\in\mathcal{B}}}\mathop{\max}\limits_{\substack{\mathbb{K}\in\mathcal{K}}}\|[\mathbf{z}_t]_\mathbb{K}+\gamma \mathbb{E}_{\mathbf{z}_{t+1} \sim \mathcal{B}}\left[[W_1\left(\mathbf{z}_{t+1}\right)]_\mathbb{K}\right]   \\
&-[\mathbf{z}_t]_\mathbb{K}-\gamma \mathbb{E}_{\mathbf{z}_{t+1} \sim \mathcal{B}}\left[W_2[\left(\mathbf{z}_{t+1}\right)]_\mathbb{K}\right]\|_{D} \\
& \leq \gamma\cdot\mathop{\max}\limits_{\substack{\mathbb{K}\in\mathcal{K}}}\mathop{\sup}\limits_{\substack{\mathbf{z}\in\mathcal{B}}}\| \mathbb{E}_{\mathbf{z}_{t+1} \sim \mathcal{B}}\left[[W_1\left(\mathbf{z}_{t+1}\right)]_\mathbb{K}-[W_2\left(\mathbf{z}_{t+1}\right)]_\mathbb{K}\right]\|_{D} \\
& \leq \gamma\cdot\mathop{\max}\limits_{\substack{\mathbb{K}\in\mathcal{K}}}\mathop{\sup}\limits_{\mathbf{z} \in \mathcal{B}} \left\|[W_1(\mathbf{z}_{t+1})-W_2(\mathbf{z}_{t+1})]_\mathbb{K}\right\|_D\\
& =\gamma\cdot\left\|W_1-W_2\right\|_{\mathcal{W}}, 
\end{aligned}
\end{equation}
where $\Gamma=diag(\gamma,\gamma,\cdots,\gamma)_{\mathcal{K}\times\mathcal{K}}$ and $\gamma\in[0,1)$, proving that $\mathcal{F}$ is a contraction mapping.
\end{proof}

\subsection{Proof of the Policy Performance Distinction via Wavelet Domain Features}
\label{a-theorem2}
\begin{lemma} Following the previous definition of state distribution~\cite{achiam2017constrained}, we define the discounted task representation distribution as follows:
    \begin{equation}
    \mathcal{D}^\pi(z)=(1-\gamma) \sum_{t=0}^\infty \gamma^t \sum_{h=0}^\infty \mathcal{P}_{\omega_h}(z_t=z|\pi,y)\rho(\omega_{h}|\omega_{0:h-1}).
\end{equation}
The neural network $y$ maps the trajectory $\tau$ to the task representation, denoted as $z$. Then the expected discounted total reward under policy $\pi$ can be expressed as 
\begin{align}
    J_\pi 
    &= \sum_{t=0}^\infty \gamma^t E_{\omega_{h}\sim\rho,\tau\sim\pi}\left[ \mathcal{R}_{\omega_{h}}(\tau)\right] \notag \\
    &=  \sum_{t=0}^\infty \gamma^t \sum_{h=0}^\infty \int_{\tau} \mathcal{R}^\pi(\tau)\mathcal{P}_{\omega_{h}}(\tau|\pi)\rho(\omega_{h}|\omega_{0:h-1})\,\mathrm{d}\tau  \notag \\
    &= \int_{\tau} \mathcal{R}^\pi(\tau)\sum_{t=0}^\infty \gamma^t \sum_{h=0}^\infty \mathcal{P}_{\omega_{h}}(\tau|\pi)\rho(\omega_{h}|\omega_{0:h-1})\,\mathrm{d}\tau  \notag \\
    &= \int_{z\sim y(\tau)} \mathcal{R}^\pi(z)\sum_{t=0}^\infty \gamma^t \sum_{h=0}^\infty \mathcal{P}_{\omega_{h}}(z|\pi)\rho(\omega_{h}|\omega_{0:h-1})\,\mathrm{d}z  \notag \\
    &=\frac{1}{1-\gamma}\int_\mathcal{Z} \mathcal{R}^\pi(z)\mathcal{D}^\pi(z)\,\mathrm{d}z \notag \\
    &\xlongequal{z\sim y(s,a,s')}\frac{1}{1-\gamma} E_{\substack{z\sim \mathcal{D}^\pi,a\sim\pi(\cdot|s,z) \\ s'\sim \mathcal{P}(\cdot|s,a,z)}}\left[\mathcal{R}(z)\right].
\end{align}
\label{a-lemma}
\end{lemma}

\begin{theorem}
\label{a-th:pi-w}
Suppose that the reward function \(\mathcal{R}(z)\) can be expanded into a Bth-degree Taylor series for \(z\in\mathbb{R}^D\), then for any two policies \(\pi_1\) and \(\pi_2\), their performance difference is bounded as:
\begin{align}
    |J_{\pi_1}-J_{\pi_2}|\leq \frac{\sqrt{D}}{1-\gamma}\cdot \sum_{b=1}^{B} \frac{\left\|\mathcal{R}^{(b)}(0)\right\|_D}{b!}\cdot\mathop{\max}\limits_{1\leq q \leq D}\mathop{\sup}\limits_{d_q\in\mathbb{R},\beta_q>0}\left|\mathcal{W}_{\pi_1}^{(b)}(d_q, \beta_q)-\mathcal{W}_{\pi_2}^{(b)}(d_q, \beta_q)\right|,
\end{align}
where \(\mathcal{W}_{\pi}^{(b)}(d,\beta)\) denotes the wavelet transform of the $b$th power of the task representation sequence $\mathbf{z}^{(b)}=[z_0,z_1,\ldots,z_{T_H}]^{(b)}$ for any integer \(b\in [1,B]\).
\end{theorem}

\begin{proof}
First, the reward function can be written as \(\mathcal{R}(z)=\sum_{b=0}^{B}\frac{\mathcal{R}^{(b)}(0)^{\mathsf{T}}}{b!}z^b\) based on the Taylor series expansion~\cite{ye2024state}. According to lemma ~\ref{a-lemma}, for any integer \(b\in[1,B]\), we have 
\begin{align}
    |J_{\pi_1}-J_{\pi_2}| &\leq \frac{1}{1-\gamma}\int_\mathcal{Z} \left[\mathcal{R}^{\pi_1}(z)\mathcal{D}^{\pi_1}(z)-\mathcal{R}^{\pi_2}(z)\mathcal{D}^{\pi_2}(z)\right]\,\mathrm{d}z \notag\\
    &\leq \sum_{b=0}^{B}\frac{\left\|\mathcal{R}^{(b)}(0)\right\|_D}{b!}\cdot\left\| \int_\mathcal{Z}\left[ z^b \mathcal{D}^{\pi_1}(z)-z^b \mathcal{D}^{\pi_2}(z)\right]\mathrm{d}z \right\|_D \notag\\
    &= \sum_{b=0}^{B}\frac{\left\|\mathcal{R}^{(b)}(0)\right\|_D}{b!} \left\| \mathop{E}\limits_{z\sim \mathcal{D}^{\pi_1}}\left[ z^b \right] - \mathop{E}\limits_{z\sim \mathcal{D}^{\pi_2}}\left[ z^b \right]\right\|_D. \label{a-eq:thm1}
\end{align}
Since the inverse wavelet transform~\cite{burrus1998wavelets} of \(\mathcal{W}_{\pi}^{(b)}(d,\beta)\) is the reconstructed task representation sequence \(\mathbf{z}^{(b)}\), we have
\begin{equation}
    \mathop{E}\limits_{z_q\sim \mathcal{D}^{\pi}}\left[z_q^b \right]=\int_0^{\infty} \int_{-\infty}^{\infty} \mathcal{W}_{\pi}^{(b)}(d_q, \beta_q) \tilde{\psi}\left(\frac{t-d_q}{\beta_q}\right) \mathrm{d}d_q \mathrm{d}\beta_q, \quad\forall q=1,2,\dots,D,
\end{equation}
where $\tilde{\psi}$ is the dual function form of the mother wavelet function $\psi$ and has the normalized orthogonal property. The dimensions of translation position \(d\) and scaling factor \(\beta\) are identical to those of \(z\). 

Then we have the following derivation:
\begin{align}
   & \left|\mathop{E}\limits_{z_q\sim \mathcal{D}^{\pi_1}}\left[z_q^b \right]  - \mathop{E}\limits_{z_q\sim \mathcal{D}^{\pi_2}}\left[z_q^b \right]\right|
   \leq \int_0^{\infty} \int_{-\infty}^{\infty} \left | \mathcal{W}_{\pi_1}^{(b)}(d_q, \beta_q)-\mathcal{W}_{\pi_2}^{(b)}(d_q, \beta_q) \right|\cdot \left| \tilde{\psi}\left(\frac{t-d_q}{\beta_q}\right) \right| \mathrm{d}d_q \mathrm{d}\beta_q \notag\\
   & \leq \mathop{\sup}\limits_{d_q\in\mathbb{R},\beta_q>0}\left|\mathcal{W}_{\pi_1}^{(b)}(d_q, \beta_q)-\mathcal{W}_{\pi_2}^{(b)}(d_q, \beta_q)\right| \int_0^{\infty} \int_{-\infty}^{\infty} \left| \tilde{\psi}\left(\frac{t-d_q}{\beta_q}\right) \right| \mathrm{d}d_q \mathrm{d}\beta_q \notag\\
   & \leq \mathop{\sup}\limits_{d_q\in\mathbb{R},\beta_q>0}\left|\mathcal{W}_{\pi_1}^{(b)}(d_q, \beta_q)-\mathcal{W}_{\pi_2}^{(b)}(d_q, \beta_q)\right|. 
\end{align}
Then the policy performance difference bound is derived as follows:
\begin{align*}
     |J_{\pi_1}-J_{\pi_2}|
     &\leq\frac{1}{1-\gamma}\cdot \sqrt{D}\cdot \sum_{b=1}^{B} \frac{\left\|\mathcal{R}^{(b)}(0)\right\|_D}{b!}\cdot\mathop{\max}\limits_{1\leq q \leq D}\mathop{\sup}\limits_{d_q\in\mathbb{R},\beta_q>0}\left|\mathcal{W}_{\pi_1}^{(b)}(d_q, \beta_q)-\mathcal{W}_{\pi_2}^{(b)}(d_q, \beta_q)\right|,
\end{align*}
which concludes the proof of Theorem~\ref{a-th:pi-w}.
\end{proof}

\subsection{Proof of the Policy Improvement with Wavelet Task Representations}
\label{a-th:PI}
Suppose that the function class for the contextual policy $\pi$ of \modelname is defined as $\Pi_{\mathrm{\modelname}} \doteq\left\{\pi \text {, s.t. } \exists \theta \text { with } \pi(\cdot| s,\hat{z})=f_\theta\left(s,Y_{\phi}(e_\eta(\mathcal{C}))\right) \forall s\right\}$, where $f_\theta$ is a neural network characterizing the policy conditioned on top of the context encoder $e_\eta$ and the wavelet representation network $Y_\phi$.

\begin{assumption}
(1.1) Access to the true history policy: $\pi_{\hat{h}}=\pi_{h}$. (1.2) The context encoder $e_\eta$ and the wavelet representation network $Y_{\phi}$ allows to represent the estimate of the history policy : $\pi_{\hat{h}}(a | s,Y_{\phi}(e_\eta(\mathcal{C})))=\pi_{\hat{h}}(a | s,\hat{z}) \in \Pi_{\mathrm{\modelname}}$. (1.3) Access to the true performance of policies: $\hat{J}(\pi)=J(\pi)$ for all policies $\pi$. (1.4) The non-stationary RL algorithm performs perfect optimization on top of $e_\eta$ and $Y_{\phi}$: $J_{\mathrm{\modelname}}=\max _{\pi \in \Pi_{\mathrm{\modelname}}} \hat{J}(\pi)$.
\end{assumption}

\begin{theorem}
\label{a-th:pi-hatz}
Under the Assumption 1.1-1.4, \modelname returns a policy $\pi(\cdot | s,\hat{z})$ conditioned on the wavelet task representation that improves upon its iterated history policy $\pi_h$: $J_{\text {\modelname}}\geq J_{\pi_h}$. 
\end{theorem}

\begin{proof}
\begin{align}
    J_{\text {\modelname}}& =J\left(\underset{\pi \in \Pi_{\text {\modelname}}}{\arg \max } \hat{J}(\pi)\right) \tag*{(Assumption 1.4)}\\ 
& =\max _{\pi \in \Pi_{\text {\modelname}}} J(\pi) \tag*{(Assumption 1.3)}\\
& \geq J\left(\pi_{\hat{h}}\right) \tag*{(Assumption 1.2)}\\
& \geq J\left(\pi_h\right) \tag*{(Assumption 1.1)}
\end{align}
which concludes the proof of Theorem~\ref{a-th:pi-hatz}.
\end{proof}

\section{Experimental Details}
\label{detail setting}
\subsection{Baselines} 
We compared \modelname with seven prevailing and competitive baselines. Initially, SAC, a standard RL method, serves as the backbone to validate the importance of task inference. Conversely, PEARL and RL$^2$, which are classical meta-RL methods, perform task inference using an encoder, emphasizing the significance of modeling the task evolution process in non-stationary environments. Moreover, the last four baselines (CEMRL, SeCBAD, TRIO, and COREP) are specifically designed to address RL tasks in non-stationary settings, further substantiating the advantages of modeling and predicting the task evolution process through wavelet transform in WISDOM. Their introductions are as follows:

(1) \textbf{SAC}~\cite{Haarnoja2018SoftAA} is a standard model-free, off-policy reinforcement learning method that aims to maximize both the expected return and the entropy of the policy;

(2) \textbf{PEARL}~\cite{rakelly2019efficient} is a classical context-based meta-RL method that enhances sample efficiency and policy performance by disentangling task inference and control, employing online probabilistic filtering of latent task variables for effective adaptation to new tasks from limited experience; 

(3) \textbf{RL$^2$}~\cite{duan2016rl} employs RNNs to encode states for meta-learning, storing learned prior knowledge in the hidden states. This prior knowledge is then applied across multiple tasks, and RL algorithms are used to train the weights of the RNN, addressing the challenge of rapid learning;

(4) \textbf{CEMRL}~\cite{bing2023meta} leverages a Gaussian Mixture Model to represent non-stationary task distributions with cluster structure and learns a more compact latent space through reconstructing tasks; 

(5) \textbf{SeCBAD}~\cite{chen2022adaptive} perceives and adapts to new evolving tasks through reward functions, leading to a policy that can adapt to rapid variations; 

(6) \textbf{TRIO}~\cite{Poiani2021MetaReinforcementLB} meta-trains a variational module for inferring distributions over latent context and leverages a Gaussian Process (GP) to model the task evolution process; 

(7) \textbf{COREP}~\cite{zhang2024tackling} utilizes a dual graph structure to retrospect causal origin of non-stationarity, with a core graph emphasizing stable representation and a general graph compensating for overlooked edge information.

For fair comparison, \modelname and all baseline models, except TRIO and RL$^2$ which emphasize the use of a recurrent encoder, employed the same MLP network architecture for their context encoders.

\subsection{Evaluation Metrics}
All models were trained with equal time steps, and subsequent policy evaluations were performed with the same number of time steps. For evaluating the policy's performance on Meta-World, the average success rate was computed across six random seeds. On MuJoCo and Type-1 Diabetes control, the average return was computed across six random seeds. 

\subsection{Non-Stationary Experimental Settings}
We evaluated our proposed model \modelname on various non-stationary tasks with uncertain evolution periods, using three benchmarks: \textbf{Meta-World}, \textbf{Type-1 Diabetes} and \textbf{MuJoCo}. Then we describe the details of these three environments.
\begin{figure*}[ht]
\centering
\includegraphics[width=0.95\textwidth]{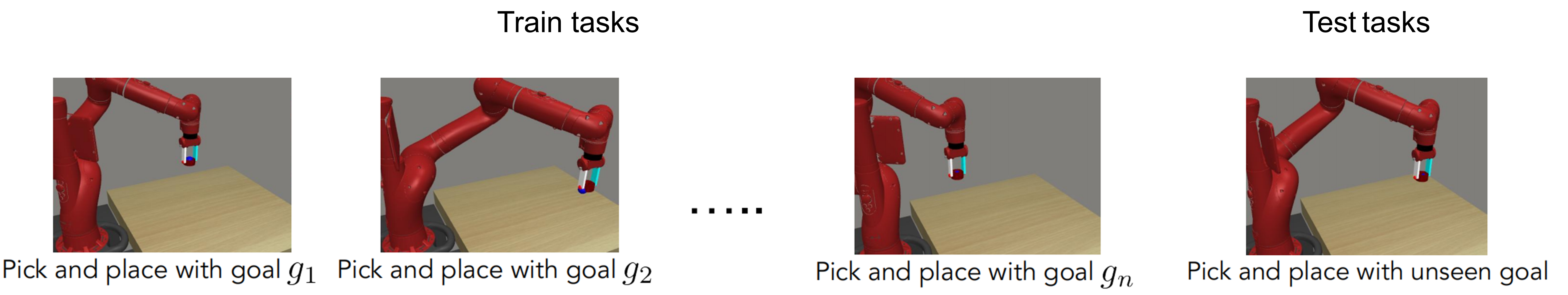}
\caption{Visualization of ML1 evaluation protocol.}
\label{meta-world}
\vskip -0.2in
\end{figure*}

\paragraph{Meta-World} Meta-World is an open-source simulated benchmark for meta-RL and multi-task learning, consisting of 50 distinct robotic manipulation tasks. As shown in Fig.~\ref{meta-world}, we experiment with the Meta-Learning 1 (ML1) evaluation protocol, which evaluates meta-RL algorithms' few-shot adaptation to goal variation within one task. ML1 uses single Meta-World tasks, with the meta-training tasks corresponding to 50 random initial object and goal positions, and meta-testing conducted on 10 held-out positions. The non-stationarity in Meta-World is evidenced by the continuous fluctuation of the target position for the robotic arm's movement, which is correlated with the reward function. We selected 8 different types of tasks, and the detailed descriptions are listed in Table~\ref{ml1-descriptions}.

\begin{table*}[ht]
\centering
\caption{Meta-World task descriptions}
\begin{adjustbox}{max width=0.85\textwidth}
\begin{tabular}{ll}
    \toprule
    Task  & Description \\
    \midrule
    Button-Press & Randomize button positions and press a button. \\
    Door-Close & Randomize door positions and close a door with a revolving joint. \\
    Door-Lock & Randomize door positions and lock the door by rotating the lock clockwise. \\
    Door-Unlock & Randomize door positions and unlock the door by rotating the lock counter-clockwise. \\
    Faucet-Close & Randomize faucet positions and rotate the faucet clockwise. \\
    Handle-Press & Randomize the handle positions and press a handle down. \\
    Plate-Slide & Randomize the plate and cabinet positions and slide a plate into a cabinet.  \\
    Plate-Slide-Back & Randomize positions for both the plate and cabinet retrieve a plate from the cabinet. \\
    \bottomrule
    \end{tabular}
\end{adjustbox}
\label{ml1-descriptions}
\end{table*}

\paragraph{Type-1 Diabetes} In addition to the MuJoCo and Meta-World physical control tasks, we also evaluate our proposed approach on a more challenging real-world task: Realistic Type-1 Diabetes, which aims to control an insulin pump to regulate the blood glucose level of a patient. Each step of an episode corresponds to a minute in an in-silico patient's body, and each episode consists of 1440 time steps, corresponding to one day. When a patient consumes food, the blood glucose level will increase. Extremely high or low blood glucose levels can lead to life-threatening conditions such as hyperglycemia and hypoglycemia, respectively. Therefore, insulin dosage is required to mitigate the risks associated with these conditions. 

In the glucose control task, the states are hidden patient parameters, the observations are continuous blood glucose readings, the actions consist of discrete insulin dosages within the range [0,~5], and the transition dynamics are specific to the patient but unknown to the agent. Following the previous work~\cite{Basu2023OnTC}, we design a biologically inspired custom zone reward that incentivizes the time spent in the target zone and penalizes hyperglycemia and hypoglycemia:
\begin{equation}
    r_{t}(s_{t-1},s_{t})=\left\{ \begin{array}{l}
	-20\ \ \ s_t<70\ \text{or\ }s_t>200\ \left( \text{episode\ termination} \right)\\
	-1\ \ \ \ \ \ \ s_t<100\ \text{or\ }s_t-s_{t-1}<0.5\ \left( \text{hypoglycemia} \right)\\
	-1\ \ \ \ \ \ \ s_t>150\ \text{or\ }s_t-s_{t-1}>0.5\ \left( \text{hyperglycemia} \right)\\
	50\ \ \ \ \ \ \ \ 100\le s_t \le 150 \left( \text{target\ blood\ glucose} \right)\\
\end{array} \right.
\end{equation}
where $r_{t}$ encourages the agent to maintain the target vital statics within a healthy range.

We induced non-stationarity by oscillating the body parameters, such as the consumption amount of total glucose. Specifically, we simulate non-stationary fluctuations in blood glucose levels by varying the daily intake of lunch (1 meal) or both lunch and dinner (2 meals) in both adolescent and adult populations. Specifically, the meal plan (meal time and size) is set as follows: 7:00 AM: 45 (breakfast), 12:00 AM: 70 (launch), 16:00 PM: 15 (snack1), 18:00 PM: 80 (dinner), 23:00 PM: 10 (snack). The meal plan for adults is exactly the same as that for adolescents. The range for the size of the adults' meals for lunch or lunch and dinner is 60-80, while the range for adolescents for lunch or lunch and dinner is 50-80. The goal of this system is to responsibly update the doctor's initial prescription, ensuring that treatment continually improves. 

\paragraph{MuJoCo} We also adopted 4 classic MuJoCo control tasks for comparison. In the setting of non-stationary RL, we primarily considered two different types of MuJoCo evolution tasks: changes in reward functions and dynamics. The environment details and the number of tasks for meta-training and meta-testing in different tasks are shown in Table~\ref{mujoco-details}.

\begin{itemize}
    \item \textbf{Changes in reward functions.} In the Hopper-Vel and Walker-Vel tasks, the target velocity of the agent is sampled from the uniform distribution $U\left[ 0.5,3.0 \right]$ every $60\pm 20$ time steps to simulate non-stationarity.

    \item \textbf{Changes in dynamics.} In Cheetah-Damping, the damping parameters are sampled from the set $\left[0.85, 0.9, 0.95, 1.0 \right]$. In Walker-Rand-Params, the physical parameters of the agent, such as body mass, damping, and friction, are randomized. The variation in these physical parameters at $60\pm 20$ time steps results in non-stationary environmental dynamics.
\end{itemize}

\begin{table*}[ht]
\centering
\caption{Classical MuJoCo details}
\vskip +0.03in
\begin{adjustbox}{max width=0.65\textwidth}
  \begin{tabular}{lcccc}
    \toprule
    Task  & Observation dim & Action dim & \multicolumn{1}{c}{Training tasks} & \multicolumn{1}{c}{Test tasks} \\
    \midrule
    Cheetah-Damping & 17    & 6     & 60     & 10 \\
    Hopper-Vel & 11    & 3     & 100   & 30 \\
    Walker-Rand-Params & 17    & 6     & 100   & 30 \\
    Walker-Vel & 17    & 6     & 100   & 30 \\
    \bottomrule
    \end{tabular}
\end{adjustbox}
\label{mujoco-details}
\vskip -0.1in
\end{table*}

\paragraph{Hyperparameters} The hyperparameters utilized in the experiments are outlined in Table \ref{hyperparameters}.

\begin{table}[htbp]
  \caption{Hyperparameters}
  \label{hyperparameters}
  \centering
  \begin{tabular}{ll}
    \toprule
    Hyperparameter & Value \\
    \midrule
    Encoder training steps & 200 \\
    RL layer size & 300 \\
    Target entropy factor & 1.0 \\
    Learning rate & 3e-4 \\
    Dims. of task representation $z$ & 5 \\
    Coef. of soft update $\sigma$ & 5e-3 \\
    Replay buffer size & 10,000,000 \\
    Batch size of the policy & MuJoCo \& Type-1 Diabetes: 256, Meta-World: 512 \\
    Batch size of the encoder & MuJoCo \&  Type-1 Diabetes: 256, Meta-World: 512 \\
    Maximum trajectory length & MuJoCo \& Type-1 Diabetes: 800, Meta-World: 200 \\
    Evaluation trajectories & MuJoCo \& Type-1 Diabetes: 2, Meta-World: 1 \\
    Policy training steps & MuJoCo: 2,000, Meta-World: 4,000, Type-1 Diabetes: 200 \\
    Training tasks per episode & MuJoCo (Cheetah: 20), Type-1 Diabetes: 25, Meta-World: 50 \\
    Initial trainsitions & MuJoCo \& Meta-World: 200, Type-1 Diabetes: 400 \\
    Trainsitions per episode & MuJoCo: 800, Meta-World: 600, Type-1 Diabetes: 200 \\
    Decomposition level $M$ & MuJoCo (Walker-Rand-Params: 3), Type-1 Diabetes: 2 \\ & Meta-World (Button-Press \& Handle-Press: 5): 2 \\
    Coef. of wavelet TD loss $\alpha_Y$ & MuJoCo (Hopper-Vel: 0.5): 0.9\\
    & Type-1 Diabetes (Adolescent-1meal: 0.5): 0.1 \\ & Meta-World (Faucet-Close \& Door-Close: 0.9, \\
    & Door-Lock: 0.8): 0.1\\
    \bottomrule
    \end{tabular}%
\end{table}

\section{Pseudocode}
\label{pseudocode}

\begin{center}
\begin{minipage}{12.7cm}
\begin{algorithm}[H]
\caption{\modelname algorithm}
\label{alg}
\begin{algorithmic}
\STATE \textbf{Input}: training tasks $\mathcal{\tilde{T}}^{train}$ from $\mathcal{P}(\mathcal{\tilde{T}})$, replay buffer $\mathcal{B}$, context encoder $e_\eta(\mathbf{z}|\mathcal{C})$,\\ 
    wavelet representation network $Y_{\phi}(\mathbf{\hat{z}}|\mathbf{z})$, $W$ network $W_{\varphi}(\mathbf{u}|\mathcal{C},e)$, entropy term $\mathcal{H}$\\
    contextual policy $\pi_\theta(a|s,\hat{z})$, critic $Q_\upsilon(s,a,\hat{z})$.

\begin{tcolorbox}
[colback=blue!5!white,colframe=blue!50!black!50!,left=2pt,right=2pt,top=0pt,bottom=1pt,colbacktitle=blue!25!white,title=\textbf{\color{black} Collecting Training Data}] 
        \FOR{training task $\mathcal{\tilde{T}}^{train}$}
        \STATE{Generate a task representation sequence $\mathbf{z}\sim e_\eta(\mathbf{z}|\mathcal{C})$}
        \STATE{Roll-out policy $\pi_\theta(a|s,z)$ and add transitions ${(s_j, a_j, s'_j, r_j)}_{j:1\cdots \mathbb{J}}$ to $\mathcal{B}$}
        \ENDFOR
\end{tcolorbox}

\begin{tcolorbox}[colback=red!5!white,colframe=red!50!black!50!,left=2pt,right=2pt,top=0pt,bottom=1pt,colbacktitle=red!25!white,title=\textbf{\color{black} Training Context Encoder and Wavelet Representation Network}]
\FOR{each context encoder training step}
         \STATE{Sample $c_i={(s_i, a_i, s'_i, r_i)}_{i:1\cdots \mathbb{I}}\sim \mathcal{B}$ for context encoder}
        \STATE{Generate a task representation $z\sim e_\eta(z|c_i)$}
        \STATE{Obtain a wavelet task representation sequence: $\hat{\mathbf{z}}=Y_{\phi}(\mathbf{z})$}
        \STATE{Train context encoder $e_\eta$: $\eta \leftarrow \eta-\lambda_\eta \hat{\nabla}_\eta \mathcal{J}_{\eta}$ \textcolor{blue}{(Eq.~\ref{equ:e})}}
        \STATE{Train wavelet representation network $Y_{\phi}$: $\phi \leftarrow \phi-\lambda_\phi \hat{\nabla}_\phi \mathcal{J}_{\phi}$ \textcolor{blue}{(Eq.~\ref{eq:Ynet})}} 
        \STATE{Train $W$ network $W_{\varphi}$: $\varphi \leftarrow \varphi-\lambda_\varphi \hat{\nabla}_\varphi \mathcal{J}_{\phi}$ \textcolor{blue}{(Eq.~\ref{eq:Ynet})}}
        \STATE{Train target $W$ network $W_{\mu}$: $\mu \leftarrow \sigma\varphi+(1-\sigma)\mu$}
        \ENDFOR
\end{tcolorbox}

\begin{tcolorbox}
[colback=yellow!5!white,colframe=yellow!50!black!50!,left=2pt,right=2pt,top=0pt,bottom=1pt,colbacktitle=yellow!25!white,title=\textbf{\color{black} Training Policy}]
\FOR{each policy training step}
            \STATE{Train contextual policy $\pi_\theta$: $\theta \leftarrow \theta-\lambda_\pi \hat{\nabla}_\theta \mathcal{J}_{\theta}$ \textcolor{blue}{(Eq.~\ref{eq:actor})}}
            \STATE{Obtain a wavelet task representation sequence: $\hat{\mathbf{z}}=Y_{\phi}(\mathbf{z})$}
            \STATE{Train contextual critic $Q_\upsilon$: $\upsilon_l \leftarrow \upsilon_l-\lambda_Q \hat{\nabla}_{\upsilon_l}\mathcal{J}_{\upsilon_l}\text {, for } l \in\{1,2\}$ \textcolor{blue}{(Eq.~\ref{eq:critic})}}
            \STATE{Train temperature coefficient $\alpha$: $\alpha \leftarrow \alpha-\lambda_\alpha \hat{\nabla}_\alpha\mathbb{E}\left[-\alpha \log \pi_\theta\left(a|s,\hat{z}\right)-\alpha \mathcal{H}\right]$}
            \STATE{Train target contextual critic $Q_\zeta$: $\zeta_l \leftarrow \sigma \upsilon_l+(1-\sigma) \zeta_l \text{, for } l \in\{1,2\}$}
        \ENDFOR
\end{tcolorbox}

\begin{tcolorbox}
[colback=teal!5!white,colframe=teal!50!black!50!,left=2pt,right=2pt,top=0pt,bottom=1pt, colbacktitle=teal!25!white,title=\textbf{\color{black} Testing on Unseen Tasks}]
\STATE{Initialize transitions $c^{\mathcal{\tilde{T}}} = \left\{\right\} $}
    \STATE{Sample test tasks $\mathcal{\tilde{T}}^{test}$ from $\mathcal{P}(\mathcal{\tilde{T}})$}
    \FOR {$g=1, \ldots, \mathbb{G}$}
    \STATE{Generate a task representation $z\sim e_\eta(z|c_g)$}
    \STATE{Roll out policy $\pi_{\theta}(a|s,z)$ to generate transitions $\mathbb{D}^{\mathcal{\tilde{T}}}_{g}={(s_j, a_j, s'_j, r_j)}_{j:1\cdots \mathbb{J}}$}
    \STATE{Store the transitions: $c^{\mathcal{\tilde{T}}} = c^{\mathcal{\tilde{T}}}\cup \mathbb{D}^{\mathcal{\tilde{T}}}_{g}$}
    \ENDFOR
\end{tcolorbox}

\end{algorithmic}
\end{algorithm}
\end{minipage}
\end{center}

\section{Implementation Details}
\label{a-implement}
\paragraph{Model Architecture} All components of \modelname are implemented as MLPs. We report the total number of learnable parameters for our model initialized for the Walker-Vel task ($\mathcal{S}\in \mathbb{S}^{17},\ \mathcal{A}\in \mathbb{A}^6$). We  summarize the architecture and model size of \modelname using PyTorch-like notation in Alg.~\ref{alg:net}.

\begin{algorithm}[htbp]
\caption{\modelname Network Structure Pytorch-like Pseudocode}
\definecolor{codeblue}{rgb}{0.28125,0.46875,0.8125}
\lstset{
    basicstyle=\fontsize{7.5pt}{7.5pt}\ttfamily,
    commentstyle=\fontsize{7.5pt}{7.5pt}\color{codeblue},
}
\begin{lstlisting}[language=Python]
Encoder parameters: 90,810
Q1, Q2, target Q1, target Q2 parameters: 189,601
Policy parameters: 191,112
Alpha network parameters: 351

Encoder(
  (fc0): Linear(in_features=2 * o_dim + a_dim + r_dim, out_features=200, bias=True)
  (fc1): Linear(in_features=200, out_features=200, bias=True)
  (fc2): Linear(in_features=200, out_features=200, bias=True)
  (last_fc): Linear(in_features=200, out_features=10, bias=True)
)

Q1, Q2, target Q1, target Q2(
  (fc0): Linear(in_features=(o_dim + latent_dim) + a_dim, out_features=300, bias=True)
  (fc1): Linear(in_features=300, out_features=300, bias=True)
  (fc2): Linear(in_features=300, out_features=300, bias=True)
  (last_fc): Linear(in_features=300, out_features=1, bias=True)
)

Policy(
  (fc0): Linear(in_features=22, out_features=300, bias=True)
  (fc1): Linear(in_features=300, out_features=300, bias=True)
  (fc2): Linear(in_features=300, out_features=300, bias=True)
  (last_fc): Linear(in_features=300, out_features=6, bias=True)
  (last_fc_log_std): Linear(in_features=300, out_features=6, bias=True)
)

Alpha network(
  (fc0): Linear(in_features=5, out_features=50, bias=True)
  (last_fc): Linear(in_features=50, out_features=1, bias=True)
)
\end{lstlisting}
\label{alg:net}
\end{algorithm}

\paragraph{Training Details} We evaluate our proposed method and other baselines on the NVIDIA GeForce RTX 3090 GPU. And PyTorch-like pseudo-code for training our model is illustrated in Alg.~\ref{alg:train}.

\begin{algorithm}[htbp]
\caption{\modelname Training Pytorch-like Pseudocode}
\definecolor{codeblue}{rgb}{0.28125,0.46875,0.8125}
\lstset{
    basicstyle=\fontsize{7pt}{7pt}\ttfamily,
    commentstyle=\fontsize{7pt}{7pt}\color{codeblue},
    stringstyle=\color{red}
}
\begin{lstlisting}[language=Python]
def train(self):
    
    # Collecting initial samples ...
    if self.initial_transitions > 0:
        self.env_steps += self.collect_data(self.train_tasks, self.initial_transitions)

    for epoch in gt.timed_for(range(self.num_epochs), save_itrs=True):

        # 1. Collect data with rollout coordinator
        collection_tasks = np.random.permutation(self.train_tasks)[:self.num_train_tasks]
        self.env_steps += self.collect_data(collection_tasks, self.num_transitions)

        # 2. Replay buffer stats
        self.replay_buffer.stats_dict = self.replay_buffer.get_stats()

        # 3. Train context encoder and wavelet representation network
        self.reconstruction_trainer.train(self.num_reconstruction_steps)

        # 4. Train policy with data from the replay buffer
        temp, sac_stats = self.policy_trainer.train(self.num_policy_steps)

        # 5. Evaluation
        eval_output = self.evaluate(LOG, 'train', collection_tasks, self.eval_traj)
        eval_output = self.evaluate(LOG, 'test', self.test_tasks, self.eval_traj)
        average_test_reward, std_test_reward = eval_output   

\end{lstlisting}
\label{alg:train}
\end{algorithm}

\section{Additional Experiment Results}
\label{a-results}
\subsection{Additional Ablation Results}

\begin{figure*}[ht]
\centering
\vskip -0.1in
\subfigure{
\includegraphics[width=0.32\textwidth]{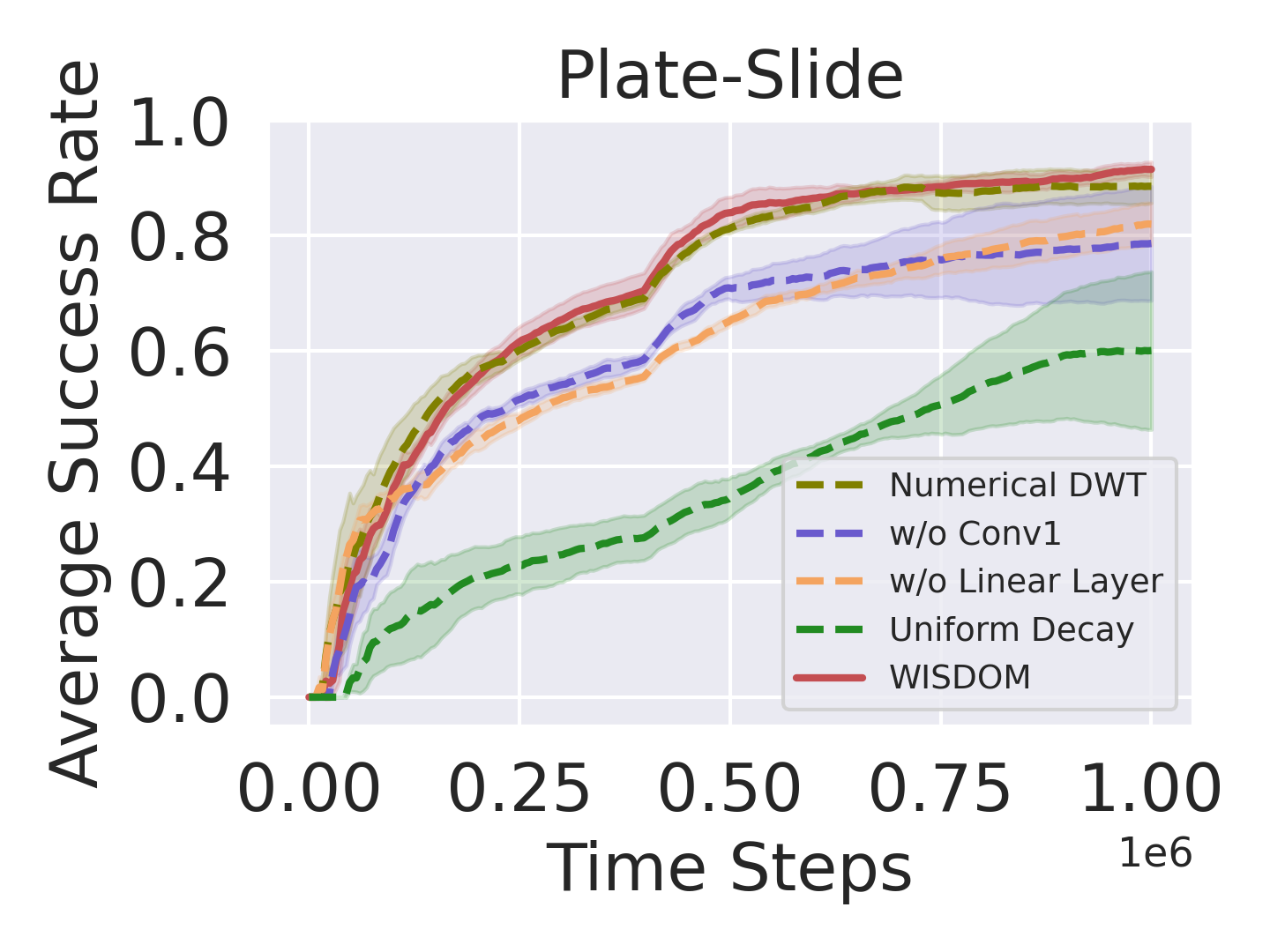}
}
\subfigure{
\includegraphics[width=0.32\textwidth]{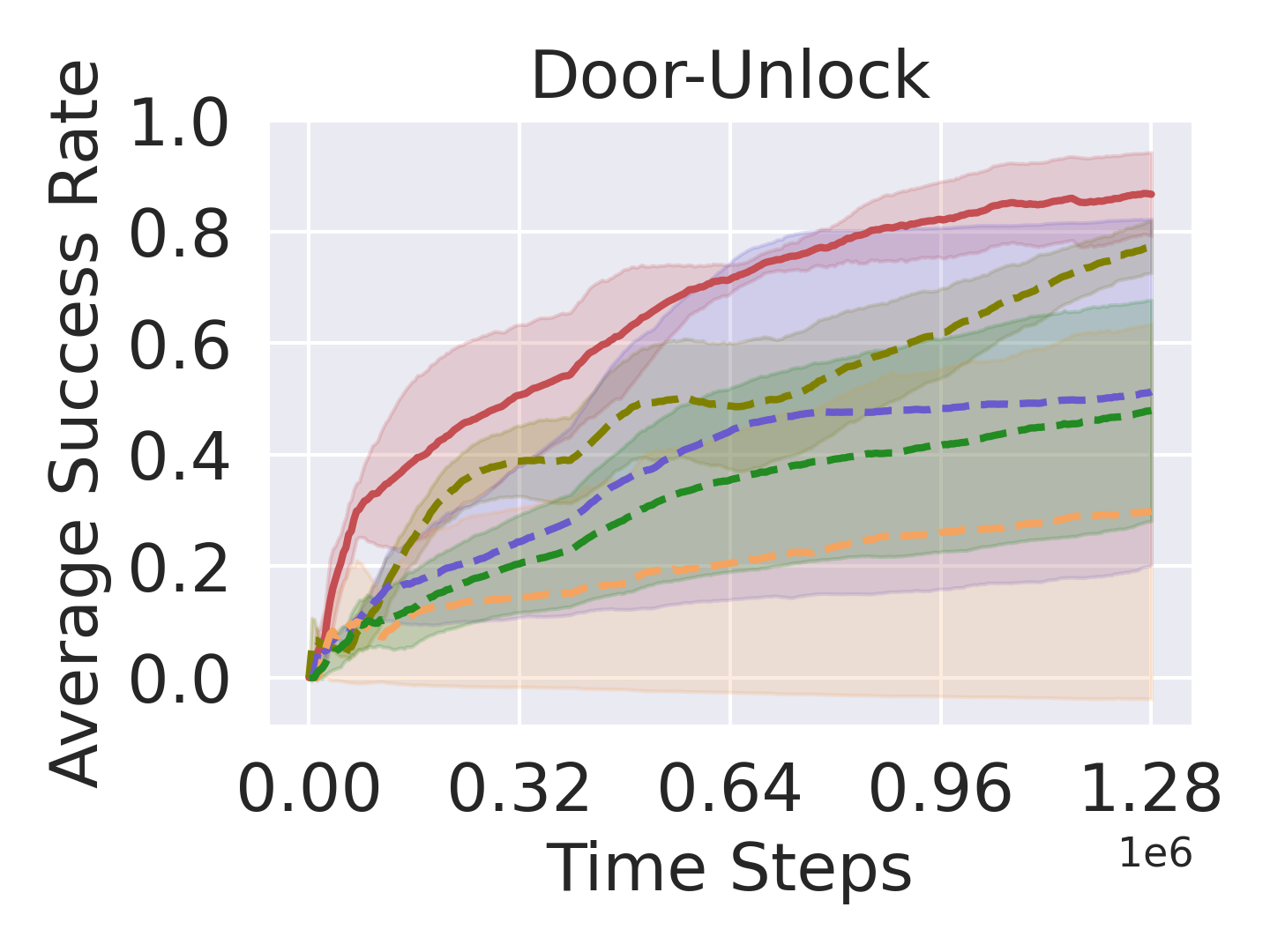}
}
\vskip -0.1in
\caption{Ablation on the design of the wavelet representation network.}
\label{a-fig:ablation}
\end{figure*}

\subsection{Evaluation Results Across Diverse Non-Stationary Degrees}
\begin{figure*}[h]
\vskip -0.1in
\centering
\subfigure{
\includegraphics[width=0.23\textwidth]{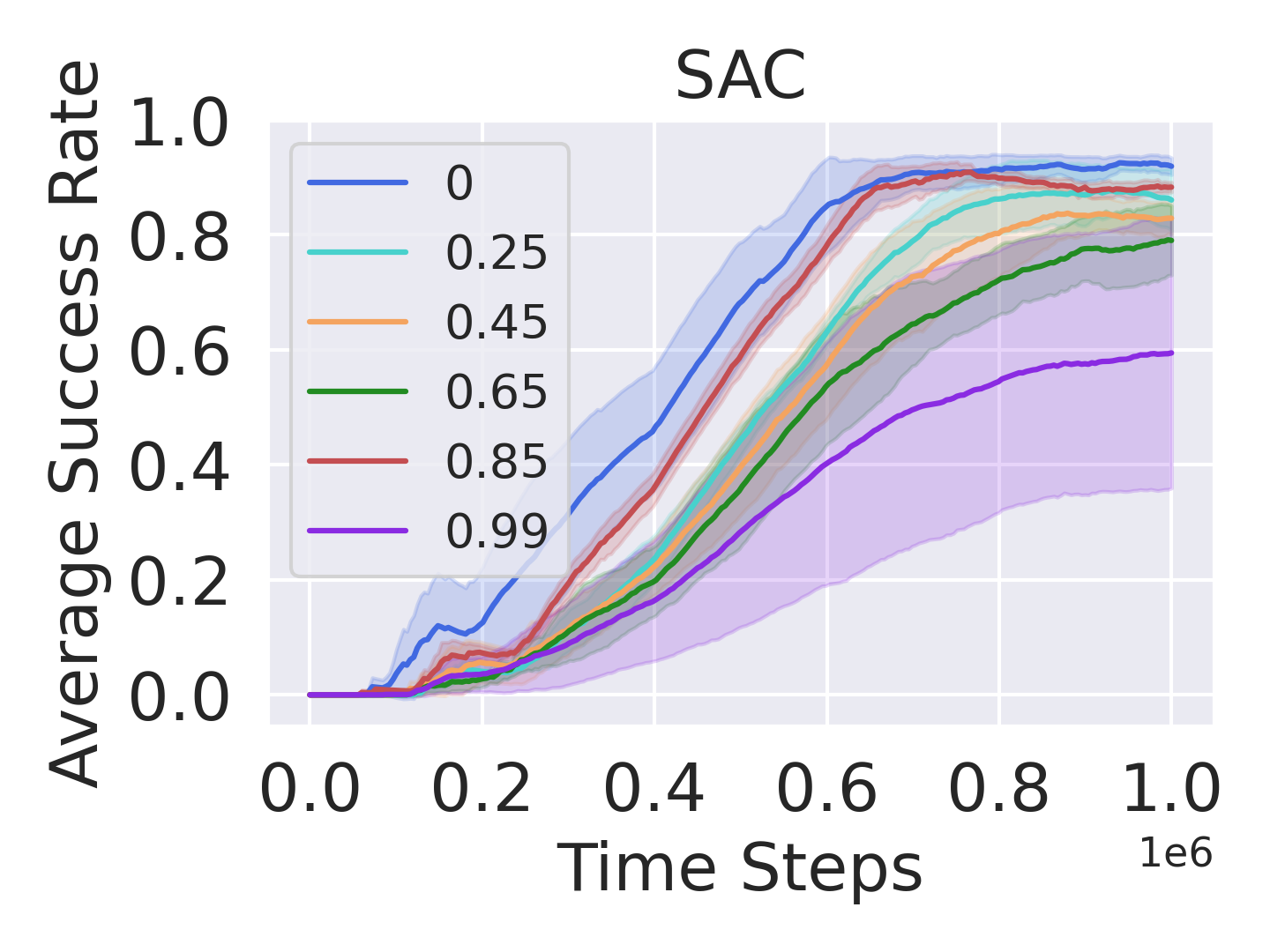}
}
\subfigure{
\includegraphics[width=0.23\textwidth]{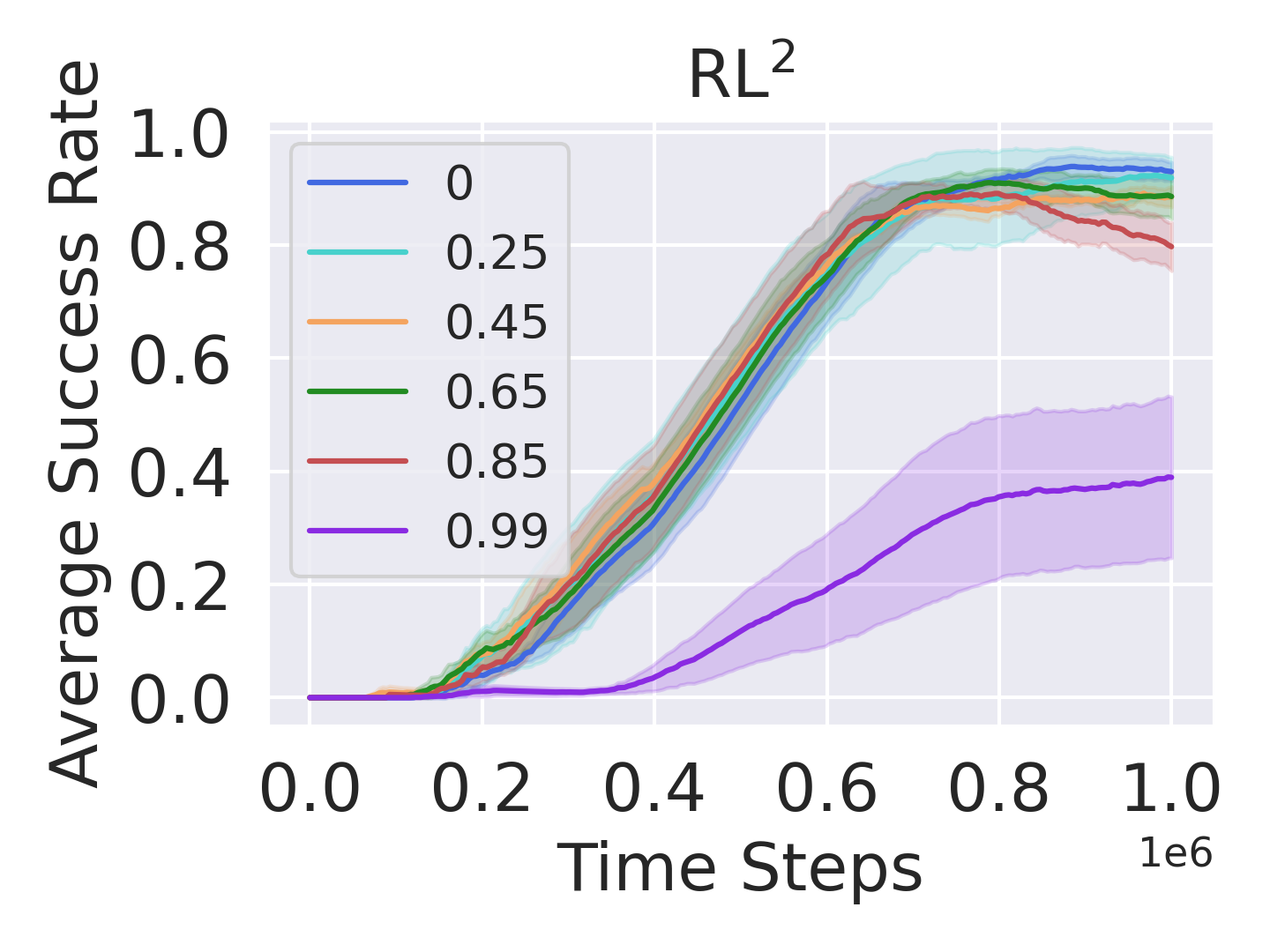}
}
\subfigure{
\includegraphics[width=0.23\textwidth]{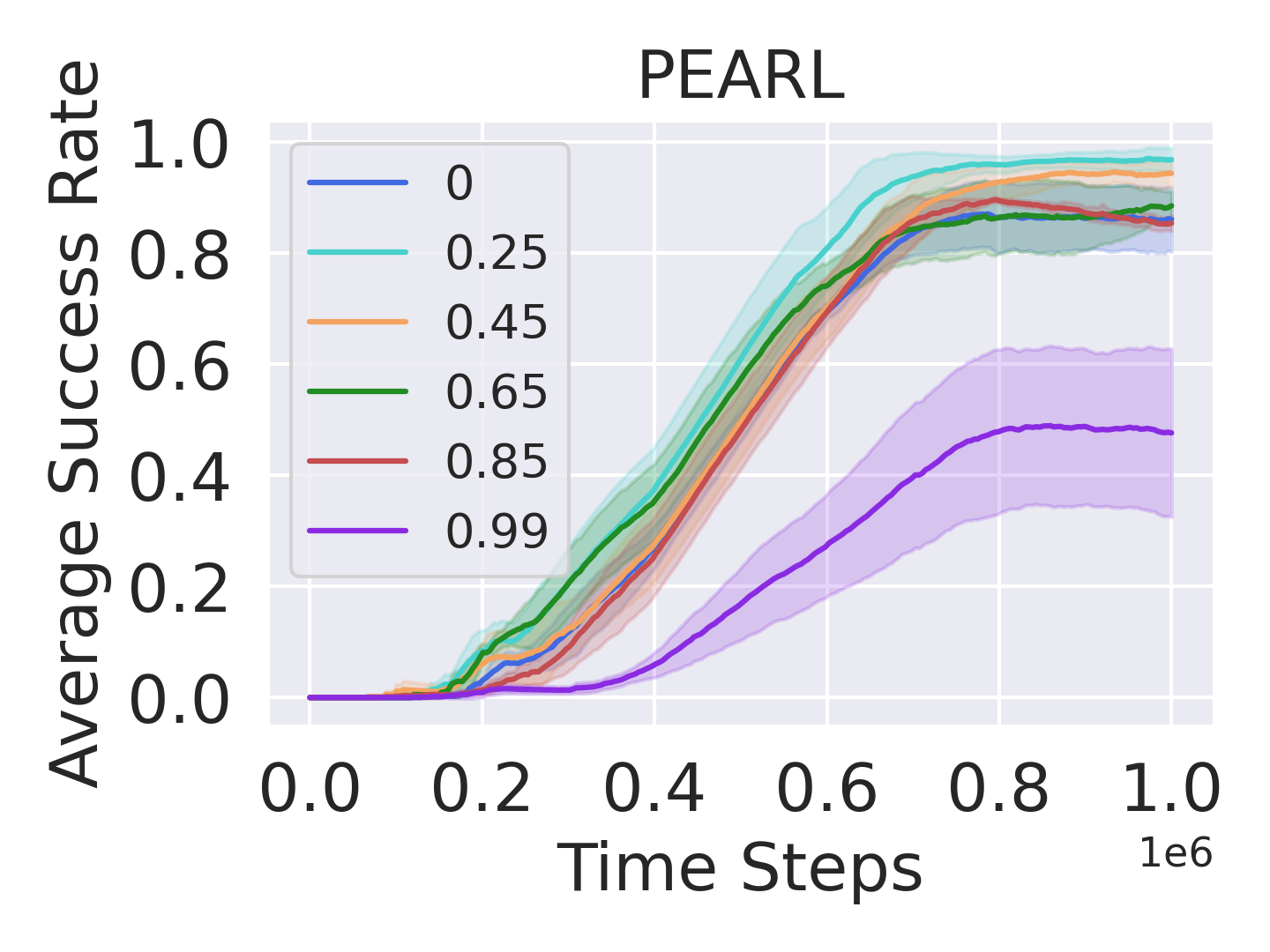}
}
\subfigure{
\includegraphics[width=0.23\textwidth]{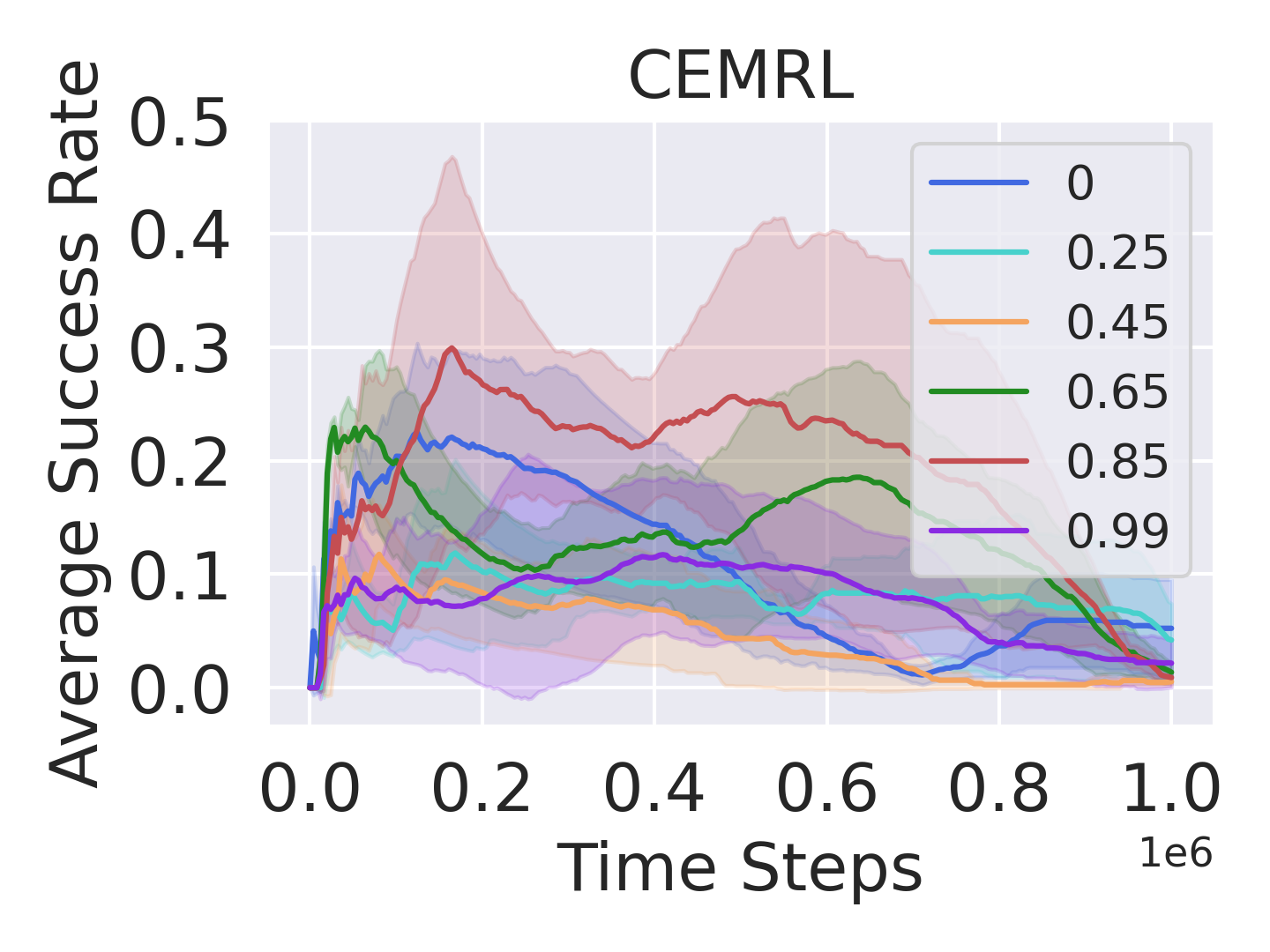}
}
\vskip -0.1in
\subfigure{
\includegraphics[width=0.23\textwidth]{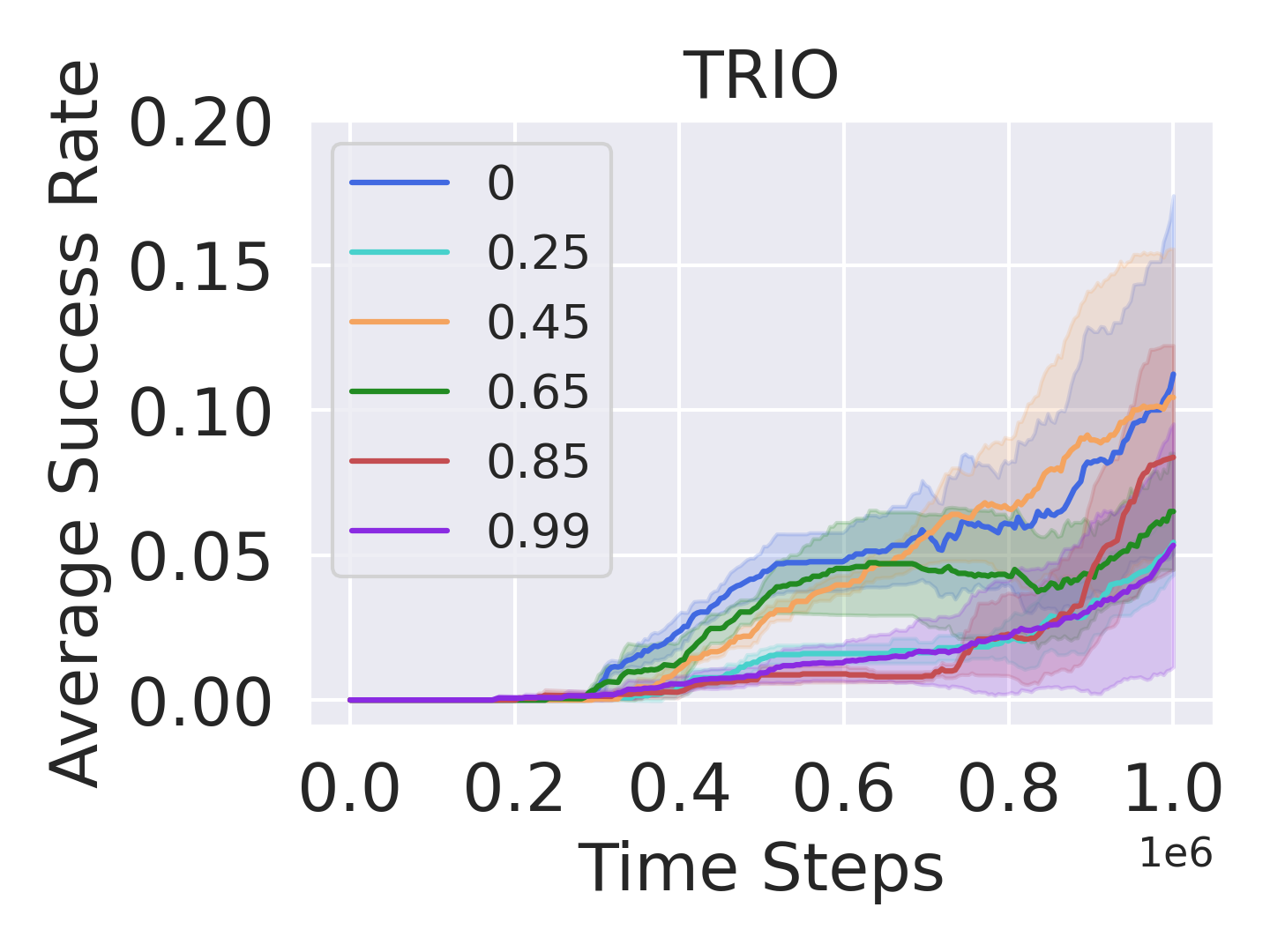}
}
\subfigure{
\includegraphics[width=0.23\textwidth]{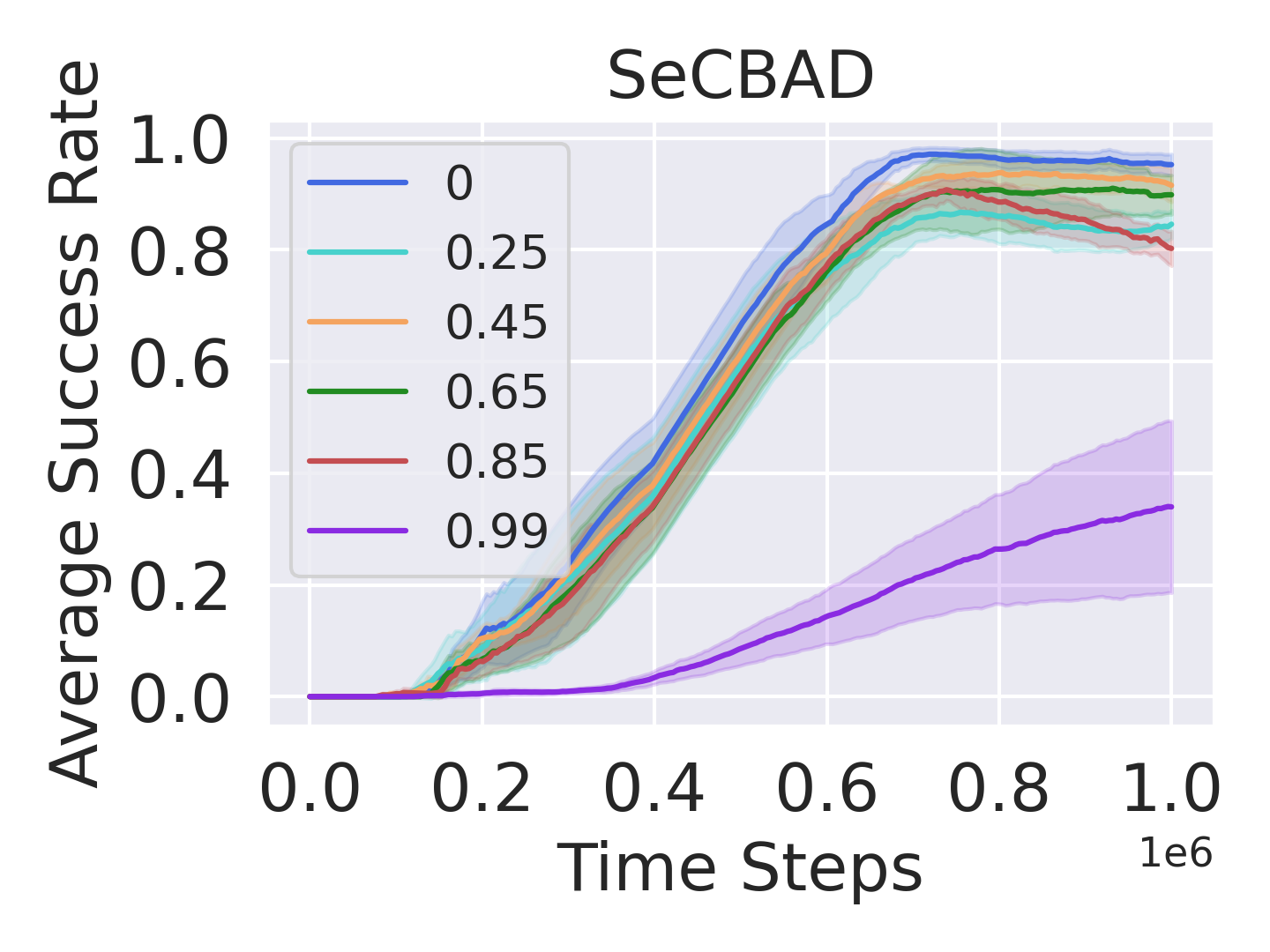}
}
\subfigure{
\includegraphics[width=0.23\textwidth]{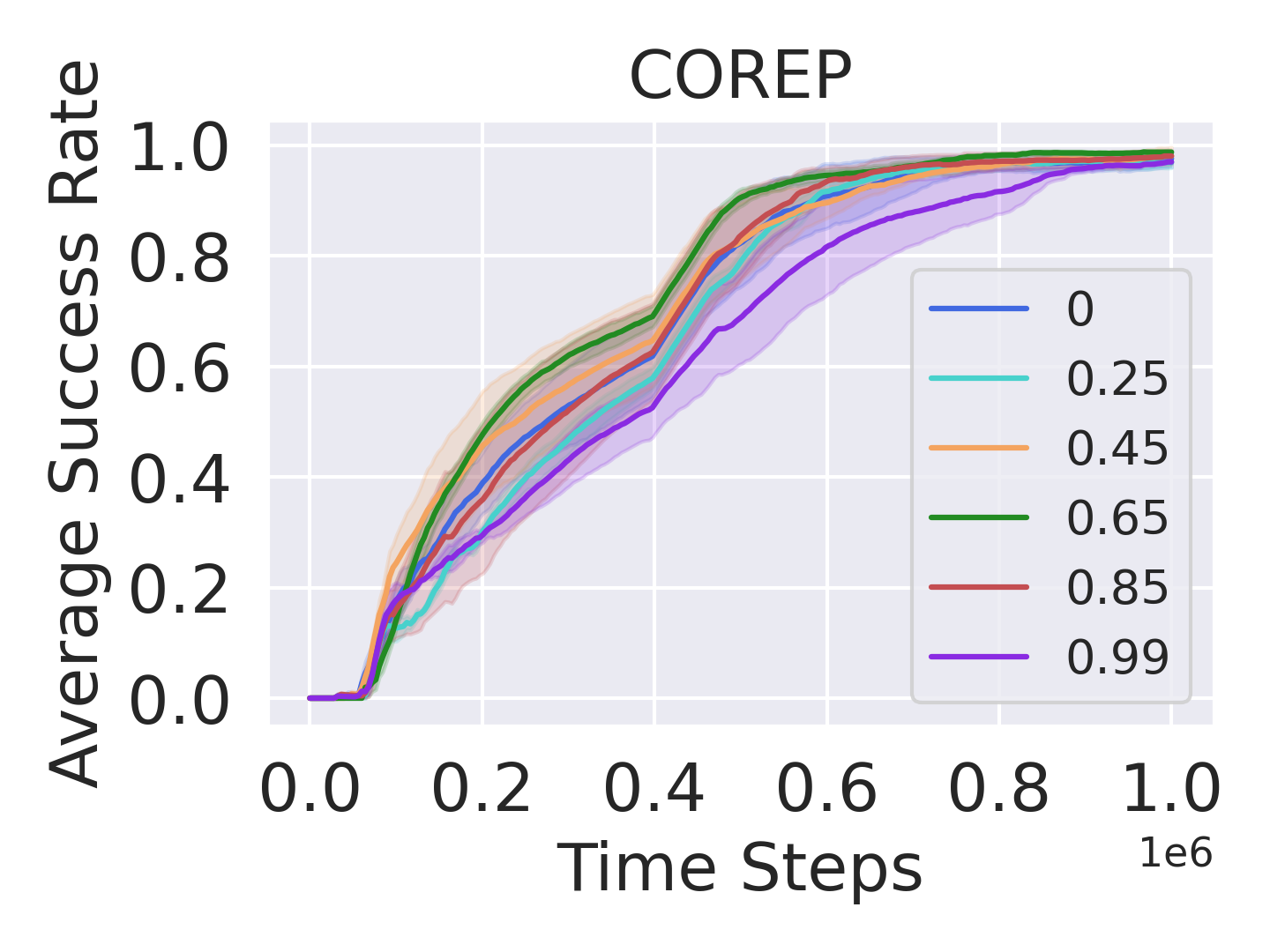}
}
\subfigure{
\includegraphics[width=0.23\textwidth]{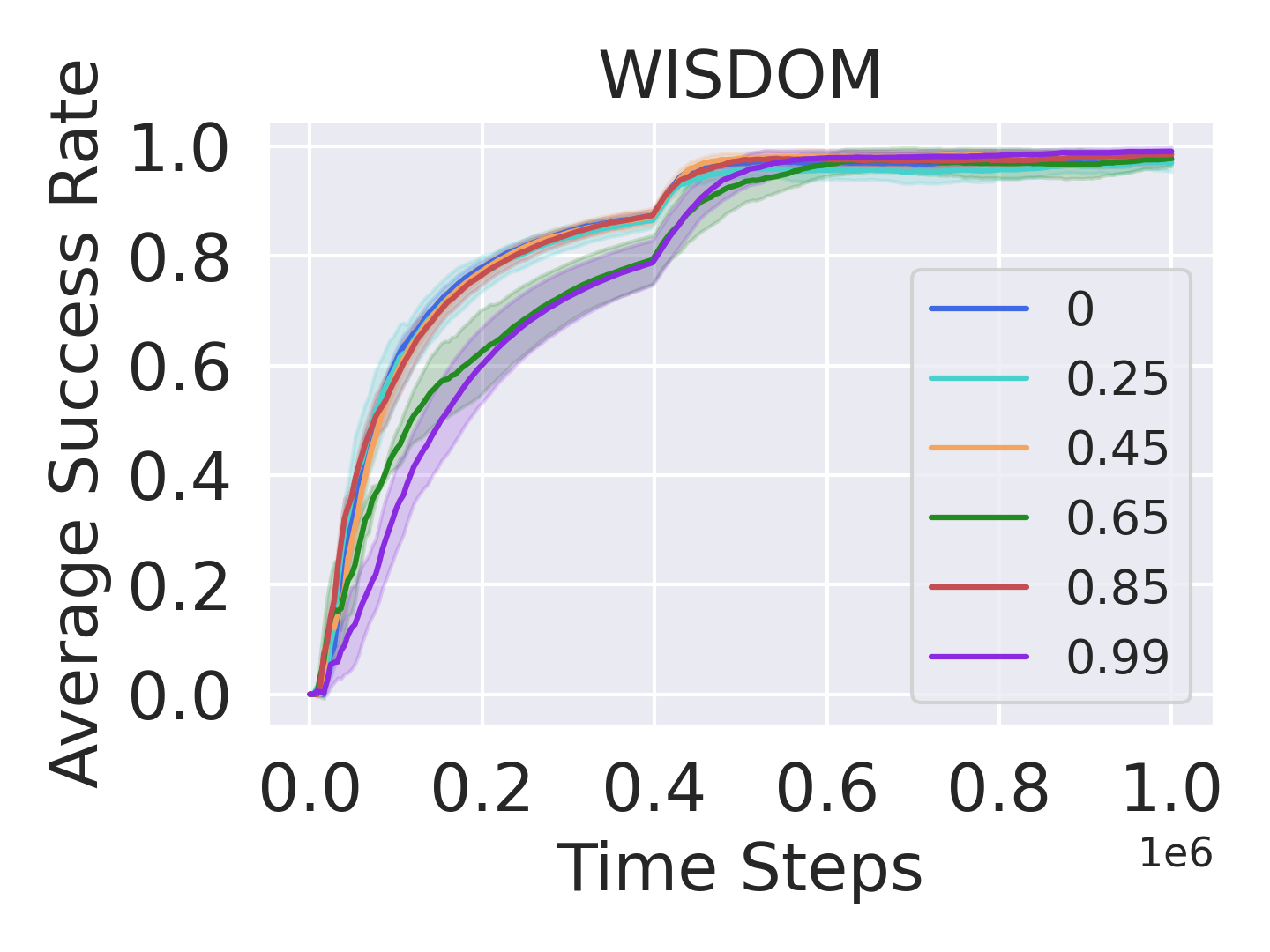}
}
\vskip -0.1in
\caption{Testing average performance on button-press over with varying non-stationary degrees.}
\vskip -0.1in
\label{a-fig:ns degree}
\end{figure*}

\subsection{Performance Disparities of Meta-RL in Stationary and Non-stationary Tasks}
\label{a-fig:s-vs-ns}
\begin{figure*}[h]
\centering
\vskip -0.1in
\subfigure{
\includegraphics[width=0.3\textwidth]{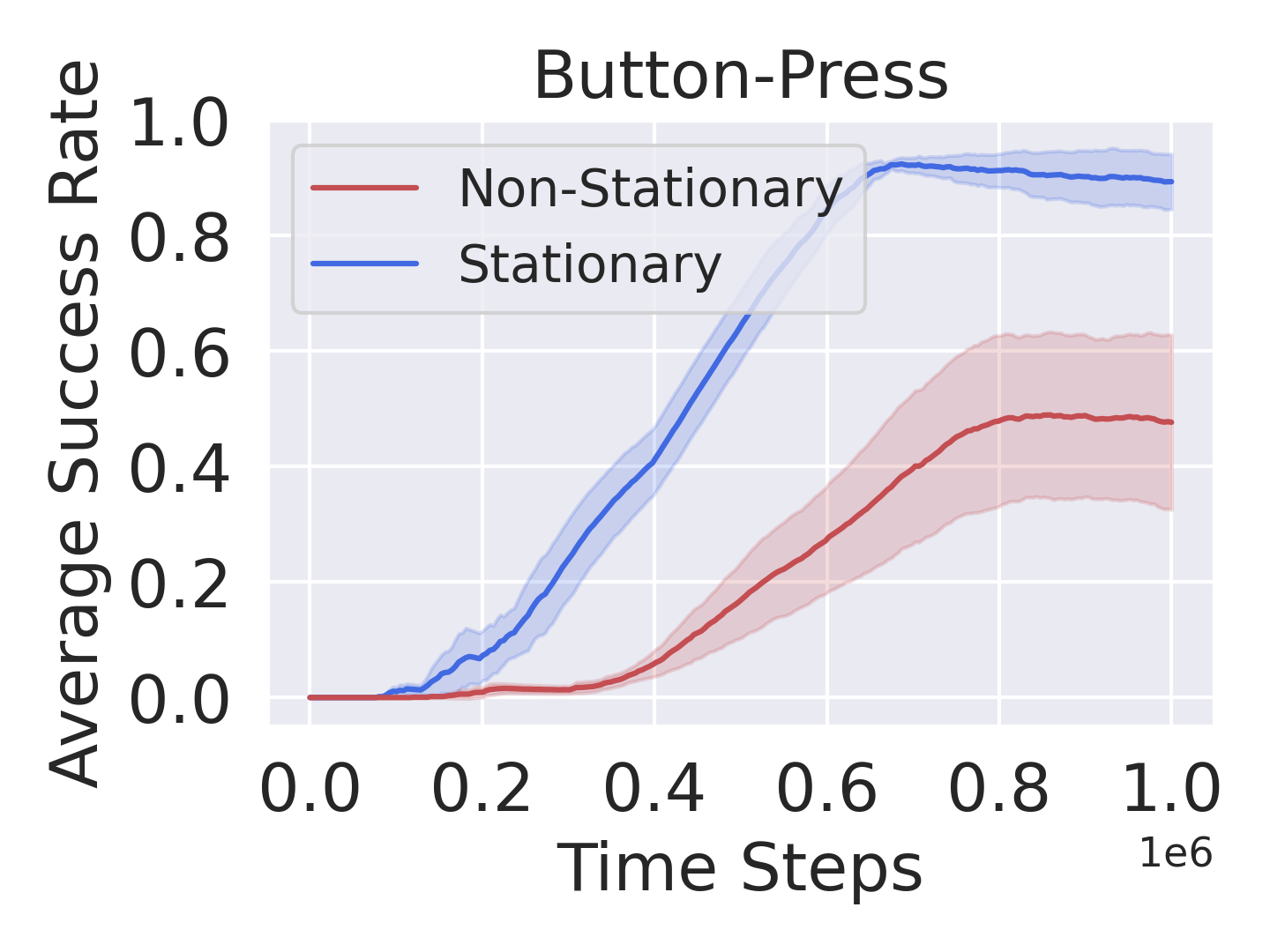}
}
\subfigure{
\includegraphics[width=0.3\textwidth]{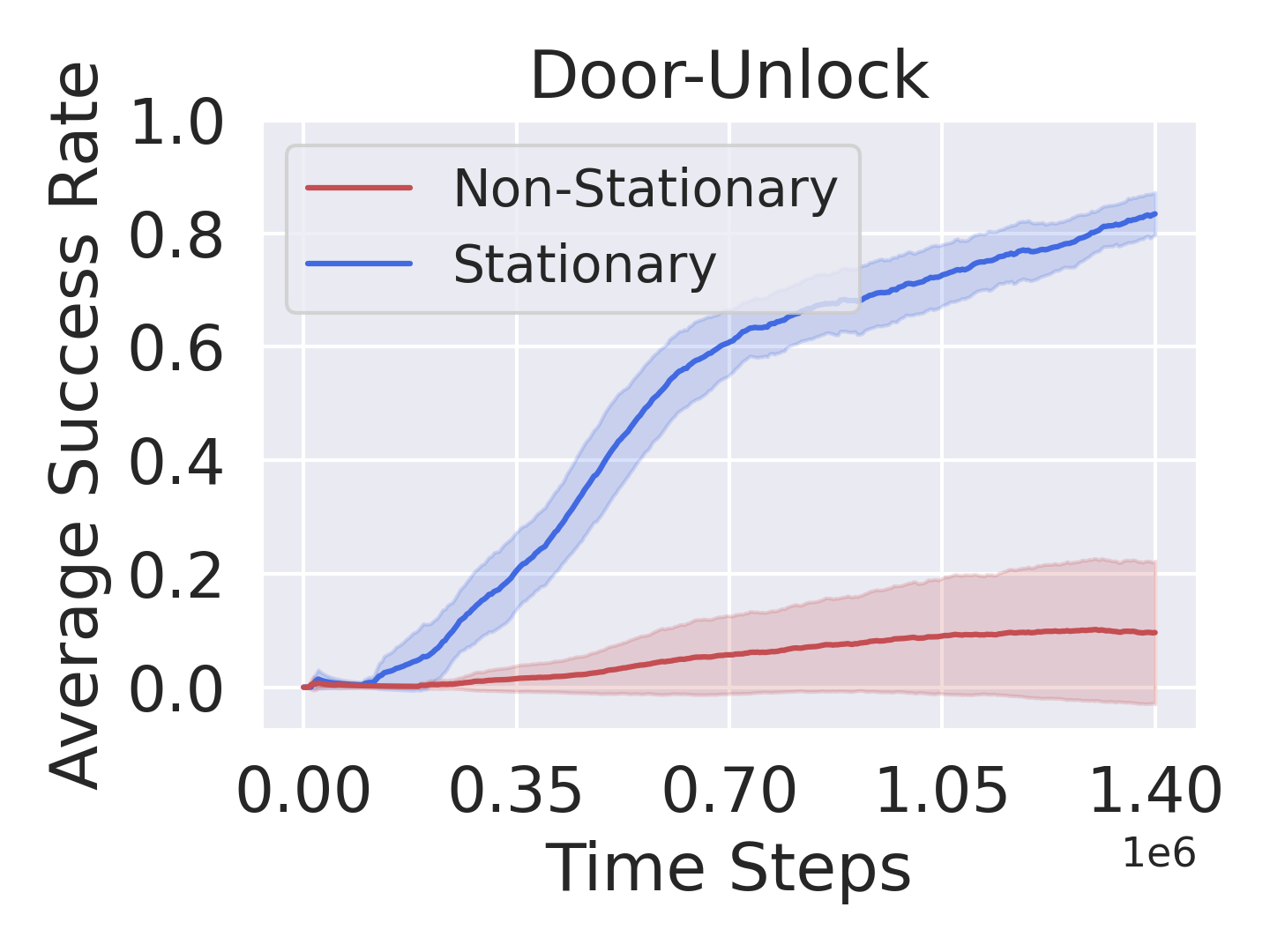}
}
\subfigure{
\includegraphics[width=0.3\textwidth]{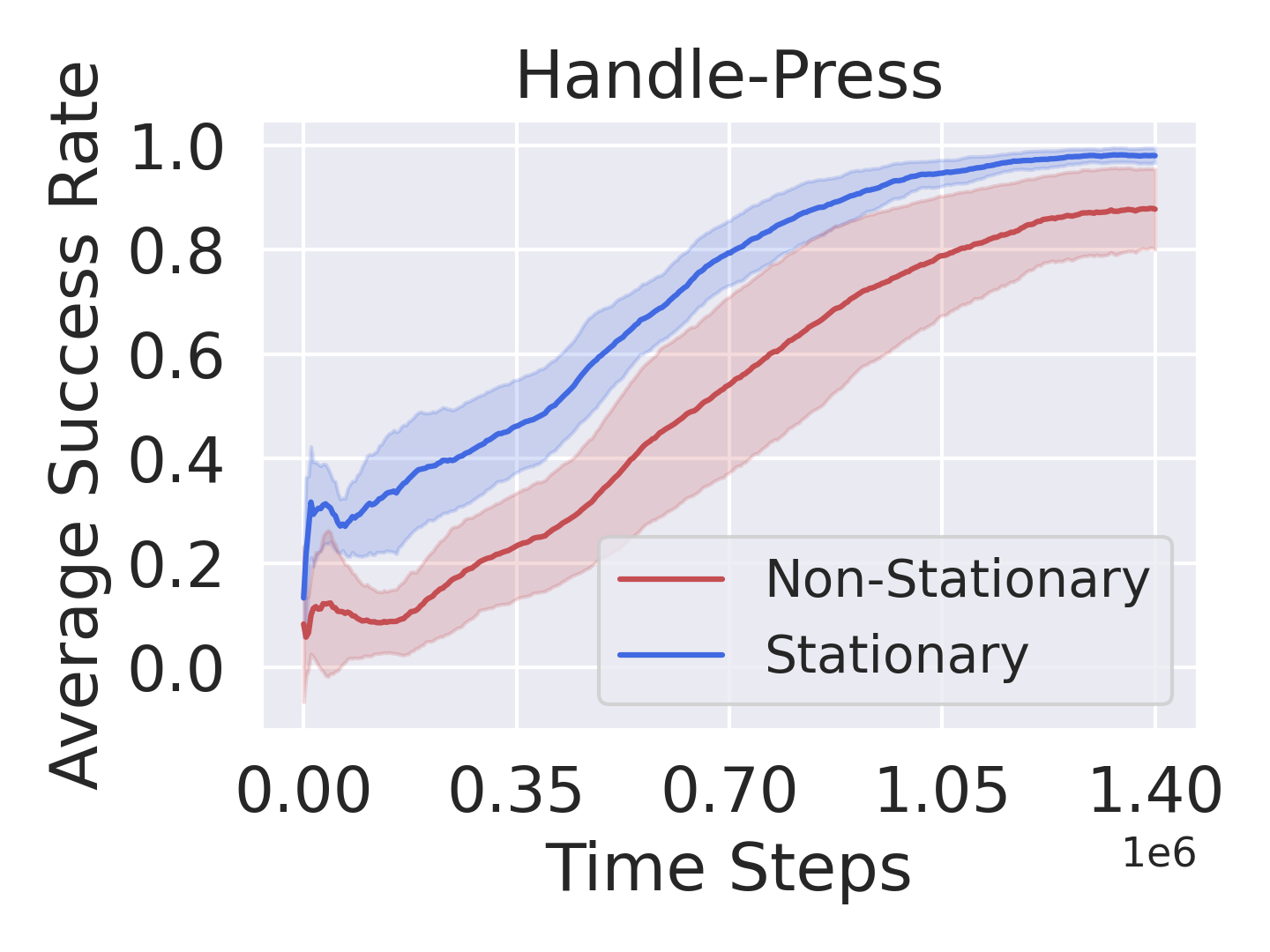}
}
\vskip -0.1in
\caption{Testing performance of the classical meta-RL algorithm PEARL~\cite{rakelly2019efficient} in stationary and non-stationary environments on Meta-World~\cite{yu2020meta}.}
\end{figure*}

\subsection{Sensitivity analysis of wavelet decomposition levels and wavelet TD loss coefficients} 
In Fig.~\ref{fig:hyper} (left), an appropriate setting of wavelet decomposition levels allows for more precise separation of frequency components corresponding to distinct evolutionary periods, thereby facilitating the extraction of more critical features. As non-stationarity increases, $M$ should correspondingly increase. However, excessive decomposition may lose vital low-frequency components via the approximation coefficients and introduce unnecessary detail, leading to inferior performance. In Fig.~\ref{fig:hyper} (right), $\alpha_Y$ typically influences the convergence speed. When wavelet TD loss or AR loss predominates, the convergence accelerates; even so, the final performance remains insensitive to $\alpha_Y$.

\begin{figure*}[ht]
\centering
\vskip -0.1in
\subfigure{
\includegraphics[width=0.32\textwidth]{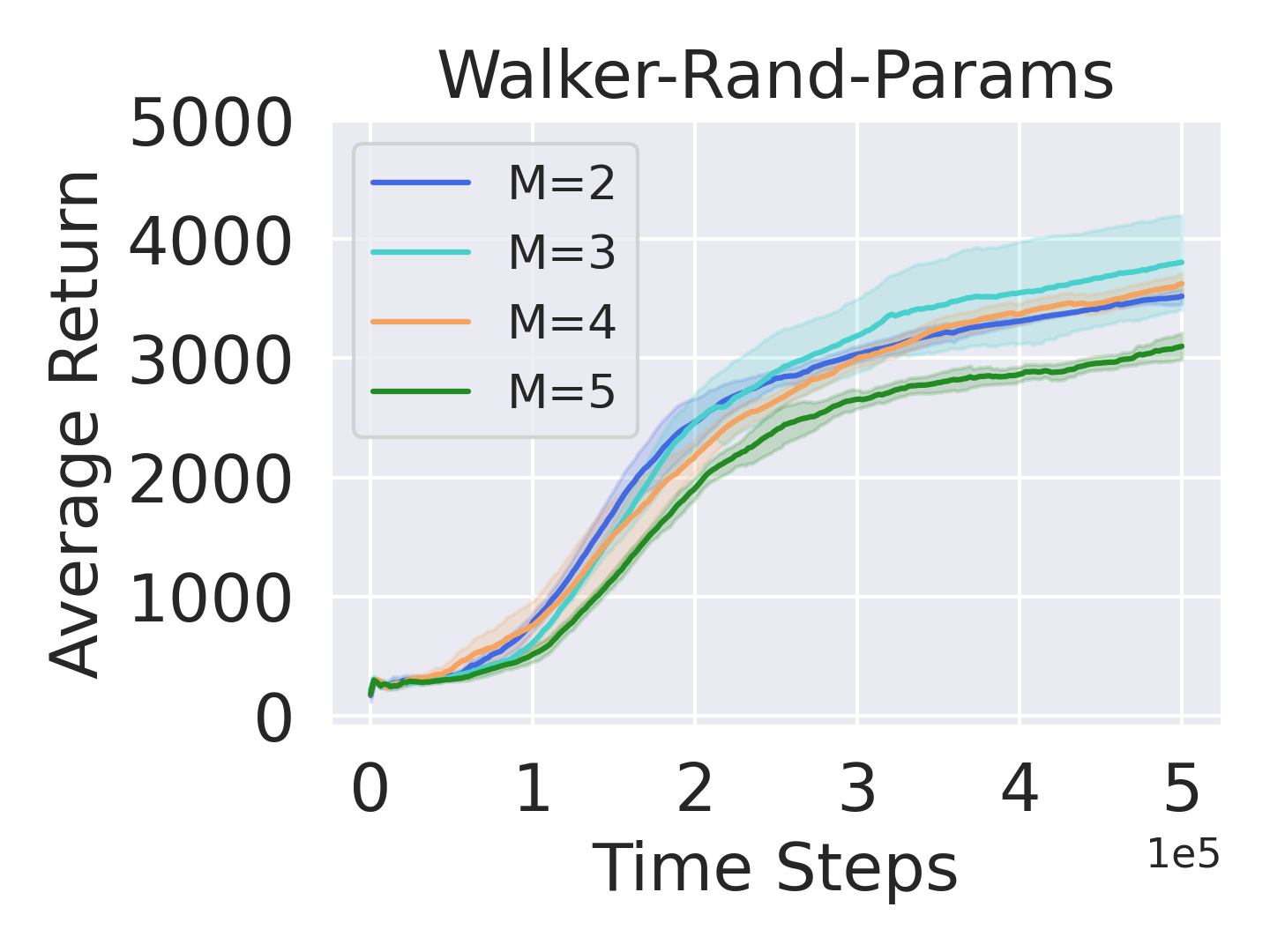}
}
\subfigure{
\includegraphics[width=0.32\textwidth]{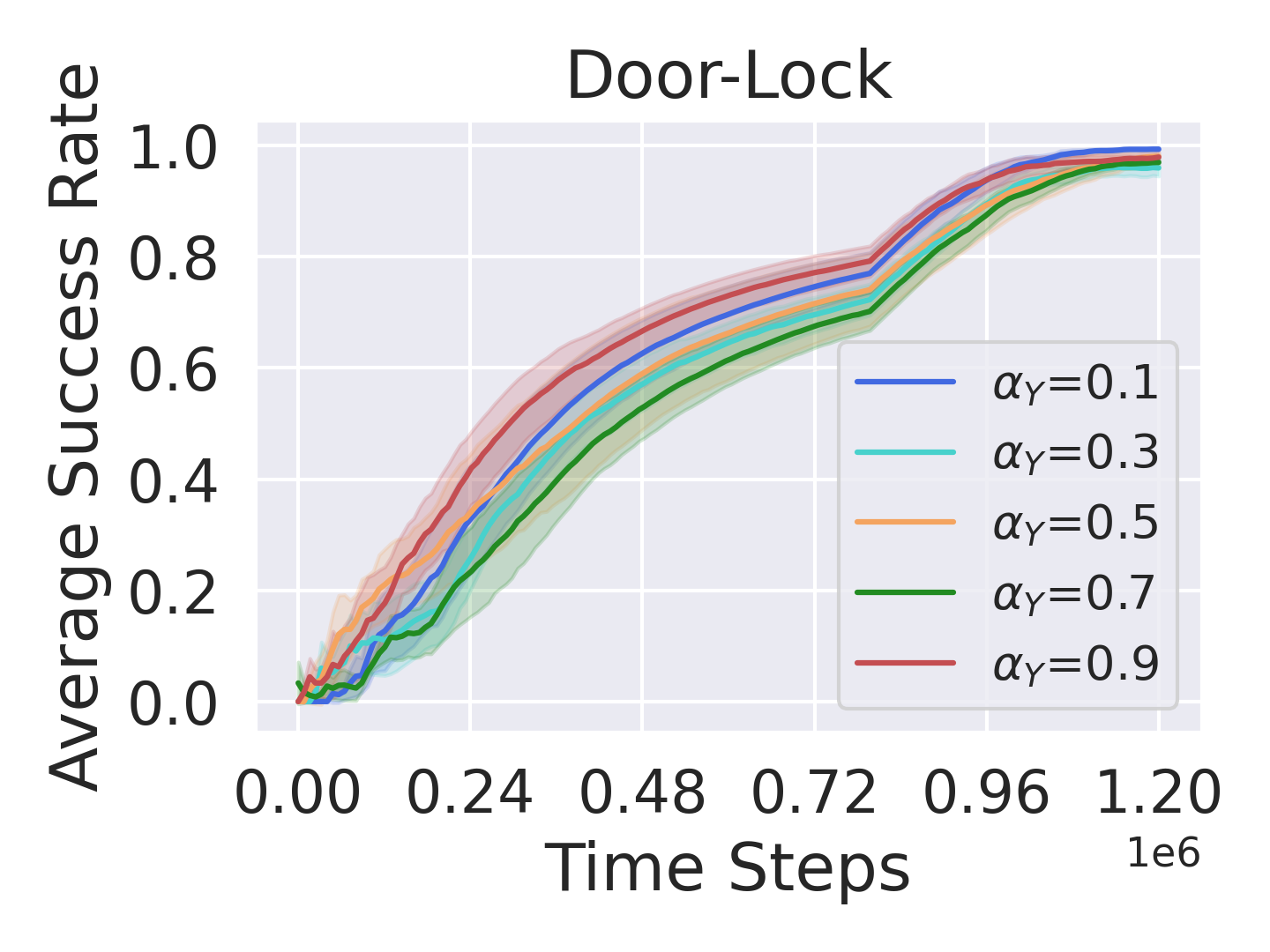}
}
\vskip -0.1in
\caption{Sensitivity to wavelet decomposition levels $M$ (left) and TD loss coefficients $\alpha_Y$ (right).}
\label{fig:hyper}
\end{figure*}

\subsection{Additional Experimental Results on Glucose Control Benchmark (Fig.~\ref{a-fig:medicine})}

\begin{figure*}[h]
\centering
\vskip -0.1in
\subfigure{
\includegraphics[width=0.32\textwidth]{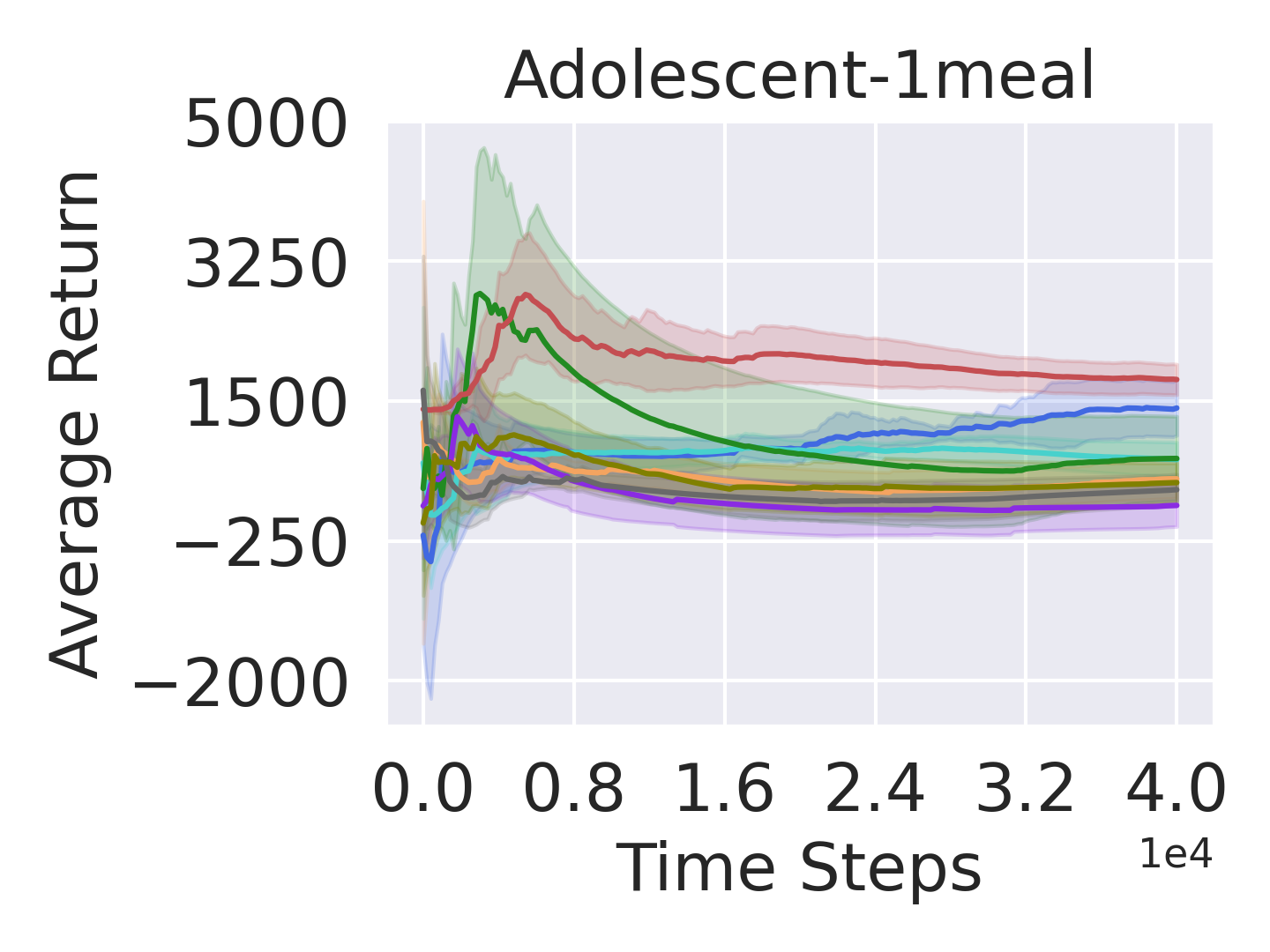}
}
\subfigure{
\includegraphics[width=0.32\textwidth]{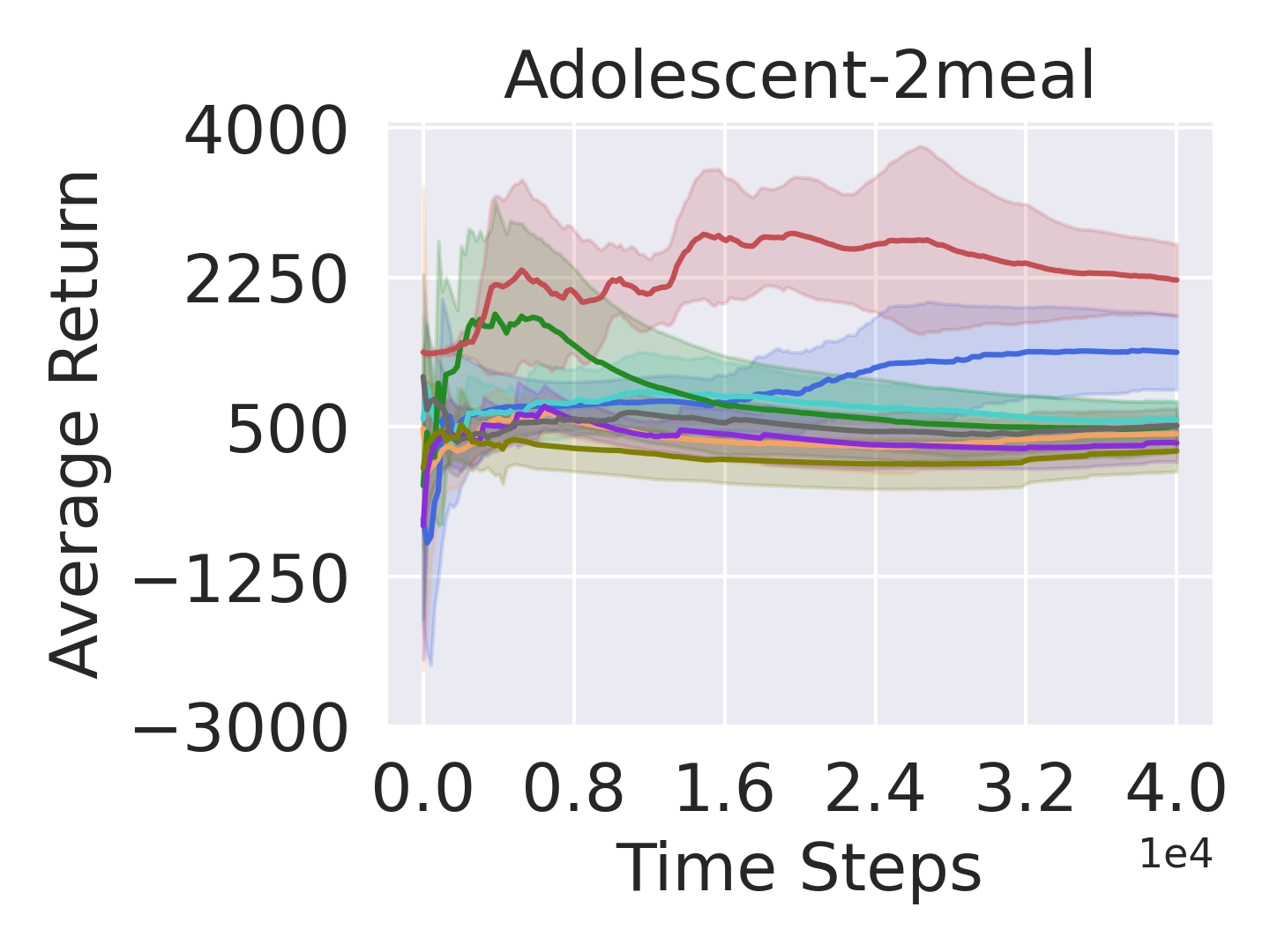}
}
\vskip -0.1in
\caption{Testing average return on 2 glucose control environments over 6 random seeds.
}
\label{a-fig:medicine}
\end{figure*}

\subsection{Additional Experimental Results on MuJoCo Benchmark (Fig.~\ref{a-fig:mujoco})}

\begin{figure*}[h]
\centering
\vskip -0.1in
\subfigure{
\includegraphics[width=0.32\textwidth]{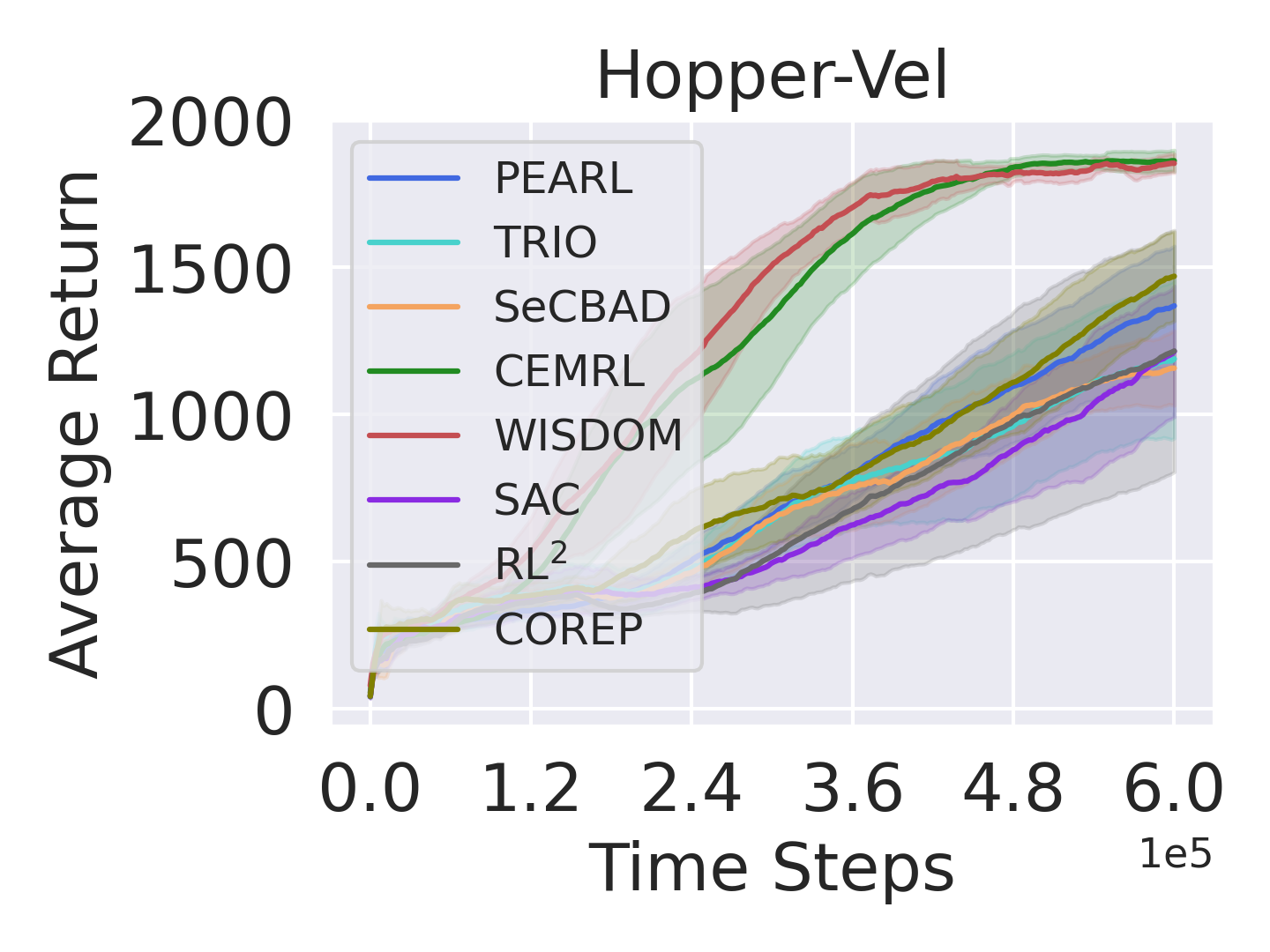}
}
\subfigure{
\includegraphics[width=0.32\textwidth]{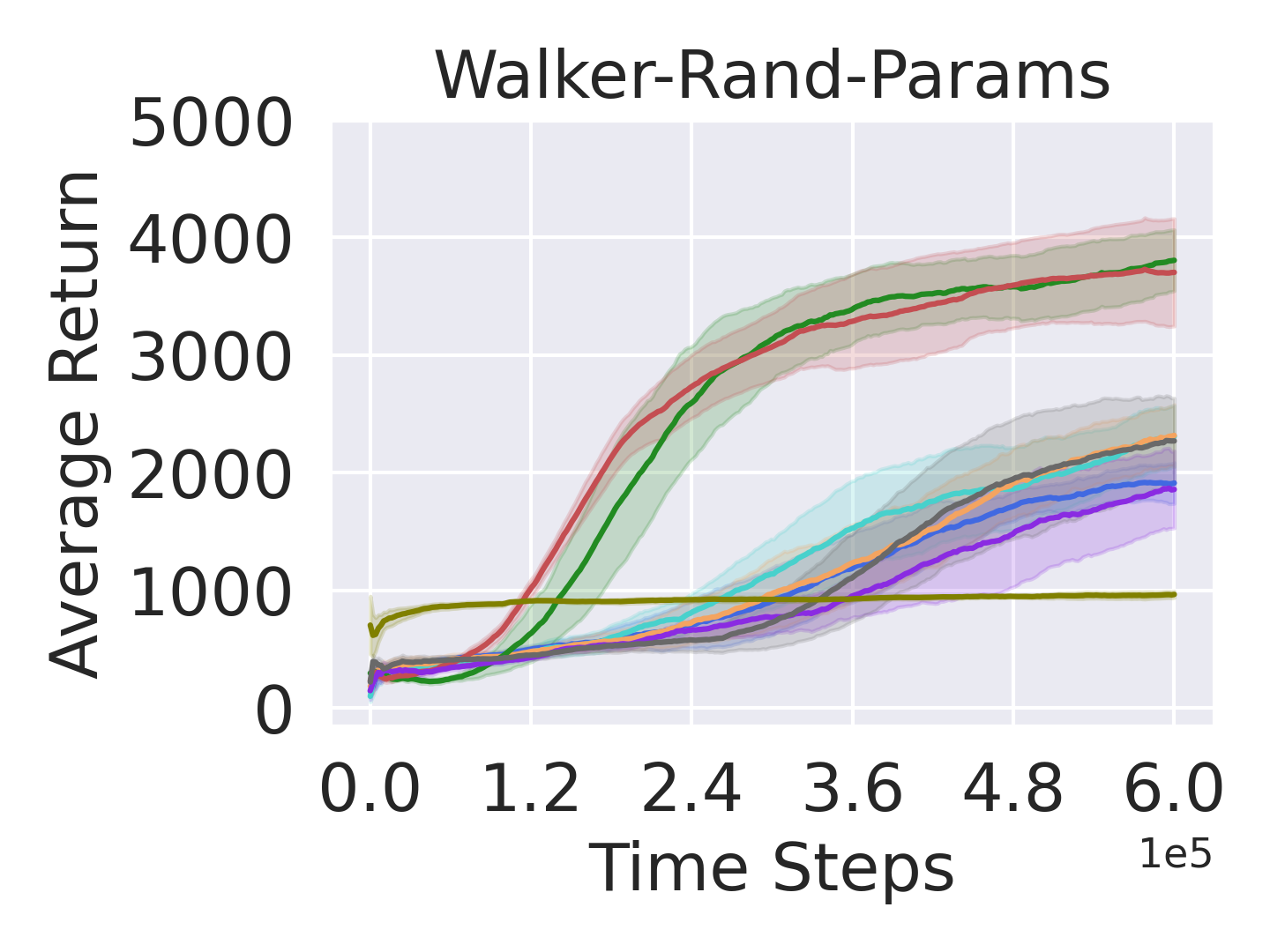}
}
\vskip -0.1in
\caption{Testing average return on MuJoCo over 6 random seeds.}
\label{a-fig:mujoco}
\end{figure*}

\end{document}